\documentclass[11pt]{article}

\usepackage{ragged2e}
\usepackage{lmodern}
\usepackage[utf8]{inputenc} 
\usepackage[T1]{fontenc}    
\usepackage{amsfonts}       
\usepackage{nicefrac}       
\usepackage{microtype}      
\usepackage{fullpage}
\usepackage{appendix}
\usepackage{mystyle1}

\usepackage{setspace}
\title{Provably Efficient Fictitious Play Policy Optimization for   Zero-Sum Markov Games with Structured Transitions}

\begin{document}

\author{Shuang Qiu\thanks{University of Chicago. 
Email: \texttt{qiush@umich.edu}.} 
       \qquad
        Xiaohan Wei\thanks{Meta Platforms, Inc.
Email: \texttt{ubimeteor@fb.com}.}         	
	\qquad
		Jieping Ye\thanks{University of Michigan. 
    Email: \texttt{jpye@umich.edu}.}
	\qquad    
       	Zhaoran Wang\thanks{Northwestern University.
    Email: \texttt{zhaoranwang@gmail.com}.}
    \qquad 
    Zhuoran Yang\thanks{
   Yale University. 
	Email: \texttt{zhuoran.yang@yale.edu}.}
}

\maketitle


\begin{abstract}
While single-agent policy optimization in a fixed environment has attracted a lot of research attention recently in the reinforcement learning community, much less is known theoretically when there are multiple agents playing in a potentially competitive environment. We take steps forward by proposing and analyzing new fictitious play policy optimization algorithms for zero-sum Markov games with structured but unknown transitions. We consider two classes of transition structures: factored independent transition and single-controller transition. For both scenarios, we prove tight $\widetilde{\mathcal{O}}(\sqrt{K})$ regret bounds after $K$ episodes in a two-agent competitive game scenario. The regret of each agent is measured against a potentially adversarial opponent who can choose a single best policy in hindsight after observing the full policy sequence. Our algorithms feature a combination of Upper Confidence Bound (UCB)-type optimism and fictitious play under the scope of simultaneous policy optimization in a non-stationary environment. When both players adopt the proposed algorithms, their overall optimality gap is $\widetilde{\mathcal{O}}(\sqrt{K})$.
\end{abstract}

\section{Introduction}

Widely applied in multi-agent reinforcement learning \citep{sutton2018reinforcement, bu2008comprehensive}, Policy Optimization (PO) has achieved tremendous empirical success \citep{foerster2016learning,leibo2017multi,silver2016mastering,silver2017mastering,berner2019dota,vinyals2019grandmaster}, due to its high efficiency and easiness to combine with different optimization techniques.  Despite these empirical successes, theoretical understanding of multi-agent policy optimization, especially the zero-sum Markov game \citep{littman1994markov} via policy optimization, lags rather behind.  Most recent works studying zero-sum Markov games (e.g.,  \citet{xie2020learning,bai2020provable}) focus on value-based methods achieving $\tilde{\cO}(\sqrt{K})$ regrets and they assume there is a central controller available solving for coarse correlated equilibrium or Nash equilibrium at each step, which brings extra computational cost. Here we let $K$ denote the total number of episodes. On the other hand, although there has been great progress in understanding single-agent PO algorithms \citep{sutton2000policy,kakade2002natural,schulman2015trust,papini2018stochastic,cai2019provably,bhandari2019global,liu2019neural}, directly extending the single-agent PO to the multi-agent setting encounters the main challenge of non-stationary environments caused by agents changing their own policies simultaneously \citep{bu2008comprehensive,zhang2019multi}. In this paper, we aim to answer the following open question:
\begin{center}
\emph{Can policy optimization provably solve zero-sum Markov games to achieve $\cO(\sqrt{K})$ regrets?}
\end{center}
As an initial attempt to tackle the problem, in this work, we focus on two \emph{non-trivial} classes of zero-sum Markov games with structured transitions: the \emph{factored independent} transition and the \emph{single-controller transition}. For the game with factored independent transition, the transition model is factored into two independent parts, and each player makes a transition following their own transition model. The single-controller zero-sum game assumes that the transition model is entirely controlled by the actions of Player 1. In both settings, the rewards received are decided jointly by the actions of both players. These two problems capture the non-stationarity of the multi-agent reinforcement learning in the following aspects: (1) the rewards depend on both players' potentially adversarial actions and policies in both settings; (2) the rewards further depend on both players' states in the factored independent transition setting; (3) Player 2 in the single-controller transition setting faces non-stationary states determined by Player 1’s policies. In addition to the non-stationarity, practically, the true transition model of the environment could be unseen and only bandit feedback is accessible to players. Thus, the non-stationarity, as well as the unknown transition and reward function, poses great challenges to multi-agent PO problems. 

In this paper, we propose two novel optimistic Fictitious Play (FP) policy optimization algorithms for the games with factored independent transition and single-controller zero-sum games respectively. Our algorithms are motivated by the close connection between the multi-agent PO and the FP framework. Specifically, FP \citep{robinson1951iterative} is a classical framework for solving games based on simultaneous policy updates, which includes two major steps: inferring the opponent (including learning the opponent's policy) and taking the best response policy against the policy of the opponent. As an extension of FP to Markov games, our proposed PO algorithms possess two phases of learning, namely policy evaluation and policy improvement. The policy evaluation phase involves exchanging the policies of the previous episode, which is motivated by the step of inferring the opponent in FP. By making use of the policies from the previous episode, the algorithms further compute the value function and the Q-function with the estimated reward function and transition model. By the principle of ``optimism in the face of uncertainty'' \citep{auer2002finite,bubeck2012regret}, their estimation incorporates UCB bonus terms to handle the non-stationarity of the environment as well as the uncertainty arising from only observing finite historical data. Furthermore, the policy improvement phase 
corresponds to taking the (regularized) best response policy via a mirror descent/ascent step (where the regularization comes from the KL divergence), which can be viewed as a soft-greedy step based on the historical information about the opponent and the environment. This step resembles the smoothed FP \citep{fudenberg1995consistency,perolat2018actor,zhang2019multi} for normal form games (or matrix games). During this phase, both players in the factored independent transition setting and Player 2 in the single-controller setting demand to estimate the opponent's state reaching probability to handle the non-stationarity.

For each player, we measure the performance of its algorithm by the regret of the learned policy sequence comparing against the best policy in hindsight after $K$ episodes. In the two settings, our proposed algorithms can achieve an $\tilde{\cO}(\sqrt{K})$ regret for both players, matching the regret of value-based algorithms. Furthermore, with both players running the proposed PO algorithms, they have $\tilde{\cO}(\sqrt{K})$ optimality gap.  To the best of our knowledge, this seems the first provably sample-efficient fictitious play policy optimization algorithm for zero-sum Markov games with the two structured but unknown transitions. Our work also partially solves one open question in \citet{bai2020provable} that how to solve a zero-sum Markov game of multiple steps $(H \geq 2)$ with an $\tilde{\cO}(\sqrt{K})$ regret via mirror descent-type (policy optimization) algorithms.

\section{Related Work}

There have been a large number of classical works studying the games with the independent transition model, e.g., \citet{altman2005zero,altman2008constrained,flesch2008stochastic,singh2014characterization}.
In addition, the single-controller games are also broadly investigated in many existing works, .e.g., \citet{parthasarathy1981orderfield,filar1984matrix,rosenberg2004stochastic,guan2016regret}. Most of the aforementioned works do not focus on the non-asymptotic regret analysis. \citet{guan2016regret} studies the regret of the single-controller zero-sum game but with an assumption that the transition model is known to players. In contrast, our paper provides a regret analysis for both transition models under a more realistic setting that the transition model is unknown. Games with the two structured transition models are closely associated with the applications in communications. The game with the factored independent transition \citep{altman2005zero} finds applications in wireless communications.  An application example of the single-controller game is the attack-defense modeling in communications \citep{eldosouky2016single}.

Recently, many works are focusing on the non-asymptotic analysis of Markov games \citep{heinrich2016deep, guan2016regret, wei2017online, perolat2018actor,zhang2019policy,xie2020learning, bai2020provable}. Some of them aim to propose sample-efficient algorithms with theoretical regret guarantees for zero-sum games.  \citet{wei2017online} proposes an algorithm extending single-agent UCRL2 algorithm \citep{jaksch2010near}, which requires solving a constrained optimization problem each round. \citet{zhang2019policy} also studies PO algorithms but does not provide regret analysis, which also assumes an extra linear quadratic structure and a known transition model. In addition, recent works on Markov games   \citep{xie2020learning,bai2020provable,liu2020sharp,bai2020near} propose value-based algorithms under the assumption that there exists a central controller that specifies the policies of agents by finding the coarse correlated equilibrium or Nash equilibrium for a set of matrix games in each episode. \citet{bai2020provable} also makes an attempt to investigate PO algorithms in zero-sum games. However, their work shows restrictive results where each player only plays one step in each episode. Right prior to our work, \citet{daskalakis2021independent} also studies the policy optimization algorithm for a two-player zero-sum Markov game under an assumption of bounded distribution mismatch coefficient in a non-episodic setting. To achieve a certain error $\varepsilon$ for the convergence measure defined in their work, their proposed algorithm requires an $\cO(\varepsilon^{-12.5})$ sample complexity. A concurrent work \citep{tian2020provably} studies zero-sum games under a different online agnostic setting with PO methods and achieves an $\tilde{\cO}(K^{2/3})$ regret. Motivated by classical fictitious play works \citep{robinson1951iterative,fudenberg1995consistency,heinrich2015fictitious,perolat2020poincar},
for the episodic Markov game, we focus on the setting where there is no central controller which determines the policies of the two players and we propose a policy optimization algorithm where each player updates its own policy based solely on the historical information at hand. Moreover, under the framework of FP, our work does not require the bounded distribution mismatch coefficient assumption (or gradient domination) that is used in some recent works. Our result matches the  $\cO(\sqrt{K})$  regret upper bounds in  \citet{xie2020learning,bai2020provable} that are obtained by value-based methods.

Furthermore, we note that the game for each individual player can be viewed as a special case of MDPs with adversarial rewards and  bandit feedbacks due to the adversarial actions of opponents. 
For such a class of MDP models in general, \citet{jin2019learning} 
proposes an  algorithm based on  mirror descent involving occupancy measures  and attains an $\tilde \cO(\sqrt{K})$ regret. 
However, each update step of the algorithm requires solving another optimization problem which is more computationally demanding than our PO method. Besides, it is also unclear whether the  algorithm in \citet{jin2019learning}  can be extended to zero-sum games. 
Moreover, for the same MDP model,  \citet{efroni2020optimistic} proposes an optimistic policy optimization algorithm that achieves an $\tilde{\cO}(K^{2/3})$ regret. 
Thus, directly applying this result would yield an $\tilde{\cO}(K^{2/3})$ regret. 
In fact, regarding the problem as an MDP with adversarial rewards neglects the fact that such ``adversarial reward functions'' are determined by the actions and policies of the opponent. 
Thus, since each player knows the past actions taken and policies executed by the opponent under the FP framework, both players can construct accurate estimators of the environment after a sufficiently large number of episodes. 
As we will show in Sections \ref{sec:FIT} and \ref{sec:SCT}, 
the proposed PO methods explicitly utilize the information of the opponent in the policy evaluation step, which is critical for the methods to obtain an $\tilde{\cO}(\sqrt{K})$ regret.

\section{Background and Preliminaries}
In this section, we formally introduce notations and setups. Then, we describe the two transition structures in details.

\subsection{Notations and Setups}
We define a tabular episodic two-player zero-sum Markov game (MG) by the tuple $(\cS, \cA, \cB,  H, \cP, r)$, where $\cS$ is the state space, $\cA$ and $\cB$ are the action spaces of Player 1 and Player 2 respectively, $H$ is the length of each episode, $\cP_h(s' \given s, a )$ denotes the transition probability at the $h$-th step to the state $s'$ in the $(h+1)$-th step when Player 1 takes action $a\in\cA$ in an episode, $r_h : \cS \times \cA \times \cB \mapsto [0, 1]$ denotes the reward function at the $h$-step, with the value normalized in the range $[0, 1]$. In this paper, we let $\cP=\{\cP_h\}_{h=1}^H$ be the \emph{true} transition model, which is \emph{unknown} to both players. Throughout this paper, we let $\langle\cdot, \cdot \rangle_\cS$, $\langle \cdot, \cdot \rangle_\cA$, and $\langle \cdot, \cdot \rangle_\cB$ denote the inner product  over $\cS$, $\cA$, and $\cB$ respectively.
 
The policy of Player 1 is a collection of probability distributions $\mu = \{ \mu_h\}_{h=1}^H$ where $\mu_h( a | s) \in \Delta_\cA$ with $\Delta_\cA$ denoting a probability simplex defined on space $\cA$.  Analogously, we have the policy of Player 2 as a collection of probability distributions $\nu = \{ \nu_h\}_{h=1}^H $, where $\nu_h(b|s) \in \Delta_\cB$ with $\Delta_\cB$ denoting the probability simplex on space $\cB$. We denote $\mu^k = \{ \mu^k_h\}_{h=1}^H$ and $\nu^k = \{ \nu^k_h\}_{h=1}^H$ as the policies at episode $k$ for Players 1 and 2. 

\vspace{3pt}
\noindent\textbf{Fictitious Play.} At the beginning of the $k$-th episode, each player observes the opponent's policy during the $(k-1)$-th episode. For simplicity of theoretical analysis, here we assume there exists an oracle such that each player can obtain the opponent's past policy. Then, they update their own policies for this episode and make simultaneous moves. By the end of the $k$-th episodes, each player observes only the trajectory $\{(s^k_h, a^k_h, b^k_h, s_{h+1}^k)\}_{h=1}^H$ and the bandit feedback along the trajectory.  The bandit setting is more challenging than the full-information setting, where only the reward values $\{r_h^k (s^k_h, a^k_h, b_h^k) \}_{h=1}^H$ on the trajectory are observed rather than the exact value function $r_h(s,a,b)$ for all $(s,a,b)\in \cS\times\cA\times \cB$. Moreover, the rewards $r_h^k(\cdot, \cdot,\cdot) \in [0,1]$ is time-varying with its expectation $r_h= \EE[r_h^k]$ which can be adversarially affected by the opponent's action or policy, indicating the non-stationarity of the environment.

\vspace{3pt}
\noindent\textbf{Value Function.} We define the value function $V_h^{\mu, \nu}: \cS \mapsto \RR$ under any policies $\mu = \{ \mu_h \}_{h=1}^H$, $\nu = \{ \nu_h \}_{h=1}^H$ and the transition model $\cP = \{ \cP_h \}_{h=1}^H$ by $
V_h^{\mu, \nu}(s) := \EE[ \sum_{h'=h}^H r_{h'} (s_{h'}, a_{h'}, b_{h'})  \given s_h = s ]$, where the expectation is taken over the random state-action pairs $\{(s_{h'}, a_{h'}, b_{h'})\}_{h'=h}^H$. The corresponding action-value function (Q-function)  $Q_h^{\mu, \nu}: \cS \times \cA \times \cB \mapsto \RR$ is then defined as $Q_h^{\mu, \nu}(s, a, b) := \EE[ \sum_{h'=h}^H r_{h'}(s_{h'}, a_{h'}, b_{h'} )  \given s_h = s, a_h = a,  b_h = b]$. Therefore, according to the above definitions, we have the following Bellman equation
\begin{align}
&V_h^{\mu, \nu}(s) = [\mu_h(\cdot | s)]^\top Q_h^{\mu, \nu}(s, \cdot , \cdot) \nu_h(\cdot | s), \label{eq:bellman_V}\\
&Q_h^{\mu, \nu}(s, a, b) = r_h (s, a, b) +  \big\langle  \cP_h(\cdot | s, a, b), V_{h+1}^{\mu, \nu} (\cdot) \big\rangle_\cS, \label{eq:bellman_Q}
\end{align}
where $\mu_h(\cdot | s)$ and $\nu_h(\cdot | s)$ are column vectors over the space $\cA$ and the space $\cB$ respectively,  $V_{h+1}^{\mu, \nu} (\cdot)$ is a column vector over the space $\cS$, and $Q_h^{\mu, \nu}(s, \cdot , \cdot)$ is a matrix over the space $\cA \times \cB$. The above Bellman equation holds for all $h \in [H]$ with setting $V_{H+1}^{\mu, \nu} (s) = 0, \forall s \in \cS$. Hereafter, to simplify the notation, we let $\cP V( s, a, b) :=\langle\cP(\cdot | s, a, b), V(\cdot) \big\rangle_\cS$ for any value function $V$ and transition $\cP$. 


\vspace{3pt}
\noindent\textbf{Regret and Optimality Gap.} The goal for Player 1 is to learn a sequence of policies, $\{\mu^k\}_{k > 0}$, to have a small regret as possible in $K$ episodes, which is defined as 
\begin{align}
\Regret_1(K) := \sum_{k=1}^K \Big[ V_1^{\mu^*, \nu^k}(s_1) - V_1^{\mu^k, \nu^k}(s_1) \Big], \label{eq:regret_1}
\end{align}
and $\{\nu^k\}_{k=1}^K$ is any possible and potentially adversarial policy sequence of Player 2. The policy $\mu^*$ is \emph{the best policies in hindsight}, which is defined as $\mu^* := \argmax_{\mu} \sum_{k=1}^K V_1^{\mu , \nu^k}(s_1)$ for any specific $\{\nu^k\}_{k=1}^K$. Similarly, Player 2 aims to learn a sequence of policies, $\{\nu^k\}_{k > 0}$, to have a small regret defined as
\begin{align}
\Regret_2(K) := \sum_{k=1}^K \Big[ V_1^{\mu^k, \nu^k}(s_1) - V_1^{\mu^k, \nu^*}(s_1) \Big]. \label{eq:regret_2}
\end{align}
where $\{\mu^k\}_{k=1}^K$ is any possible policy sequence of Player 1. The policies $\nu^*$ is also \emph{the best policies in hindsight} which is defined as $\nu^* := \argmin_{\nu} \sum_{k=1}^K V_1^{\mu^k, \nu}(s_1)$ for any specific $\{\mu^k\}_{k=1}^K$. Note that $\mu^*$ and $\nu^*$ depend on opponents' policy sequence and is non-deterministic, and we drop such a dependency in the notation for simplicity.
We further define the \emph{optimality gap} $\Gap(K)$ as follows
\begin{align}
\label{eq:regret}
\Gap(K) :=&\Regret_1(K) + \Regret_2(K). 
\end{align} 
Our definition of optimality gap is consistent with a certain form of the regret to measure the learning performance of zero-sum games defined in \citet[Definition 8]{bai2020provable}. Specifically, when the two players executes their algorithms to have small regrets, i.e., $\Regret_1(K)$ and $\Regret_2(K)$ are small, then their optimality gap $\Gap(K)$ is small as well. 

On the other hand, letting the uniform mixture policies $\hat{\pi} \sim \mathrm{Unif}(\pi^1, \ldots, \pi^K)$ and $\hat{\nu} \sim \mathrm{Unif}(\nu^1, \ldots, \nu^K)$ be random policies sampled uniformly from the learned policies, then $(\hat{\pi}, \hat{\nu})$ can be viewed as an $\varepsilon$-approximate NE if $ \Regret(K)/K \leq \varepsilon$. This build a connection between the approximate NE and the optimality gap.

\subsection{Structured Transition Models}

\noindent \textbf{Factored Independent Transition.} Consider a two-player MG where the state space are factored as $\cS = \cS_1 \times \cS_2$ such that a state can be represented as $s = (s^1, s^2)$ with $s^1\in \cS_1$ and $s^2\in \cS_2$. Moreover, the size of the space $\cS$ is $|\cS| = |\cS_1|\cdot |\cS_2|$. Under this setting, the transition model is factored into two independent components, i.e.,  
\begin{align}\label{eq:ind_trans}
\cP_h(s'\given s, a, b) = \cP^1_h(s^1{}'\given s^1, a) \cP^2_h( s^2{}'\given  s^2, b),
\end{align}
where we also have $s'=(s^1{}', s^2{}')$, and $\cP_h(s^1{}'\given s^1, a)$ is the transition model for Player 1 and $\cP_h( s^2{}'\given  s^2, b)$ for Player 2. Additionally, we consider the case where the policy of Player 1 only depends on its own state $s^1$ such that we have $\mu(a|s) = \mu(a|s^1)$ and meanwhile Player 2 similarly has the policy of the form $\nu(b|s) = \nu(b|s^2)$. Though the transitions, policies, and state spaces of two players are independent of each other, the reward function still depends on both players' actions and states, i.e., $r_h(s,a,b) = r_h(s^1,s^2,a,b)$. 

\vspace{3pt}
\noindent \textbf{Single-Controller Transition.} In this setting, we take steps forward by not assuming the relatively independent structures of the policies and state spaces for two players. For the single-controller game, we consider that the transition model is controlled by the action of one player, e.g., Player 1 in this paper, which is thus characterized by 
\begin{align} \label{eq:sc_trans}
\cP_h(s'\given s, a, b) = \cP_h(s'\given s, a).
\end{align}
In addition, the policies remain to be $\mu(a|s)$ and $\nu(b|s)$ that depend on the state $s$ jointly decided by both players,  and reward $r_h(s,a,b)$ is determined by both players as well.  

\begin{remark}[Misspecification] 
When the above models are not ideally satisfied, one can potentially consider scenarios that the transition model satisfies, for example, $\max_{s'}\allowbreak|\PP_h(s'\given s, a, b) - \PP^1_h(s^1{}'\given s^1, a) \PP^2_h( s^2{}'\given  s^2, b)| \leq \varrho$ or $\max_{s'} |\PP_h(s'\given s, a, b)  - \PP_h(s'\given s, a)| \leq \varrho$, $\forall (s,a,b, h)$, with a misspecification error $\varrho$. One can still follow the techniques in this paper to analyze such misspecified scenarios and obtain regrets with an extra bias term depending on the misspecification error $\varrho$. When $\varrho$ is small, it implies that the MG  has approximately factored independent transition or single-controller transition structures, and then the bias term depending on $\varrho$ should be small. 
\end{remark}

\section{MG with Factored Independent Transition}\label{sec:FIT}
In this section, we propose optimistic policy optimization algorithms for both players under the setting of factored independent transition.

\vspace{3pt} 
\noindent\textbf{Algorithm for Player 1.} The algorithm for Player 1 is illustrated in Algorithm \ref{alg:po1_it}. Assume that the game starts from a fixed state $s_1= (s^1_1, s^2_1)$ each round. We also assume that the true transition model $\cP$ is not known to Player 1, and Player 1 can only access the bandit feedback of the rewards along this trajectory instead of the full information. Thus, Player 1 needs to empirically estimate the reward function and the transition model for all $(s,a,b,s')$ and $h\in [H]$  via 
\begin{align}
\begin{aligned}\label{eq:estimate_it}
&\hat{r}_h^k(s,a,b) = \frac{\sum_{\tau = 1}^k \mathbbm{1}_{\{(s,a,b) = (s_h^\tau, a_h^\tau, b_h^\tau)\}  } r^k_h(s,a,b)}{\max\{ N_h^k(s,a, b), 1\}}, \\
&\hat{\cP}_h^{1,k}(s^1{}' |s^1, a) = \frac{\sum_{\tau = 1}^k \mathbbm{1}_{\{(s^1,a, s^1{}') = (s_h^{1,\tau}, a_h^\tau, s_{h+1}^{1,\tau})\}}}{\max\{ N_h^k(s^1,a), 1\}},\\
&\hat{\cP}_h^{2,k}(s^2{}' |s^2, b) = \frac{\sum_{\tau = 1}^k \mathbbm{1}_{\{(s^2,b, s^2{}') = (s_h^{2,\tau}, b_h^\tau, s_{h+1}^{2,\tau})\}}}{\max\{ N_h^k(s^2, b), 1\}},
\end{aligned}
\end{align}
where we denote $\mathbbm{1}_{\{\cdot\}}$ as an indicator function, and $N^k_h(s,a, b)$ counts the empirical number of observation for a certain tuple $(s,a, b)$ at step $h$ until $k$-th iteration as well as  $N^k_h(s^1,a)$ for $(s^1, a)$ and $N^k_h(s^2,b)$ for $(s^2, b)$. Then, we have the estimation of the overall transition as $\hat{\cP}_h^k(s'|s,a,b) = \hat{\cP}_h^{1,k}(s^1{}' |s^1, a)  \hat{\cP}_h^{2,k}(s^2{}' |s^2, b)$. For simplicity of presentation, in this section, we let $s = (s^1,s^2)$ and we use $s^1,s^2$ when necessary.

Based on the estimation of the transition model and reward function, we further estimate the Q-function and value-function as shown in Line 7 and 8 in Algorithm \ref{alg:po1_it}. In terms of the principle of \emph{``optimism in the face of uncertainty''}, bonus terms are introduced to construct a UCB update for Q-function as shown in Line 7 of Algorithm \ref{alg:po1_it}. Here, we can set the bonus term as
\begin{align}
\beta_h^{k}(s,a,b) = \beta_h^{r, k}(s,a,b) + \beta_h^{\cP, k}(s,a, b), \label{eq:bonus_decomp_it}
\end{align}
where we define $\beta_h^{r, k}(s,a,b) := \sqrt{\frac{4
\log (|\cS_1| |\cS_2| |\cA| |\cB| H K /\delta)}{ \max \{ N^k_h(s ,a,b), 1\} }}$ as well as $
\beta_h^{\cP, k}(s,a,b) := \sqrt{\frac{2H^2|\cS_1| \log (2|\cS_1| |\cA| H K/\delta)}{\max \{N_h^k(s^1,a), 1\}}} +  \sqrt{\frac{2H^2|\cS_2| \log (2|\cS_2| |\cB| H K/\delta)}{\max \{N_h^k(s^2,b), 1\}}}$ with $\delta \in (0, 1)$. Here, we decompose $\beta_h^{k}(s,a,b)$ into two terms where $\beta_h^{r, k}(s,a,b)$ is the bonus term for the reward and $\beta_h^{\cP, k}(s,a)$ for the transition estimation. As shown in Lemmas \ref{lem:r_bound_it} and \ref{lem:P_bound_it} of the supplementary material, the bonus terms $\beta_h^{r, k}(s,a,b)$ and $\beta_h^{\cP, k}(s,a,b)$ are obtained by using Hoeffding's inequality. Note that the two terms in the definition of $\beta_h^{\cP, k}$ stem from the uncertainties of estimating the transitions $\cP^1_h(s^1{}'\given s^1, a)$ and $\cP^2_h( s^2{}'\given  s^2, b)$.

Here we introduce the notion of the state reaching probability $q^{\nu^k, \cP^2}(s^2)$ for any state $s^2\in \cS_2$ under the policy $\nu^k$ and the true transition $\cP^2$, which is defined as 
\begin{align*}
q_h^{\nu^k, \cP^2}(s^2):=\Pr(s^2_h = s^2 \given \nu^k, \cP^2, s^2_1), \forall h\in [H].
\end{align*}
To handle non-stationarity of the opponent, as in Line 10, Player 1 needs to estimate the state reaching probability of Player 2 by the empirical reaching probability under the empirical transition model $\hat{\cP}^{2,k}$ for Player 2, i.e., 
\begin{align*}
d_h^{\nu^k,\hat{\cP}^{2,k}} (s^2) = \Pr(s^2_h = s^2 \given  \nu^k, \hat{\cP}^{2,k}, s^2_1), \forall h\in [H]. 
\end{align*}
The empirical reaching probability can be simply computed dynamically from $h=1$ to $H$ by $d_h^{\nu^k,\hat{\cP}^{2,k}} (s^2)  = \sum_{s^2{}'\in \cS_2} \sum_{a'\in \cA} d_{h-1}^{\nu^k,\hat{\cP}^{2,k}} (s^2{}')  \nu_{h-1}^k(b'|s^2{}') \hat{\cP}_{h-1}^{2,k}(s^2|s^2{}',b')$. 

Based on the estimated state reaching probability, the policy improvement step is associated with solving the following optimization problem
\begin{align}  \label{eq:ascent_it}
\max_{\mu} \sum_{h=1}^H  [\overline{G}_h^{k-1}(\mu_h) - \eta^{-1}  D_{\mathrm{KL}}( \mu_h(\cdot| s^1), \mu_h^{k-1}(\cdot| s^1 ))], 
\end{align}
where we define the linear function as $\overline{G}_h^{k-1}(\mu_h) := \langle \mu_h(\cdot | s^1)-\mu_h^{k-1}(\cdot | s^1), \sum_{s^2 \in \cS_2} F_h^{1,k-1}(s, \cdot) \cdot \allowbreak d_h^{\nu^{k-1},\hat{\cP}^{2,k-1}}(s^2) \rangle_\cA$ with $F_h^{1,k-1}(s, a) = \langle \overline{Q}_h^{k-1}(s, a, \cdot),  \nu_h^{k-1}(\cdot | s^2)\rangle_{\cB}$. One can see that \eqref{eq:ascent_it} is a mirror ascent step and has a solution as $\mu_h^k(a|s^1) = (Y_h^{k-1})^{-1} \mu_h^{k-1}(a\given s^1)\cdot  \exp\{ \eta  \sum_{s^2 \in \cS_2} F_h^{1,k-1}(s, a) \cdot \allowbreak d_h^{\nu^{k-1},\hat{\cP}^{2,k-1}}(s^2)\}$, where $Y_h^{k-1}$ is a probability normalization term.

\begin{algorithm}[t]\caption{Optimistic Policy Optimization for Player 1 with Factored Independent Transition} 
   \setstretch{1.0}
	\begin{algorithmic}[1]
		\State {\bfseries Initialize:} For all $h\in [H]$, $(s^1, s^2, a, b)\in \cS_1 \times \cS_2 \times \cA \times \cB$: $\mu_h^0(\cdot|s^1) = \boldsymbol{1}/ |\cA|$, $\hat{\cP}_h^{1,0}(\cdot|s^1,a) = \boldsymbol{1}/ |\cS_1|$, $\hat{\cP}_h^{2,0}(\cdot|s^2,b)  = \boldsymbol{1}/ |\cS_2|$, $\hat{r}_h^0(\cdot,\cdot,\cdot) = \beta_h^0(\cdot,\cdot,\cdot) = \boldsymbol{0}$. 
		\For{episode $k=1,\ldots,K$}   
			\State Observe Player 2's policy $ \{\nu_h^{k-1}\}_{h=1}^H $.	
		   	\State Start from state $ s_1= (s^1_1, s^2_1)$, set $\overline{V}_{H+1}^{k-1}(\cdot) = \boldsymbol 0$. 
		        \For{step $h=H, H-1,\ldots, 1$}  \Comment{{\color{blue} Policy Evaluation}}
					\State Estimate the transition and reward function by $\hat{\cP}_h^{k-1}(\cdot |\cdot, \cdot)$ and $\hat{r}^{k-1}_h(\cdot, \cdot, \cdot)$ as \eqref{eq:estimate}. 			
					\State \label{line:Q_up_it1}Update Q-function $\forall (s, a, b) \in \cS \times \cA \times \cB$:
						\begin{align*}
							&\overline{Q}_h^{k-1}(s,a, b) =  \min \{  ( \hat{r}^{k-1}_h + \hat{\cP}_h^{k-1}\overline{V}_{h+1}^{k-1}   + \beta_h^{k-1})(s,a,b), H-h+1\}^+.  
						\end{align*} 
					\State  \label{line:V_up_it1}Update value-function $\forall s \in \cS$: 
						\begin{align*}						
							\overline{V}_h^{k-1}(s) =  \big[\mu_h^{k-1}(\cdot | s)\big]^\top \overline{Q}_h^{k-1}(s, \cdot , \cdot) \nu_h^{k-1}(\cdot | s). 
						\end{align*}
	            \EndFor	    				
				\State \label{line:reach_it1}
				Estimate the state reaching probability of Player 2 by $d_h^{\nu^{k-1},\hat{\cP}^{2,k-1}} (s^2)$, $\forall s^2\in \cS_2, h\in [H]$.
				
					\State Update policy $\mu_h^k(a|s^1)$ by solving \eqref{eq:ascent_it}, $\forall (s^1, a, h)$.	 \Comment{{\color{blue}Policy Improvement}}
				\State Take actions following $a_h^k \sim \mu_h^{k}(\cdot | s_h^{1,k}),\  \forall h \in [H]$. 	
				\State Observe the trajectory $\{(s^k_h, a^k_h, b^k_h, s_{h+1}^k)\}_{h=1}^H$, and  rewards $\{r^k_h(s^k_h, a^k_h, b_h^k) \}_{h=1}^H$.   						                 
    \EndFor             
	\end{algorithmic}\label{alg:po1_it}
\end{algorithm}

\vspace{3pt}

\noindent\textbf{Algorithm for Player 2.} For the setting of MG with factored independent transition, the algorithm for Player 2 is trying to minimize the expected cumulative reward w.r.t. $r_h(\cdot,\cdot,\cdot)$. In another word, Player 2 is maximizing the expected cumulative reward w.r.t. $-r_h(\cdot,\cdot,\cdot)$. From this perspective, one can view the algorithm for Player 2 as a \emph{`symmetric'} version of Algorithm \ref{alg:po1_it}. We summarized the optimistic policy optimization algorithm for Player 2 as in Algorithm \ref{alg:po2_it}. Specifically, in this algorithm, Player 2 also estimates the transition model and the reward function the same as \eqref{eq:estimate}.  Since Player 2 is minimizing the expected cumulative reward, the bonus terms as \eqref{eq:bonus_decomp_it} are subtracted in the Q-function estimation step by the UCB optimism principle. The algorithm further estimates the state reaching probability of Player 1, $q_h^{\mu^k, \cP^1}(s^1)$, by the empirical one $d_h^{\mu^k,\hat{\cP}^{1,k}} (s^1)$, which can be dynamically computed. For the policy improvement step, Algorithm \ref{alg:po2_it} performs a mirror descent step based on the empirical reaching probability.  Based on the  empirical state reaching probability, the policy improvement step is associated with solving the following optimization problem
\begin{align}  \label{eq:descent_it}
\max_{\mu} \sum_{h=1}^H  [\underline{G}_h^{k-1}(\nu_h) + \gamma^{-1}  D_{\mathrm{KL}}( \nu_h(\cdot| s^2), \nu_h^{k-1}(\cdot| s^2 ))], 
\end{align}
where we define $\underline{G}_h^{k-1}(\mu_h) := \langle \nu_h(\cdot | s^2)-\nu_h^{k-1}(\cdot | s^2), \sum_{s^1 \in \cS_1} F_h^{2,k-1}(s, \cdot)  d_h^{\mu^{k-1},\hat{\cP}^{1,k-1}}(s^1) \rangle_\cB$ with $F_h^{2,k-1}$ defined as $F_h^{2,k-1}(s, b) = \langle \underline{Q}_h^{k-1}(s, \cdot, b),  \mu_h^{k-1}(\cdot | s^1)\rangle_{\cA}$ where $s=(s^1,s^2)$. Here \eqref{eq:descent_it}  is a standard mirror descent step and admits a closed-form solution as $\nu_h^k(b|s^2) = (\tilde{Y}_h^{k-1})^{-1} \nu_h^{k-1}(b\given s^2)\cdot \exp\{ - \gamma \sum_{s^1 \in \cS_1} F_h^{2,k-1}(s,b) \cdot \allowbreak d_h^{\mu^{k-1},\hat{\cP}^{1,k-1}}(s^1) \}$, where $\tilde{Y}_h^{k-1}$ is a probability normalization term.

\begin{algorithm}[t]\caption{Optimistic Policy Optimization for Player 2 with Factored Independent Transition} 
   \setstretch{1.0}
	\begin{algorithmic}[1]
		\State {\bfseries Initialize:} For all $h\in [H]$, $(s^1, s^2, a, b)\in \cS_1 \times \cS_2 \times \cA \times \cB$: $\mu_h^0(\cdot|s^1) = \boldsymbol{1}/ |\cA|$, $\hat{\cP}_h^{1,0}(\cdot|s^1,a) = \boldsymbol{1}/ |\cS_1|$, $\hat{\cP}_h^{2,0}(\cdot|s^2,b)  = \boldsymbol{1}/ |\cS_2|$, $\hat{r}_h^0(\cdot,\cdot,\cdot) = \beta_h^0(\cdot,\cdot,\cdot) = \boldsymbol{0}$. 
		\For{episode $k=1,\ldots,K$}   
			\State Observe Player 1's policy $ \{\mu_h^{k-1}\}_{h=1}^H $.	
		   	\State Start from state $s_1= (s^1_1, s^2_1)$, set $\overline{V}_{H+1}^{k-1}(\cdot) = \boldsymbol 0$. 
		        \For{step $h=H, H-1,\ldots, 1$} \Comment{{\color{blue}Policy Evaluation}}
					\State Estimate the transition and reward function by $\hat{\cP}_h^{k-1}(\cdot |\cdot, \cdot)$ and $\hat{r}^{k-1}_h(\cdot, \cdot, \cdot)$ as \eqref{eq:estimate}. 			
					\State \label{line:Q_up_it2}Update Q-function $\forall (s, a, b) \in \cS \times \cA \times \cB$:
						\begin{align*}
							&\underline{Q}_h^{k-1}(s,a, b) =  \min \{   (\hat{r}^{k-1}_h + \hat{\cP}_h^{k-1}\underline{V}_{h+1}^{k-1}  - \beta_h^{k-1})(s,a,b), H-h+1\}^+.  
						\end{align*} 
					\State  \label{line:V_up_it2}Update value-function $\forall s \in \cS$: 
						\begin{align*}						
							\underline{V}_h^{k-1}(s) =  \big[\mu_h^{k-1}(\cdot | s)\big]^\top \underline{Q}_h^{k-1}(s, \cdot , \cdot) \nu_h^{k-1}(\cdot | s). 
						\end{align*}
	            \EndFor	    				
				\State \label{line:reach_it2}
				Estimate the state reaching probability of Player 1 by $d_h^{\mu^{k-1},\hat{\cP}^{1,k-1}} (s^1)$,  $\forall s^1\in \cS_1, h\in [H]$.
					\State Update policy $\nu_h^k(b|s^2)$ by solving \eqref{eq:descent_it}, $\forall (s^2, b, h)$.	 	\Comment{{\color{blue}Policy Improvement}}		
				\State Take actions following $b_h^k \sim \nu_h^{k}(\cdot | s_h^{2,k}),\  \forall h \in [H]$. 	
				\State Observe the trajectory $\{(s^k_h, a^k_h, b^k_h, s_{h+1}^k)\}_{h=1}^H$, and  rewards $\{r^k_h(s^k_h, a^k_h, b_h^k) \}_{h=1}^H$.   			                  
    \EndFor             
	\end{algorithmic}\label{alg:po2_it}
\end{algorithm}

\subsection{Theoretical Results} \label{sec:FIT_result}
In this subsection, we show our main results of the upper bounds of the regrets for each player under the setting of the factored independent transition model.
\begin{theorem} \label{thm:main_FIT1} By setting $\eta = \sqrt{  \log |\cA|/(KH^2)}$,  with probability at least $1-4\delta$, Algorithm \ref{alg:po1_it} ensures the sublinear regret bound for Player 1\footnote{Hereafter, we use $\tilde{\cO}$ to hide the logarithmic factors on $|\cS|, |\cA|, |\cB|,H, K$, and $1/\delta$.}, i.e., $\Regret_1(K) \leq  \tilde{\cO} \big( C\sqrt{T} \big)$, where $T = KH$ denotes the total rounds, and the constant $C = \sqrt{(|\cS_1|^2|\cA| + |\cS_2|^2|\cB|)H^3} + \sqrt{|\cS_1||\cS_2||\cA||\cB|H}$.
\end{theorem}

Theorem \ref{thm:main_FIT1} shows that Player 1 can obtain an $\tilde{\cO}(\sqrt{K})$ regret by Algorithm \ref{alg:po1_it}, when the opponent, Player 2, takes actions following potentially adversarial policies. 

\begin{theorem} \label{thm:main_FIT2} By setting $\gamma = \sqrt{  \log |\cB|/(KH^2)}$, with probability at least $1-4\delta$, Algorithm \ref{alg:po2_it} ensures the sublinear regret bound for Player 2, i.e., $\Regret_2(K) \leq \tilde{\cO} \big( C\sqrt{T} \big)$, where $T = KH$ denotes the total rounds, and the constant $C = \sqrt{(|\cS_1|^2|\cA| + |\cS_2|^2|\cB|)H^3} + \sqrt{|\cS_1||\cS_2||\cA||\cB|H}$.
\end{theorem}
Theorem \ref{thm:main_FIT2} shows that $\Regret_2(K)$ admits the same $\tilde{\cO}(\sqrt{T})$ regret as Theorem \ref{thm:main_FIT1} given any arbitrary and adversarial policies of the opponent Player 1, due to the symmetric nature of the two algorithms.

From the perspective of each individual player, the game can be viewed as a special case of an MDP with adversarial bandit feedback due to the potentially adversarial actions or policies of the opponent.  For MDPs with adversarial bandit feedback, \citet{jin2019learning} attains an $\tilde{\cO}(\sqrt{T})$ regret via an occupancy measure based method, which requires solving a constrained optimization problem in each update step that is more computationally demanding than PO. \citet{efroni2020optimistic} proposes a PO method for the same MDP model, achieving an $\tilde{\cO}(T^{2/3})$ regret. Thus, directly applying this result would yield an $\tilde{\cO}(T^{2/3})$ regret. However, for the problem of zero-sum games, regarding the problem faced by one player as an MDP with adversarial rewards neglects the fact that such ``adversarial reward functions'' are determined by the actions and policies of the opponent. Thus, under the FP framework, by utilizing the past actions and policies of the opponent, Algorithm \ref{alg:po1_it} and \ref{alg:po2_it} obtain an $\tilde{\cO}(\sqrt{T})$ regret.

In particular, if Player 1 runs Algorithm \ref{alg:po1_it} and Player 2 runs Algorithm \ref{alg:po2_it} \emph{simultaneously}, then we have the following corollary of Theorems \ref{thm:main_FIT1} and \ref{thm:main_FIT2}.

\begin{corollary} \label{coro:main_FIT} By setting $\eta$ and $\gamma$ as in Theorem \ref{thm:main_FIT1} and Theorem \ref{thm:main_FIT2}, letting $T = K H$, with probability at least $1-8\delta$,  Algorithm \ref{alg:po1_it} and Algorithm \ref{alg:po2_it} ensure the following optimality gap $\Gap(K) \leq  \tilde{\cO} \big( \sqrt{T } \big)$.
\end{corollary}

\section{MG with Single-Controller Transition}\label{sec:SCT}

In this section, we propose optimistic policy optimization algorithms for the single-controller game.
 
\vspace{3pt} 
\noindent\textbf{Algorithm for Player 1.} The algorithm for Player 1 is illustrated in Algorithm \ref{alg:po1}. Since transition model is unknown and only bandit feedback of the rewards is available, Player 1 needs to empirically estimate the reward function and the transition model for all $(s,a,b,s')$ and $h\in [H]$ via 
\begin{align}
\begin{aligned}\label{eq:estimate}
&\hat{r}_h^k(s,a,b) = \frac{\sum_{\tau = 1}^k \mathbbm{1}_{\{(s,a,b) = (s_h^\tau, a_h^\tau, b_h^\tau)\}  } r^k_h(s,a,b)}{\max\{ N_h^k(s,a, b), 1\}}, \\
&\hat{\cP}_h^k(s' |s, a) = \frac{\sum_{\tau = 1}^k \mathbbm{1}_{\{(s,a, s') = (s_h^\tau, a_h^\tau, s_{h+1}^\tau)\}}}{\max\{ N_h^k(s,a), 1\}}.
\end{aligned}
\end{align}
Based on the estimations, Algorithm \ref{alg:po1} further estimates the Q-function and value-function for policy evaluation. In terms of the optimism principle, bonus terms are added to construct a UCB update for Q-function as shown in Line 7 of Algorithm \ref{alg:po1}. The bonus terms are computed as
\begin{align}
\beta_h^{k}(s,a,b) = \beta_h^{r, k}(s,a,b) + \beta_h^{\cP, k}(s,a), \label{eq:bonus_decomp}
\end{align}
where the two bonus terms above are expressed as $\beta_h^{r, k}(s,a,b) := \sqrt{\frac{ 4\log (|\cS| |\cA| |\cB| HK/\delta)}{\max\{N^k_h(s,a, b), 1\}}}$ and $\beta_h^{\cP, k}(s,a) := \sqrt{\frac{2H^2 |\cS| \log (|\cS||\cA|HK/\delta)}{ \max\{N_h^k(s,a), 1\}}}$ for $\delta \in (0, 1)$. Here we also decompose $\beta_h^{k}(s,a,b)$ into two terms with $\beta_h^{r, k}(s,a,b)$ denoting the bonus term for the reward and $\beta_h^{\cP, k}(s,a)$ for the transition estimation. Note that the transition bonus are only associated with $(s,a)$ due to the single-controller structure. The bonus terms are derived in Lemmas \ref{lem:r_bound} and \ref{lem:P_bound} of the supplementary material.

Different from Algorithm \ref{alg:po1_it}, in this algorithm for Player 1, there is no need to estimate the state reaching probability of the opponent as the transition only depends on Player 1. The policy improvement step is then associated with solving the following optimization problem
\begin{align} 
&\max_{\mu} \sum_{h=1}^H [ \overline{L}_h^{k-1}(\mu_h)  - \eta^{-1}  D_{\mathrm{KL}}\big( \mu_h(\cdot| s), \mu_h^{k-1}(\cdot| s ) ) ] , \label{eq:ascent}
\end{align} 
where we define the function $\overline{L}_h^{k-1}(\mu_h) :=  \big[\mu_h(\cdot|s) - \mu_h^{k-1}(\cdot|s)\big]^\top \overline{Q}_h^{k-1}(s, \cdot, \cdot) \nu_h^{k-1}(\cdot|s)$. This is a mirror ascent step with the solution $\mu_h^k(a|s) =  (Z_h^{k-1})^{-1}   \mu_h^{k-1}(a\given s)\exp\{ \eta  \big \langle \overline{Q}_h^{k-1}(s, a , \cdot), \nu_h^{k-1}(\cdot\given s) \big\rangle_\cB \}$,
where $Z_h^{k-1}$ i s a probability normalization term.
 
 \begin{algorithm}[t]\caption{Optimistic Policy Optimization for Player 1 with Single-Controller Transition} 
   \setstretch{1}
	\begin{algorithmic}[1]
		\State {\bfseries Initialize:} $\mu_h^0(\cdot|s) = \boldsymbol{1}/ |\cA|$ for all $s\in \cS$ and $h\in [H]$. $\hat{\cP}_h^0(\cdot|s,a) = \boldsymbol{1}/ |\cS|$ for all $(s, a)\in \cS \times \cA$ and $h\in [H]$. $\hat{r}_h^0(\cdot,\cdot,\cdot) = \beta_h^0(\cdot,\cdot,\cdot) = \boldsymbol{0}$ for all $h\in [H]$. 
		\For{episode $k=1,\ldots,K$}   
			\State Observe Player 2's policy $ \{\nu_h^{k-1}\}_{h=1}^H $.	
		   	\State Start from $s_1^k= s_1$, and set $\overline{V}_{H+1}^{k-1}(\cdot) = \boldsymbol 0$. 
		        \For{step $h=H, H-1,\ldots, 1$} \Comment{{\color{blue}Policy Evaluation}}
					\State Estimate the transition and reward function by $\hat{\cP}_h^{k-1}(\cdot |\cdot, \cdot)$ and $\hat{r}^{k-1}_h(\cdot, \cdot, \cdot)$ as \eqref{eq:estimate}. 			
					\State \label{line:Q_up}Update Q-function $\forall (s, a, b) \in \cS \times \cA \times \cB$:
							\begin{align*}
							\overline{Q}_h^{k-1}(s,a, b) =  \min \{   \hat{r}^{k-1}_h(s,a,b) + \hat{\cP}_h^{k-1}\overline{V}_{h+1}^{k-1}(s,a)  + \beta_h^{k-1}(s,a,b), H-h+1 \}^+  
						\end{align*} 
					\State  \label{line:V_up}Update value-function $\forall s \in \cS$: 
						\begin{align*}						
							\overline{V}_h^{k-1}(s) =  \big[\mu_h^{k-1}(\cdot | s)\big]^\top \overline{Q}_h^{k-1}(s, \cdot , \cdot) \nu_h^{k-1}(\cdot | s). 
						\end{align*}
	            \EndFor	    				

					\State Update policy $\mu_h^k(a|s)$ by solving \eqref{eq:ascent}, $\forall (s, a, h)$. \Comment{{\color{blue}Policy Improvement}}
				\State Take actions following $a_h^k \sim \mu_h^{k}(\cdot | s_h^k),\  \forall h \in [H]$. 	
				\State Observe the trajectory $\{(s^k_h, a^k_h, b^k_h, s_{h+1}^k)\}_{h=1}^H$, and  rewards $\{r^k_h(s^k_h, a^k_h, b_h^k) \}_{h=1}^H$.   						                 
    \EndFor             
	\end{algorithmic}\label{alg:po1}
\end{algorithm}

\vspace{3pt} 
\noindent\textbf{Algorithm for Player 2.} The algorithm for Player 2 is illustrated in Algorithm \ref{alg:po2}. Player 2 also estimates the transition model and the reward function the same as \eqref{eq:estimate}. However, due to the asymmetric nature of the single-controller transition model, Player 2 has a different way to learning the policy. The main differences to Algorithm \ref{alg:po1} are summarized in the following three aspects: First, according to our theoretical analysis shown in Lemma \ref{lem:V_diff_2}, no transition model estimation is involved. Instead, only a reward function estimation is considered in Line 7 of Algorithm \ref{alg:po2}. Second, in the policy improvement step, Player 2 needs to approximate the state reaching probability  $q_h^{\mu^k, \cP}(s) := \Pr(s_h = s \given  \mu^k, \cP, s_1 )$ under $
\mu^k$ and true transition $\cP$ by the empirical reaching probability $d_h^k(s) = \Pr(s_h = s \given  \mu^k, \hat{\cP}^k, s_1 )$ with the empirical transition model $\hat{\cP}^k$, which can also be computed dynamically from $h=1$ to $H$.
Third, we subtract a reward bonus term $\beta^{r,k-1}_h$ in Line 7 instead of adding the bonus. Similar to our discussion in Section \ref{sec:FIT}, it is still a UCB estimation if viewing Player 2 is maximizing the cumulative reward w.r.t. $-r_h(\cdot, \cdot, \cdot)$.

Particularly, the policy improvement step of Algorithm \ref{alg:po2} is associated with solving the following minimization problem
\begin{align} 
\min_{\nu} \sum_{h=1}^H \{ \underline{L}^{k-1}_h(\nu_h) + \gamma^{-1} D_{\mathrm{KL}}\big( \nu_h(\cdot| s), \nu_h^{k-1}(\cdot| s) \big)\} ,\label{eq:descent} 
\end{align} 
where we define $\underline{L}^{k-1}_h(\nu_h) := d^{k-1}_h(s) [\mu_h^{k-1}(\cdot|s) ]^\top \cdot \allowbreak \tilde{r}_h^{k-1}(s, \cdot, \cdot)  [\nu_h(\cdot|s)-\nu_h^{k-1}(\cdot|s)]$. This is a mirror descent step with the solution $\nu_h^k(b|s) = (\tilde{Z}_h^{k-1})^{-1} \cdot \allowbreak \nu_h^k(b\given s)\exp \{-\gamma   d^{k-1}_h(s)  \langle \tilde{r}_h^{k-1}(s, \cdot , b),   \mu^{k-1}_h( \cdot \given s)  \rangle_\cA \}$, with the denominator $\tilde{Z}_h^{k-1}$ being a normalization term.

\begin{algorithm}[t]\caption{Optimistic Policy Optimization for Player 2 with Single-Controller Transition} 
    \setstretch{1}
	\begin{algorithmic}[1]
		\State {\bfseries Initialize:} $\nu_h^0(\cdot|s) = \boldsymbol{1}/ |\cB|$ for all $s\in \cS$ and $h\in [H]$. $\hat{\cP}_h^0(\cdot|s,a) = \boldsymbol{1}/ |\cS|$ for all $(s, a)\in \cS \times \cA$ and $h\in [H]$. $\hat{r}_h^0(\cdot,\cdot,\cdot) = \beta_h^{r, 0}(\cdot,\cdot,\cdot) = \boldsymbol{0}$ for all $h\in [H]$. 
		        		
			\For{episode $k=1,\ldots,K$}   
				\State Observe Player 1's policy $ \{\mu_h^{k-1}\}_{h=1}^H $.
 	
			   	\State Start from the initial state $s_1^k= s_1$. 
		        \For{step $h=1, 2,\ldots, H$} \Comment{{\color{blue}Policy Evaluation}}
					\State Estimate the transition and reward function by $\hat{\cP}_h^{k-1}$ and $\hat{r}_h^{k-1}$ as \eqref{eq:estimate}. 		
			
					\State \label{line:def_til_r}Update $\tilde{r}_h^{k-1}$, $\forall (s, a, b) \in \cS \times \cA \times \cB$: 	
						\begin{align*}
\tilde{r}_h^{k-1}(s, a, b) = \max \big\{ \hat{r}^{k-1}_h(s,a,b)  - \beta_h^{r, k-1}(s,a,b), 0 \big\}. 						
						\end{align*}	
					\State 
Estimate the state reaching probability by $d_h^{\mu^{k-1}, \hat{\cP}^{k-1}}(s)$, $\forall s \in \cS, h\in [H]$.
	            \EndFor


					\State Update policy $\nu_h^k(b|s)$ by solving \eqref{eq:descent}, $\forall (s, b, h)$. \Comment{{\color{blue}Policy Improvement}}
	            \State Take actions following $b_h^k \sim \nu_h^k(\cdot | s_h^k),  \forall h \in [H]$. 
	           \State Observe the trajectory $\{(s^k_h, a^k_h, b^k_h, s_{h+1}^k)\}_{h=1}^H$, and  rewards $\{r^k_h(s^k_h, a^k_h, b_h^k) \}_{h=1}^H$.   			
		\EndFor               
	\end{algorithmic}\label{alg:po2}
\end{algorithm}

\subsection{Theoretical Results}
Next, we present the main results of the regrets for the single-controller transition model.

\begin{theorem} \label{thm:main_1} By setting $\eta = \sqrt{  \log |\cA|/(KH^2)}$,  with probability at least $1-3\delta$, Algorithm \ref{alg:po1} ensures the following regret bound for Player 1 $\Regret_1(K) \leq  \tilde{\cO} \big( C\sqrt{ T } \big)$, where $T = KH$ denotes the total steps, and the constant $C = \sqrt{  |\cS|^2 |\cA| H^3} + \sqrt{  |\cS| |\cA| |\cB| H }$.
\end{theorem}

Theorem \ref{thm:main_1} shows that $\Regret_1(K)$ is in the level of $\tilde{\cO}(\sqrt{T})$, for arbitrary policies of Player 2. Similar to the discussion after Theorem \ref{thm:main_FIT2}, from the perspective of Player 1, the game can also be viewed as a special case of an MDP with adversarial bandit feedback. Under the FP framework, by utilizing the past actions and policies of Player 2, Algorithm \ref{alg:po1} can obtain an $\tilde{\cO}(\sqrt{T})$ regret, comparing to $\tilde{\cO}(T^{2/3})$ regret by the PO method \citep{efroni2020optimistic}  and $\tilde{\cO}(\sqrt{T})$ regret by a computationally demanding non-PO method \citep{jin2019learning} for MDPs with adversarial rewards.

\begin{theorem} \label{thm:main_2} By setting $\gamma = \sqrt{  |\cS|\log |\cB|/K}$, with probability at least $1-2\delta$, Algorithm \ref{alg:po2} ensures the sublinear regret bound for Player 2, i.e., $\Regret_2(K) \leq  \tilde{\cO} ( C \sqrt{ T } )$, where $T = KH$ is the total number of steps, and the constant factor $C = \sqrt{  |\cS|^2 |\cA| H^3} + \sqrt{  |\cS| |\cA| |\cB| H }$.
\end{theorem}

Interestingly, Theorem \ref{thm:main_2} also shows that $\Regret_2(K)$ has the same bound (including the constant factor $C$) as $\Regret_1(K)$ given any opponent's policy, though the transition model bonus is not involved in Algorithm \ref{alg:po2} and the learning process for two players are essentially different. 
In fact, although the bonus term for the transition is not involved in this algorithm, approximating the state reaching probability of Player 1 implicitly reflects the gap between the empirical transition $\hat{\PP}^k$ and the true transition $\PP$, which can explain the same upper bounds in Theorems \ref{thm:main_1} and \ref{thm:main_2}.

Moreover, if Player 1 runs Algorithm \ref{alg:po1_it} and Player 2 runs Algorithm \ref{alg:po2_it} \emph{simultaneously}, we have the following corollary.

\begin{corollary} \label{coro:main} By setting $\eta$ and $\gamma$ as in Theorem \ref{thm:main_1} and Theorem \ref{thm:main_2}, letting $T = KH$, with probability at least $1-5\delta$,  Algorithm \ref{alg:po1} and Algorithm \ref{alg:po2} ensure the optimality gap $\Gap(K) \leq  \tilde{\cO} (\sqrt{T} )$. 
\end{corollary}

We further provide a simulation experiment to verify the theoretical results for the proposed Algorithms \ref{alg:po1} and \ref{alg:po2} in Appendix \ref{sec:simul}.

\section{Theoretical Analysis} \label{sec:theory}


\vspace{-0.3cm}
 
\subsection{Proofs of Theorems \ref{thm:main_FIT1} and \ref{thm:main_FIT2}}

\vspace{-0.25cm}

\begin{proof}  To bound $\Regret_1(K)$ , we need to analyze the value function difference for the instantaneous regret at the $k$-th episode, i.e., $V_1^{\pi^*, \nu^k}(s_1) - V_1^{\pi^k, \nu^k}(s_1)$. By Lemma \ref{lem:V_diff_it1}, we decompose the difference between $V_1^{\pi^*, \nu^k}(s_1)$ and $V_1^{\pi^k, \nu^k}(s_1)$ into four terms
\setlength{\belowdisplayskip}{2pt}
\begin{align*}
&V_1^{\pi^*, \nu^k}(s_1) - V_1^{\pi^k, \nu^k}(s_1)  \\
&\leq \underbrace{\overline{V}_1^k(s_1) -  V_1^{\pi^k, \nu^k}(s_1)}_{\err_k(\text{I.1})} + \underbrace{\sum_{h=1}^H \EE_{\pi^*, \PP, \nu^k} \{  [\pi^*_h(\cdot| s_h)]^\top \overline{\iota}_h^k(s_h, \cdot, \cdot)  \nu_h^k(\cdot| s_h)  \given s_1 \}}_{\err_k(\text{I.2})} \\[-0.5\baselineskip]
&\quad  + \underbrace{\sum_{h=1}^H \EE_{\pi^*, \PP^1} \{ \langle \pi_h^*(\cdot | s^1_h)-\pi_h^k(\cdot | s^1_h),  M_h^k(s^1_h, \cdot) \rangle_{\cA}  \given  s_1\}}_{\err_k(\text{I.3})}  + \underbrace{2 H \sum_{h=1}^H  \sum_{s^2_h \in \cS_2} |q_h^{\nu^k, \PP^2}(s_h^2)  - d_h^{\nu^k,\hat{\PP}^{2,k}}(s_h^2) |}_{\err_k(\text{I.4})},
\end{align*}
where $M_h^k(s^1_h, \cdot) := \sum_{s^2_h \in \cS_2} F_h^{1,k}(s^1_h, s^2_h, \cdot) d_h^{\nu^k,\hat{\PP}^{2,k}}(s_h^2)$. Here we define the model prediction error of $Q$-function as $\overline{\iota}_h^k(s, a, b) = r_h(s,a,b) +  \PP_h\overline{V}_{h+1}^k(s, a, b) - \overline{Q}_h^k(s,a,b)$. Let $s^1_h, s^2_h, a_h, b_h$ be random variables for states and actions.

Specifically, $\err_k(\text{I.1})$ is the difference between the estimated value function and the true value function, $\err_k(\text{I.2})$ is associated with the model prediction error $\overline{\iota}_h^k(s, a, b)$ of Q-function, $\err_k(\text{I.3})$ is the error from the policy mirror ascent step, and $\err_k(\text{I.4})$ is the error related to the reaching probability estimation. According to Lemmas  \ref{lem:mirror_it1}, \ref{lem:value_diff_it1},  \ref{lem:stationary_dist_err_P2}, we have that $\sum_{k=1}^K\err_k(\text{I.1}) \leq  \tilde{\cO} ( \sqrt{  |\cS_1|^2 |\cA| H^4 K } + \sqrt{  |\cS_2|^2 |\cB| H^4 K } + \sqrt{  |\cS_1||\cS_2| |\cA| |\cB| H^2 K } )$, the third error term is bounded as $\sum_{k=1}^K\err_k(\text{I.3}) \leq   \cO( \sqrt{H^4 K \log |\cA| })$, and the last error term is bounded as $\sum_{k=1}^K\err_k(\text{I.4}) \leq   \tilde{\cO} ( H^2 |\cS_2| \sqrt{ |\cB| K} )$. Moreover, as shown in Lemma \ref{lem:pred_err_it1}, since the estimated Q-function is a UCB estimate, then we have that the model prediction error $\overline{\iota}_h^k(s,a,b)\leq 0$ with high probability, which leads to $\sum_{k=1}^K\err_k(\text{I.2}) \leq 0$. This shows the significance of the principle of ``optimism in the face of uncertainty''. By the union bound, all the above inequalities hold with probability at least $1-4\delta$ . 
Therefore, letting $T = KH$, by the relation that $\Regret_1(K) = \sum_{k=1}^K [V_1^{\pi^*, \nu^k}(s_1) - V_1^{\pi^k, \nu^k}(s_1)] \leq \sum_{k=1}^K [\err_k(\text{I.1}) + \err_k(\text{I.2}) + \err_k(\text{I.3})+ \err_k(\text{I.4})]$, we can obtain the result in Theorem \ref{thm:main_FIT1}.

Due to the symmetry of Algorithm \ref{alg:po1_it} and Algorithm \ref{alg:po2_it} as we discussed in Section \ref{sec:FIT}, the proof for Theorem \ref{thm:main_FIT2} exactly follows the proof of Theorem \ref{thm:main_FIT2}. This completes the proof.
\end{proof}

\vspace{-0.65cm}
\subsection{Proofs of Theorems \ref{thm:main_1} and \ref{thm:main_2}}
\vspace{-0.23cm}
\begin{proof} We first show the proof of Theorem \ref{thm:main_1}. By lemma \ref{lem:V_diff_1}, we have 
\begin{align*}
V_1^{\pi^*, \nu^k}(s_1) - V_1^{\pi^k, \nu^k}(s_1) &\leq  \underbrace{\overline{V}_1^k(s_1) -  V_1^{\pi^k, \nu^k}(s_1)}_{\err_k(\text{II.1})} + \underbrace{ \sum_{h=1}^H \EE_{\pi^*, \PP, \nu^k} \big[ \overline{\varsigma}_h^k(s_h, a_h, b_h) \biggiven s_1 \big]}_{\err_k(\text{II.2})} \\[-0.5\baselineskip]
&\quad +\underbrace{\sum_{h=1}^H \EE_{\pi^*, \PP} [ \langle \pi_h^*(\cdot | s_h)-\pi_h^k(\cdot | s_h), U_h^k(s_h,\cdot) \rangle_{\cA}  \given s_1 ]}_{\err_k(\text{II.3})},
\end{align*}
where $s_h, a_h, b_h$ are random variables for states and actions, $U_h^k(s,a) := \langle \overline{Q}_h^k(s, a , \cdot), \nu_h^k(\cdot\given s) \rangle_\cB$, and we define the model prediction error of $Q$-function as $\overline{\varsigma}_h^k(s, a, b) = r_h(s,a,b) +  \PP_h\overline{V}_{h+1}^k(s, a) - \overline{Q}_h^k(s,a,b)$. 

Particularly, $\err_k(\text{II.1})$ is the difference between the estimated value function and the true value function, $\err_k(\text{II.2})$ is associated with the model prediction error $\overline{\varsigma}_h^k(s, a, b)$ for Q-function, and $\err_k(\text{II.3})$ characterizes the error from the policy mirror ascent step. As shown in Lemma \ref{lem:bias_1}, $\sum_{k=1}^K\err_k(\text{II.1}) \leq  \tilde{\cO} ( \sqrt{  |\cS|^2 |\cA| H^4 K } + \sqrt{  |\cS| |\cA| |\cB| H^2 K } )$ with probability at least $1-\delta$. In addition, we have $\sum_{k=1}^K\err_k(\text{II.2}) \leq 0$ with probability at least $1-2\delta$ as shown in Lemma \ref{lem:pred_err_1}, which is due to the optimistic estimation of the Q-function. Furthermore, Lemma \ref{lem:mirror_1} shows the cumulative error for the mirror ascent step is $\sum_{k=1}^K\err_k(\text{II.3}) \leq  \cO( \sqrt{H^4 K \log |\cA| })$ with setting $\eta = \sqrt{  \log |\cA|/(KH^2)}$. Therefore, letting $T = KH$, further by the relation that $\Regret_1(K) \leq \sum_{k=1}^K [\err_k(\text{II.1}) + \err_k(\text{II.2}) + \err_k(\text{II.3})]$, we can obtain the result in Theorem \ref{thm:main_1} with probability at least $1-3\delta$ by the union bound.

Next, we show the proof of Theorem \ref{thm:main_2}. By Lemma \ref{lem:V_diff_2}, we can decompose the difference between $V_1^{\pi^k, \nu^k}(s_1)$ and  $V_1^{\pi^k, \nu^*}(s_1)$ into four terms
\begin{align*}
&V_1^{\pi^k, \nu^k}(s_1) - V_1^{\pi^k, \nu^*}(s_1) \\[-0.3\baselineskip]
&\qquad \leq \underbrace{2\sum_{h=1}^H  \EE_{\pi^k, \PP, \nu^k} [\beta_h^{r,k}(s_h,a_h,b_h) \given s_1]}_{\err_k(\text{III.1})} + \underbrace{ \sum_{h=1}^H \sum_{s\in \cS} d_h^{\pi^k, \hat{\PP}^k}(s) \big[\pi_h^k(\cdot| s)\big]^\top  \underline{\varsigma}_h^k(s, \cdot, \cdot)\nu_h^*(\cdot | s)}_{\err_k(\text{III.2})} \\[-0.3\baselineskip]
&\qquad  \quad +  \underbrace{\sum_{h=1}^H \sum_{s\in \cS} d_h^{\pi^k, \hat{\PP}^k}(s) \big\langle W_h^k(s, \cdot),  \nu_h^k(\cdot | s) - \nu_h^*(\cdot | s) \big\rangle_\cB}_{\err_k(\text{III.3})} + \underbrace{2\sum_{h=1}^H \sum_{s\in \cS}  |q_h^{\pi^k, \PP}(s)- d_h^{\pi^k, \hat{\PP}^k}(s)|}_{\err_k(\text{III.4})},
\end{align*}
with $W_h^k(s, b) =  \langle \tilde{r}_h^k(s, \cdot , b),   \pi^k_h( \cdot \given s)  \rangle_\cA $ and $\underline{\iota}_h^k(s, a, b)= \tilde{r}_h^k (s,a,b)-r_h (s,a,b)$. The above inequality holds for all $k\in [K]$ with probability at least $1-\delta$. Due to the single-controller structure, distinct from the value function decomposition above for Theorem \ref{thm:main_1}, here we have that $\err_k(\text{III.1})$ is the expectation of reward bonus term, $\err_k(\text{III.2})$ is associated with the reward prediction error $\underline{\varsigma}_h^k$, $\err_k(\text{III.3})$ is the error from the policy mirror descent step, and $\err_k(\text{III.4})$ is the difference between the true state reaching probability and the empirical one. Technically, in the proof of this decomposition, we can show $V_1^{\pi^k, \nu^k}(s_1) - V_1^{\pi^k, \nu^*}(s_1)   =  \sum_{h=1}^H \sum_{s\in \cS} q_h^{\pi^k, \PP}(s) [\pi_h^k(\cdot| s)]^\top   r_h (s, \cdot, \cdot) (\nu_h^k-\nu_h^*)(\cdot | s)$, where the value function difference is only related to the reward function $ r_h (s, \cdot, \cdot)$ instead of the Q-function. This is the reason why only the reward bonus and reward-based mirror descent appear in Algorithm \ref{alg:po2}.

As shown in Lemmas \ref{lem:mirror_2}, \ref{lem:stationary_dist_err}, and \ref{lem:reward_err}, we can obtain upper bounds that $\sum_{k=1}^K\err_k(\text{III.1}) \leq  \tilde{\cO} (\sqrt{|\cS| |\cA| |\cB| H^2 K } )$,  $\sum_{k=1}^K\err_k(\text{III.3}) \leq  \cO ( \sqrt{H^2 |\cS| K \log |\cB|} )$, $\sum_{k=1}^K\err_k(\text{III.4}) \leq  \tilde{\cO}(H^2|\cS|\sqrt{|\cA| K} )$ by taking summation from $k=1$ to $K$ for the three error terms $\err_k(\text{III.1})$, $\err_k(\text{III.3})$, and  $\err_k(\text{III.4})$. For $\err_k(\text{III.2})$, by Lemma \ref{lem:pred_err2}, with probability at least $1-\delta$, we have that $\sum_{k=1}^K\err_k(\text{III.2})\allowbreak \leq 0$, which is due to the optimistic estimation of the reward function, i.e., $\tilde{r}$. The above inequalities hold with probability at least $1-2\delta$ by the union bound. Therefore, letting $T = KH$, further by $\Regret_1(K) \leq \sum_{k=1}^K [\err_k(\text{III.1}) + \err_k(\text{III.2}) + \err_k(\text{III.3}) + \err_k(\text{III.4})]$, we can obtain the result in Theorem \ref{thm:main_2}. This completes the proof.
\end{proof}

\section{Conclusion and Discussion}
In this paper, we propose and analyze new fictitious play policy optimization algorithms for two-player zero-sum Markov games with structured but unknown transitions. We consider two classes of transition structures: factored independent transition and single-controller transition. For both scenarios, we prove $\widetilde{\mathcal{O}}(\sqrt{T})$ regret bounds for each player after $T$ steps in a two-agent competitive game scenario. When both players adopt the proposed algorithms, their overall optimality gap is $\widetilde{\mathcal{O}}(\sqrt{T})$.

Our proposed algorithms and the associated analysis can be potentially extended to different game settings, e.g., the extensions to the multi-player or general-sum game with the factored independent transition, and the extensions from the two-player single controller game to the multi-player game with a single controller. We leave them as our future work.

\bibliography{bibliography}
\bibliographystyle{ims}

\newpage
\begin{appendices}
\onecolumn
\vspace{1em}
\renewcommand{\thesection}{\Alph{section}}

\section{Proofs for Section \ref{sec:FIT}}
\begin{lemma} \label{lem:V_diff_it1} At the $k$-th episode of Algorithm \ref{alg:po1_it}, the difference between value functions $V_1^{\mu^*, \nu^k}(s_1)$ and $V_1^{\mu^k, \nu^k}(s_1)$ is bounded as
\begin{align*}
&V_1^{\mu^*, \nu^k}(s_1) - V_1^{\mu^k, \nu^k}(s_1) \\
&\qquad  =  \overline{V}_1^k(s_1) -  V_1^{\mu^k, \nu^k}(s_1) + \sum_{h=1}^H \EE_{\mu^*, \cP, \nu^k} \big\{  [\mu^*_h(\cdot| s_h)]^\top \overline{\iota}_h^k(s_h, \cdot, \cdot)  \nu_h^k(\cdot| s_h)  \biggiven s_1 \big\} \\
&\qquad \quad + \sum_{h=1}^H \EE_{\mu^*, \cP^1} \Big\{ \Big\langle \mu_h^*(\cdot | s^1_h)-\mu_h^k(\cdot | s^1_h), \sum_{s^2_h \in \cS_2} F_h^{1,k}(s^1_h, s^2_h, \cdot) d_h^{\nu^k,\hat{\cP}^{2,k}}(s_h^2) \Big\rangle_{\cA}  \Biggiven  s^1_1, s^2_1 \Big\}  \\
&  \qquad  \quad +2 H \sum_{h=1}^H  \sum_{s^2_h \in \cS_2} \left|q_h^{\nu^k, \cP^2}(s_h^2)  - d_h^{\nu^k,\hat{\cP}^{2,k}}(s_h^2)\right |,
\end{align*}
where $s_h, a_h, b_h$ are random variables for state and actions, $F_h^{1,k}(s^1, s^2, a) := \langle \overline{Q}_h^k(s^1, s^2, a, \cdot),  \nu_h^k(\cdot | s^2)\rangle_{\cB}$, and we define the model prediction error of $Q$-function as
\begin{align}
\begin{aligned} \label{eq:pred_err_up_it1} 
&\overline{\iota}_h^k(s, a, b) = r_h(s,a,b) +  \cP_h\overline{V}_{h+1}^k(s, a, b) - \overline{Q}_h^k(s,a,b).
\end{aligned}
\end{align} 
\end{lemma}

\begin{proof} The proof of this lemma starts with decomposing the value function difference as 
\begin{align}
V_1^{\mu^*, \nu^k}(s_1) - V_1^{\mu^k, \nu^k}(s_1) = V_1^{\mu^*, \nu^k}(s_1) - \overline{V}_1^k(s_1) + \overline{V}_1^k(s_1) -  V_1^{\mu^k, \nu^k}(s_1).  \label{eq:V_diff_1_decomp_it}
\end{align}
Here the term $\overline{V}_1^k(s_1) -  V_1^{\mu^k, \nu^k}(s_1)$ is the bias between the estimated value function $\overline{V}_1^k(s_1)$ generated by Algorithm \ref{alg:po1_it} and the value function $V_1^{\mu^k, \nu^k}(s_1)$ under the true transition model $\cP$ at the $k$-th episode. We first analyze the term $V_1^{\mu^*, \nu^k}(s_1) - \overline{V}_1^k(s_1)$. For any $h$ and $s$, we consider to decompose the term $V_h^{\mu^*, \nu^k}(s) - \overline{V}_h^k(s)$, which gives
\begin{align}
\begin{aligned} \label{eq:V_diff_1_init_it}
&V_h^{\mu^*, \nu^k}(s) - \overline{V}_h^k(s) \\
&\qquad=  [\mu^*_h(\cdot | s)]^\top Q_h^{\mu^*, \nu^k}(s, \cdot , \cdot) \nu^k_h(\cdot | s) -  \big[\mu_h^k(\cdot | s)\big]^\top \overline{Q}_h^ k(s, \cdot , \cdot) \nu_h^k(\cdot | s)\\
&\qquad=  [\mu^*_h(\cdot | s)]^\top Q_h^{\mu^*, \nu^k}(s, \cdot , \cdot) \nu^k_h(\cdot | s) -  [\mu_h^*(\cdot | s)]^\top \overline{Q}_h^k(s, \cdot , \cdot) \nu_h^k(\cdot | s) \\
&\qquad \quad +   [\mu_h^*(\cdot | s)]^\top \overline{Q}_h^k(s, \cdot , \cdot) \nu_h^k(\cdot | s) -  \big[\mu_h^k(\cdot | s)\big]^\top \overline{Q}_h^k(s, \cdot , \cdot) \nu_h^k(\cdot | s) \\
&\qquad =  [\mu^*_h(\cdot | s)]^\top \big[Q_h^{\mu^*, \nu^k}(s, \cdot , \cdot)  -   \overline{Q}_h^k(s, \cdot , \cdot) \big] \nu_h^k(\cdot | s) \\
&\qquad \quad +   \big[\mu_h^*(\cdot | s)-\mu_h^k(\cdot | s)\big]^\top \overline{Q}_h^k(s, \cdot , \cdot) \nu_h^k(\cdot | s),
\end{aligned}
\end{align}
where the first inequality is by the definition of $V_h^{\mu^*, \nu^k}$ in \eqref{eq:bellman_V} and the definition of $\overline{V}_h^k$ in Line \ref{line:V_up_it1} of Algorithm \ref{alg:po1_it}. In addition, by the definition of $Q_h^{\mu^*, \nu^k}(s, \cdot , \cdot)$ in \eqref{eq:bellman_Q} and the definition of the model prediction error $\overline{\iota}_h^k$ for Player 1 in \eqref{eq:pred_err_up_it1}, 
we have
\begin{align*}
&[\mu^*_h(\cdot | s)]^\top \big[Q_h^{\mu^*, \nu^k}(s, \cdot , \cdot)  -   \overline{Q}_h^k(s, \cdot , \cdot) \big] \nu_h^k(\cdot | s) \\
& \qquad =  \sum_{a \in \cA} \sum_{b\in \cB} \mu^*_h(a | s) \bigg[ \sum_{s'\in \cS}  \cP_h(s'|s, a, b) \big[V_{h+1}^{\mu^*, \nu^k}(s') -  \overline{V}_{h+1}^k(s')\big] + \overline{\iota}_h^k(s,a, b) \bigg] \nu_h^k(b | s)\\
& \qquad =  \sum_{a \in \cA} \sum_{b\in \cB} \mu^*_h(a | s) \bigg[ \sum_{s'\in \cS}  \cP_h(s'|s, a, b) \big[V_{h+1}^{\mu^*, \nu^k}(s') -  \overline{V}_{h+1}^k(s')\big] \bigg] \nu_h^k(b | s) \\
&\qquad \quad +  \sum_{a \in \cA} \sum_{b\in \cB} \mu^*_h(a | s)  \overline{\iota}_h^k(s,a, b)  \nu_h^k(b | s).
\end{align*}
Combining this equality with \eqref{eq:V_diff_1_init_it} gives
\begin{align}
\begin{aligned} \label{eq:V_diff_1_rec_it}
V_h^{\mu^*, \nu^k}(s) - \overline{V}_h^k(s)  & =  \sum_{a \in \cA} \sum_{b\in \cB} \mu^*_h(a | s) \bigg[ \sum_{s'\in \cS}  \cP_h(s'|s, a, b) \big[V_{h+1}^{\mu^*, \nu^k}(s') -  \overline{V}_{h+1}^k(s')\big] \bigg] \nu_h^k(b | s)   \\
&\quad +  \sum_{a \in \cA} \sum_{b\in \cB}  \mu^*_h(a | s) \overline{\iota}_h^k(s,a, b)  \nu_h^k(b | s) \\
&\quad + \sum_{a \in \cA} \sum_{b\in \cB}  \big[\mu_h^*(a | s)-\mu_h^k(a | s)\big]  \overline{Q}_h^k(s, a, b) \nu_h^k(b | s).
\end{aligned}
\end{align}
The inequality \eqref{eq:V_diff_1_rec_it} indicates a recursion of the value function difference $V_h^{\mu^*, \nu^k}(s) - \overline{V}_h^k(s)$. As we have defined $V_{H+1}^{\mu^*, \nu^k}(s) = 0$ and $\overline{V}_{H+1}^k(s) = 0$, by recursively applying \eqref{eq:V_diff_1_rec_it} from $h = 1$ to $H$, we obtain
\begin{align}  
\begin{aligned}\label{eq:V_diff_1_al_it}
V_1^{\mu^*, \nu^k}(s_1) - \overline{V}_1^k(s_1)  &  =  \sum_{h=1}^H \EE_{\mu^*, \cP, \nu^k} \big\{  [\mu^*_h(\cdot| s_h)]^\top \overline{\iota}_h^k(s_h, \cdot, \cdot)  \nu_h^k(\cdot| s_h)  \biggiven s_1 \big\} \\
& \quad  + \underbrace{\sum_{h=1}^H \EE_{\mu^*, \cP, \nu^k} \big\{ \big[\mu_h^*(\cdot | s_h)-\mu_h^k(\cdot | s_h)\big]^\top  \overline{Q}_h^k(s_h, \cdot, \cdot) \nu_h^k(\cdot | s_h)  \biggiven  s_1 \big\}}_{\text{Term(I)}},
\end{aligned}
\end{align}
where $s_h$ are a random variables denoting the state at the $h$-th step following a distribution determined jointly by $\mu^*, \cP, \nu^k$. 
Note that we have the factored independent transition model structure $\cP_h(s'|s,a,b) = \cP^1_h(s^1{}'|s^1, a)\cP^2_h(s^2{}'|s^2, b)$ with $s = (s^1, s^2)$ and $s' = (s^1{}', s^2{}')$, and $\mu_h(a | s) = \mu_h(a | s^1)$ as well as $\nu_h(b | s) = \nu_h(b | s^2)$. Here we also have the state reaching probability $q^{\nu^k, \cP^2}(s^2) = \big\{q^{\nu^k, \cP^2}_h(s^2)\big\}_{h=1}^H$ under $
\nu^k$ and true transition $\cP^2$ for Player 2, and define the empirical reaching probability $d^{\nu^k,\hat{\cP}^{2,k}} (s^2) = \{d^{\nu^k,\hat{\cP}^{2,k}}_h(s^2) \}_{h=1}^H$ under the empirical transition model $\hat{\cP}^{2,k}$ for Player 2, where we let $\hat{\cP}^k_h(s'|s,a,b) = \hat{\cP}^{1, k}_h(s^1{}'|s^1, a)\hat{\cP}^{2, k}_h(s^2{}'|s^2, b)$. Then, for Term(I), we have 
\begin{align}
\text{Term(I)} &= \sum_{h=1}^H \EE_{\mu^*, \cP, \nu^k} \big\{ \big[\mu_h^*(\cdot | s_h)-\mu_h^k(\cdot | s_h)\big]^\top  \overline{Q}_h^k(s_h, \cdot, \cdot) \nu_h^k(\cdot | s_h)  \biggiven  s_1 \big\} \nonumber\\
&  = \sum_{h=1}^H \EE_{\mu^*, \cP^1, \cP^2, \nu^k} \big\{ \big[\mu_h^*(\cdot | s^1_h)-\mu_h^k(\cdot | s^1_h)\big]^\top  \overline{Q}_h^k(s^1_h, s^2_h, \cdot, \cdot) \nu_h^k(\cdot | s^2_h)  \biggiven  s^1_1, s^2_1 \big\} \label{eq:V_diff_1_rec_it_margin}\\
&  = \sum_{h=1}^H \EE_{\mu^*, \cP^1} \big\{ \sum_{s^2_h \in \cS_2} \underbrace{\big[\mu_h^*(\cdot | s^1_h)-\mu_h^k(\cdot | s^1_h)\big]^\top  \overline{Q}_h^k(s^1_h, s^2_h, \cdot, \cdot) \nu_h^k(\cdot | s^2_h)}_{=:\overline{E}^k_h(s_h^1, s_h^2)}q_h^{\nu^k, \cP^2}(s_h^2)  \biggiven  s^1_1, s^2_1 \big\}.\nonumber
\end{align}
The last term of the above inequality \eqref{eq:V_diff_1_rec_it_margin} can be further bounded as
\begin{align*}
&\sum_{h=1}^H \EE_{\mu^*, \cP^1} \big\{ \sum_{s^2_h \in \cS_2} \overline{E}^k_h(s_h^1, s_h^2) q_h^{\nu^k, \cP^2}(s_h^2)  \biggiven  s^1_1, s^2_1 \big\} \\
&\quad\  = \sum_{h=1}^H \EE_{\mu^*, \cP^1} \big\{ \sum_{s^2_h \in \cS_2} \overline{E}^k_h(s_h^1, s_h^2)  [d_h^{\nu^k,\hat{\cP}^{2,k}}(s_h^2) + q_h^{\nu^k, \cP^2}(s_h^2)  - d_h^{\nu^k,\hat{\cP}^{2,k}}(s_h^2) ]\biggiven  s^1_1, s^2_1 \big\}  \\
&\quad\  \leq \sum_{h=1}^H \EE_{\mu^*, \cP^1} \big\{ \sum_{s^2_h \in \cS_2} \overline{E}^k_h(s_h^1, s_h^2) d_h^{\nu^k,\hat{\cP}^{2,k}}(s_h^2)  \biggiven  s^1_1, s^2_1 \big\} +2 H \sum_{h=1}^H  \sum_{s^2_h \in \cS_2} \left|q_h^{\nu^k, \cP^2}(s_h^2)  - d_h^{\nu^k,\hat{\cP}^{2,k}}(s_h^2)\right |,
\end{align*}
where the factor $H$ in the last term is due to $ |\overline{Q}_h^k(s^1_h, s^2_h, \cdot, \cdot)| \leq H$. Combining the above inequality with \eqref{eq:V_diff_1_rec_it_margin}, we have
\begin{align}
\text{Term(I)} &  \leq \sum_{h=1}^H \EE_{\mu^*, \cP^1} \big\{ \big[\mu_h^*(\cdot | s^1_h)-\mu_h^k(\cdot | s^1_h)\big]^\top \sum_{s^2_h \in \cS_2} \overline{Q}_h^k(s^1_h, s^2_h, \cdot, \cdot) \nu_h^k(\cdot | s^2_h)d_h^{\nu^k,\hat{\cP}^{2,k}}(s_h^2)  \biggiven  s^1_1, s^2_1 \big\} \nonumber  \\
&  \quad +2 H \sum_{h=1}^H  \sum_{s^2_h \in \cS_2} \left|q_h^{\nu^k, \cP^2}(s_h^2)  - d_h^{\nu^k,\hat{\cP}^{2,k}}(s_h^2)\right |.\label{eq:V_diff_1_rec_it_margin_2}
\end{align}
Further combining  \eqref{eq:V_diff_1_rec_it_margin_2} with \eqref{eq:V_diff_1_decomp_it}, we eventually have
\begin{align*}
&V_1^{\mu^*, \nu^k}(s_1) - V_1^{\mu^k, \nu^k}(s_1) \\
&\qquad  \leq  \overline{V}_1^k(s_1) -  V_1^{\mu^k, \nu^k}(s_1) + \sum_{h=1}^H \EE_{\mu^*, \cP, \nu^k} \big\{  [\mu^*_h(\cdot| s_h)]^\top \overline{\iota}_h^k(s_h, \cdot, \cdot)  \nu_h^k(\cdot| s_h)  \biggiven s_1 \big\} \\
&\qquad \quad + \sum_{h=1}^H \EE_{\mu^*, \cP^1} \Big\{ \Big\langle \mu_h^*(\cdot | s^1_h)-\mu_h^k(\cdot | s^1_h), \sum_{s^2_h \in \cS_2} F_h^{1,k}(s^1_h, s^2_h, \cdot) d_h^{\nu^k,\hat{\cP}^{2,k}}(s_h^2) \Big\rangle_{\cA}  \Biggiven  s^1_1, s^2_1 \Big\}  \\
&  \qquad  \quad +2 H \sum_{h=1}^H  \sum_{s^2_h \in \cS_2} \left|q_h^{\nu^k, \cP^2}(s_h^2)  - d_h^{\nu^k,\hat{\cP}^{2,k}}(s_h^2)\right |,
\end{align*}
where $F_h^{1,k}(s_h^1, s_h^2, a) := \langle \overline{Q}_h^k(s^1_h, s^2_h, a, \cdot),  \nu_h^k(\cdot | s^2_h)\rangle_{\cB}$ for any $a\in \cA$. This completes our proof.
\end{proof}

\begin{lemma} \label{lem:mirror_it1} With setting $\eta = \sqrt{  \log |\cA|/(KH^2)}$, the mirror ascent steps of Algorithm \ref{alg:po1_it} lead to 
\begin{align*}
&\sum_{k=1}^K \sum_{h=1}^H \EE_{\mu^*, \cP^1} \Big\{ \Big\langle \mu_h^*(\cdot | s^1_h)-\mu_h^k(\cdot | s^1_h), \sum_{s^2_h \in \cS_2} F_h^{1,k}(s^1_h, s^2_h, \cdot) d_h^{\nu^k,\hat{\cP}^{2,k}}(s_h^2) \Big\rangle_{\cA}  \Biggiven  s^1_1, s^2_1 \Big\} \\
&\qquad \leq \cO\left( \sqrt{H^4 K \log |\cA| }\right).
\end{align*}
\end{lemma}
\begin{proof}
As shown in \eqref{eq:ascent_it}, the mirror ascent step at the $k$-th episode is to solve the following maximization problem
\begin{align*}
\maximize_{\mu} &\sum_{h=1}^H  \Big\langle \mu_h(\cdot | s^1)-\mu_h^k(\cdot | s^1), \sum_{s^2 \in \cS_2} F_h^{1,k}(s^1, s^2, \cdot) d_h^{\nu^k,\hat{\cP}^{2,k}}(s^2) \Big\rangle_\cA \\
&- \frac{1}{\eta} \sum_{h=1}^H  D_{\mathrm{KL}}\big( \mu_h(\cdot| s^1), \mu_h^k(\cdot| s^1 ) \big), 
\end{align*}
with $F_h^{1,k}(s^1, s^2, a) := \langle \overline{Q}_h^k(s^1, s^2, a, \cdot),  \nu_h^k(\cdot | s^2)\rangle_{\cB}$. We equivalently rewrite this maximization problem to a minimization problem as
\begin{align*}
\minimize_{\mu} &-\sum_{h=1}^H   \Big\langle \mu_h(\cdot | s^1)-\mu_h^k(\cdot | s^1), \sum_{s^2 \in \cS_2} F_h^{1,k}(s^1, s^2, \cdot) d_h^{\nu^k,\hat{\cP}^{2,k}}(s^2) \Big\rangle_\cA  \\
&+ \frac{1}{\eta} \sum_{h=1}^H  D_{\mathrm{KL}}\big( \mu_h(\cdot| s^1), \mu_h^k(\cdot| s^1 ) \big). 
\end{align*}
Note that the closed-form solution $\mu_h^{k+1}(\cdot|s^1), \forall s^1\in \cS_1$, to this minimization problem is guaranteed to stay in the relative interior of a probability simplex if initializing $\mu^0_h(\cdot | s^1) = \boldsymbol{1} / |\cA|$. Thus, we apply Lemma \ref{lem:pushback} and obtain that for any $\mu = \{\mu_h\}_{h=1}^H$, the following inequality holds
\begin{align*}
&-\eta \Big\langle \mu_h^{k+1}(\cdot|s^1),   \sum_{s^2 \in \cS_2} F_h^{1,k}(s^1, s^2, \cdot) d_h^{\nu^k,\hat{\cP}^{2,k}}(s^2)  \Big\rangle_\cA  + \eta \Big\langle \mu_h(\cdot|s^1), \sum_{s^2 \in \cS_2} F_h^{1,k}(s^1, s^2, \cdot) d_h^{\nu^k,\hat{\cP}^{2,k}}(s^2)  \Big\rangle_\cA \\
&\qquad \leq D_{\mathrm{KL}}\big( \mu_h(\cdot| s^1), \mu_h^k(\cdot| s^1 ) \big) -  D_{\mathrm{KL}}\big( \mu_h(\cdot| s^1), \mu_h^{k+1}(\cdot| s^1) \big) -  D_{\mathrm{KL}}\big( \mu_h^{k+1}(\cdot| s^1), \mu_h^k(\cdot| s^1 ) \big).
\end{align*}
Then, by rearranging the terms and letting $\mu_{h} = \mu^*_{h}$, we have
\begin{align}
\begin{aligned} \label{eq:ascent_al_it}
&\eta \Big\langle  \mu_h^*(\cdot|s^1)- \mu_h^k(\cdot|s^1),  \sum_{s^2 \in \cS_2} F_h^{1,k}(s^1, s^2, \cdot) d_h^{\nu^k,\hat{\cP}^{2,k}}(s_h^2) \Big\rangle_\cA  \\
&\qquad \leq D_{\mathrm{KL}}\big( \mu^*_h(\cdot|s^1), \mu_h^k(\cdot|s) \big) -  D_{\mathrm{KL}}\big( \mu^*_h(\cdot|s), \mu_h^{k+1} (\cdot|s) \big) -  D_{\mathrm{KL}}\big( \mu_h^{k+1}(\cdot|s), \mu_h^k(\cdot|s) \big) \\
&\qquad \quad + \eta \Big\langle  \mu_h^{k+1}(\cdot|s^1) -  \mu_h^k(\cdot|s^1),  \sum_{s^2 \in \cS_2} F_h^{1,k}(s^1, s^2, \cdot)d_h^{\nu^k,\hat{\cP}^{2,k}}(s_h^2) \Big\rangle_\cA.
\end{aligned}
\end{align}
Due to Pinsker's inequality, we have
\begin{align*}
&-D_{\mathrm{KL}}\big( \mu_h^{k+1}(\cdot| s^1), \mu_h^k(\cdot| s^1 ) \big) \leq -\frac{1}{2} \big\|\mu_h^{k+1}(\cdot| s^1) - \mu_h^k(\cdot| s^1 )\big\|^2_1.
\end{align*}
Further by Cauchy-Schwarz inequality, we have
\begin{align*}
 \eta \Big\langle  \mu_h^{k+1}(\cdot|s^1) -  \mu_h^k(\cdot|s^1),  \sum_{s^2 \in \cS_2} F_h^{1,k}(s^1, s^2, \cdot) d_h^{\nu^k,\hat{\cP}^{2,k}}(s^2) \Big\rangle_\cA \leq \eta H \big\|\mu_h^{k+1}(\cdot|s^1) - \mu_h^k(\cdot|s^1) \big\|_1.
\end{align*}
since we have
\begin{align*}
\Bigg\|\sum_{s^2 \in \cS_2} F_h^{1,k}(s^1, s^2, \cdot) d_h^{\nu^k,\hat{\cP}^{2,k}}(s^2)\Bigg\|_\infty 
& = \max_{a\in \cA}\sum_{s^2 \in \cS_2} \langle \overline{Q}_h^k(s^1, s^2, a, \cdot),  \nu_h^k(\cdot | s^2)\rangle_{\cB} \cdot d_h^{\nu^k,\hat{\cP}^{2,k}}(s^2)\\
& \leq \sum_{s^2 \in \cS_2} H \cdot d_h^{\nu^k,\hat{\cP}^{2,k}}(s^2) = H.
\end{align*}
Thus, we further obtain
\begin{align}
\begin{aligned} \label{eq:ascent_err_it}
&\hspace{-0.5cm}-D_{\mathrm{KL}}\big( \mu_h^{k+1}(\cdot| s^1), \mu_h^k(\cdot| s^1 ) \big) +  \eta \big\langle  \mu_h^{k+1}(\cdot|s^1) -  \mu_h^k(\cdot|s^1),  \sum_{s^2_h \in \cS_2} F_h^{1,k}(s^1, s^2, \cdot) d_h^{\nu^k,\hat{\cP}^{2,k}}(s^2)\big\rangle_\cA \\
&\hspace{-0.5cm}\qquad \leq -\frac{1}{2} \big\|\mu_h^{k+1}(\cdot| s^1) - \mu_h^k(\cdot| s^1 )\big\|^2_1 + \eta H \big\|\mu_h^{k+1}(\cdot|s^1) - \mu_h^k(\cdot|s^1) \big\|_1  \leq  \frac{1}{2}\eta^2H^2,
\end{aligned}
\end{align}
where the last inequality is by viewing $\big\|\mu_h^{k+1}(\cdot|s^1) - \mu_h^k(\cdot|s^1) \big\|_1$ as a variable $x$ and finding the maximal value of $-1/2\cdot x^2 + \eta H x$ to obtain the upper bound $1/2 \cdot \eta^2H^2$.

Thus, combing \eqref{eq:ascent_err_it} with \eqref{eq:ascent_al_it}, the policy improvement step in Algorithm \ref{alg:po1_it} implies
\begin{align*}
&\eta \Big\langle  \mu_h^*(\cdot|s^1)- \mu_h^k(\cdot|s^1),  \sum_{s^2 \in \cS_2} F_h^{1,k}(s^1, s^2, \cdot) d_h^{\nu^k,\hat{\cP}^{2,k}}(s^2)\Big\rangle_\cA \\
&\qquad  \leq D_{\mathrm{KL}}\big( \mu^*_h(\cdot | s^1), \mu_h^k (\cdot | s^1) \big) -  D_{\mathrm{KL}}\big( \mu^*_h(\cdot | s^1), \mu_h^{k+1}(\cdot | s^1) \big)  + \frac{1}{2}\eta^2H^2,
\end{align*}
which further leads to
\begin{align*}
&\sum_{h=1}^H \EE_{\mu^*, \cP^1} \Big\{ \Big\langle \mu_h^*(\cdot | s^1_h)-\mu_h^k(\cdot | s^1_h), \sum_{s^2_h \in \cS_2} F_h^{1,k}(s^1_h, s^2_h, \cdot) d_h^{\nu^k,\hat{\cP}^{2,k}}(s_h^2) \Big\rangle_{\cA}  \Biggiven  s^1_1, s^2_1 \Big\} \\
&\qquad  \leq \frac{1}{\eta } \sum_{h=1}^H  \EE_{\mu^*, \cP^1}  \big[D_{\mathrm{KL}}\big( \mu^*_h(\cdot | s^1_h), \mu_h^k (\cdot |s^1_h) \big) -  D_{\mathrm{KL}}\big( \mu^*_h(\cdot | s^1_h), \mu_h^{k+1}(\cdot | s^1_h) \big)\big] + \frac{1}{2}\eta H^3.
\end{align*}
Taking summation from $k=1$ to $K$ of both sides, we obtain
\begin{align*}
&\sum_{k=1}^K \sum_{h=1}^H \EE_{\mu^*, \cP^1} \Big\{ \Big\langle \mu_h^*(\cdot | s^1_h)-\mu_h^k(\cdot | s^1_h), \sum_{s^2_h \in \cS_2} F_h^{1,k}(s^1_h, s^2_h, \cdot) d_h^{\nu^k,\hat{\cP}^{2,k}}(s_h^2) \Big\rangle_{\cA}  \Biggiven  s^1_1, s^2_1 \Big\} \\
&\qquad  \leq \frac{1}{\eta } \sum_{h=1}^H  \EE_{\mu^*, \cP^1}  \big[D_{\mathrm{KL}}\big( \mu^*_h(\cdot | s_h^1), \mu_h^1 (\cdot | s_h^1) \big) -  D_{\mathrm{KL}}\big( \mu^*_h(\cdot | s_h^1), \mu_h^{K+1}(\cdot | s_h^1) \big)\big] + \frac{1}{2}\eta K  H^3 \\
&\qquad  \leq \frac{1}{\eta } \sum_{h=1}^H  \EE_{\mu^*, \cP^1}  \big[ D_{\mathrm{KL}}\big( \mu^*_h(\cdot | s_h^1), \mu_h^1 (\cdot | s_h^1) \big)\big]  + \frac{1}{2}\eta K  H^3,
\end{align*}
where the last inequality is by non-negativity of KL divergence. With the initialization in Algorithm \ref{alg:po1_it}, it is guaranteed that $\mu_h^1 (\cdot | s^1) = \boldsymbol{1}/|\cA|$, which thus leads to $D_{\mathrm{KL}}\left( \mu^*_h(\cdot | s^1), \mu_h^1 (\cdot | s^1)\right) \leq \log |\cA|$ for any $s^1$. Then, with setting $\eta = \sqrt{  \log |\cA|/(KH^2)}$, we bound the last term as
\begin{align*}
\frac{1}{\eta } \sum_{h=1}^H  \EE_{\mu^*, \cP^1}  \big[D_{\mathrm{KL}}\big( \mu^*_h(\cdot | s^1_h), \mu_h^1 (\cdot | s^1_h) \big)\big]  + \frac{1}{2}\eta K H^3 \leq \cO\left( \sqrt{H^4 K \log |\cA| }\right),
\end{align*}
which gives
\begin{align*}
&\sum_{k=1}^K \sum_{h=1}^H \EE_{\mu^*, \cP^1} \Big\{ \Big\langle \mu_h^*(\cdot | s^1_h)-\mu_h^k(\cdot | s^1_h), \sum_{s^2_h \in \cS_2} F_h^{1,k}(s^1_h, s^2_h, \cdot) d_h^{\nu^k,\hat{\cP}^{2,k}}(s_h^2) \Big\rangle_{\cA}  \Biggiven  s^1_1, s^2_1 \Big\} \\
&\qquad \leq \cO\left( \sqrt{H^4 K \log |\cA| }\right).
\end{align*}
This completes the proof.
\end{proof}

\begin{lemma} \label{lem:r_bound_it} For any $k \in [K]$, $h \in [H]$ and all $(s,a, b) \in \cS \times \cA \times \cB$, with probability at least $1-\delta$, we have
\begin{align*}
\big|\hat{r}_h^k (s, a, b) - r_h (s, a, b)\big| \leq \sqrt{\frac{4
\log (|\cS||\cA| |\cB| HK/\delta)}{ \max\{N^k_h(s,a,b), 1\}}}.
\end{align*}
\end{lemma}
\begin{proof} The proof for this theorem is a direct application of Hoeffding's inequality. For $k \geq 1$, the definition of $\hat{r}_h^k$ in \eqref{eq:estimate} indicates that $\hat{r}_h^k(s,a,b)$ is the average of $N^k_h(s,a,b)$ samples of the observed rewards at $(s,a,b)$ if $N^k_h(s,a,b) > 0$. Then, for fixed $k\in [K], h\in [H]$ and state-action tuple $(s,a,b)\in \cS\times \cA \times \cB$, when $N^k_h(s,a,b) > 0$, according to Hoeffding's inequality, with probability at least $1-\delta'$ where $\delta'\in (0, 1]$, we have
\begin{align*}
\big|\hat{r}_h^k (s, a, b) - r_h (s, a, b)\big| \leq \sqrt{\frac{
\log (2/\delta')}{ 2N^k_h(s,a,b)}},
\end{align*}
where we also use the facts that the observed rewards $r_h^k\in [0,1]$ for all $k$ and $h$, and $\EE\big[\hat{r}_h^k\big] = r_h$ for all $k$ and $h$.  For the case where $N^k_h(s,a,b) = 0$, by \eqref{eq:estimate}, we know  $\hat{r}_h^k(s,a,b) = 0$ such that $|\hat{r}_h^k (s, a, b) - r_h (s, a, b)| = |r_h (s, a, b)|\leq 1$. On the other hand, we have $ \sqrt{2\log (2/\delta')} \geq 1 > |\hat{r}_h^k (s, a, b) - r_h (s, a, b)|$. Thus, combining the above results, with probability at least $1-\delta'$, for fixed $k\in [K], h\in [H]$ and state-action tuple $(s,a,b)\in \cS\times \cA \times \cB$, we have
\begin{align*}
\big|\hat{r}_h^k (s, a, b) - r_h (s, a, b)\big| \leq \sqrt{\frac{
2\log (2/\delta')}{ \max\{N^k_h(s,a,b), 1\}}}.
\end{align*}
Moreover, by the union bound, letting $\delta = |\cS| |\cA| |\cB| H K \delta'/2$, assuming $K > 1$, with probability at least $1-\delta$, for any $k\in [K], h\in [H]$ and any state-action tuple $(s,a,b)\in \cS\times \cA \times \cB$, we have
\begin{align*}
\big|\hat{r}_h^k (s, a, b) - r_h (s, a, b)\big| \leq \sqrt{\frac{4
\log (|\cS| |\cA| |\cB| H K /\delta)}{ \max \{ N^k_h(s,a,b), 1\} }}.
\end{align*}
This completes the proof.
\end{proof}

In \eqref{eq:bonus_decomp_it}, we factor the state as $s = (s^1, s^2)$ such that we have $|\cS| = |\cS_1||\cS_2|$. Thus, we set $\beta_h^{r, k}(s,a,b) =\sqrt{\frac{4
\log (|\cS| |\cA| |\cB| H K /\delta)}{ \max \{ N^k_h(s,a,b), 1\} }}= \sqrt{\frac{4
\log (|\cS_1| |\cS_2| |\cA| |\cB| H K /\delta)}{ \max \{ N^k_h(s^1, s^2 ,a,b), 1\} }}$, which equals the bound in Lemma \ref{lem:r_bound_it}. The counter $N^k_h(s,a,b)$ is equivalent to $N^k_h(s^1, s^2 ,a,b)$.

\begin{lemma} \label{lem:P_bound_it} For any $k \in [K]$, $h \in [H]$ and all $(s,a) \in \cS \times \cA$, with probability at least $1-\delta$, we have
\begin{align*}
\left\|\hat{\cP}_h^k (\cdot \given s, a, b) - \cP_h (\cdot \given s, a, b)\right\|_1 \leq \sqrt{\frac{2|\cS| \log (|\cS||\cA|HK/\delta)}{ \max\{N_h^k(s,a), 1\}}},
\end{align*}
where we have a factored state space $s=(s^1, s^2)$, $s' = (s^1{}', s^2{}')$, and an independent state transition $\cP_h (s' \given s, a, b) = \cP_h^1 (s^1{}' \given s^1, a) \cP_h^1 (s^2{}' \given s^2, b) $ and $\hat{\cP}_h^k (\cdot \given s, a, b) = \hat{\cP}_h^{1,k} (s^1{}' \given s^1, a) \hat{\cP}_h^{2,k} (s^2{}' \given s^2, b)$.
\end{lemma}
\begin{proof} Since the state space and the transition model are factored, we need to decompose the term as follows
\begin{align*}
&\left\|\hat{\cP}_h^k (\cdot \given s, a, b) - \cP_h (\cdot \given s, a, b)\right\|_1\\
&\qquad = \sum_{s^1{}', s^2{}'} \left|\hat{\cP}_h^{1,k} (s^1{}' \given s^1, a) \hat{\cP}_h^{2,k} (s^2{}' \given s^2, b) - \cP_h^1 (s^1{}' \given s^1, a)\cP_h^2 (s^2{}' \given s^2, b)\right|\\
&\qquad = \sum_{s^1{}', s^2{}'} \Big|\left[\hat{\cP}_h^{1,k} (s^1{}' \given s^1, a)-\cP_h^1 (s^1{}' \given s^1, a) \right] \hat{\cP}_h^{2,k} (s^2{}' \given s^2, b)  \\
&\qquad \quad + \cP_h^1 (s^1{}' \given s^1, a)\left[\hat{\cP}_h^{2,k} (s^2{}' \given s^2, b) -  \cP_h^2 (s^2{}' \given s^2, b)\right] \Big|.
\end{align*}
We can further bound the last term in the above equality as follows
\begin{align*}
&\sum_{s^1{}', s^2{}'} \Big|\left[\hat{\cP}_h^{1,k} (s^1{}' \given s^1, a)-\cP_h^1 (s^1{}' \given s^1, a) \right] \hat{\cP}_h^{2,k} (s^2{}' \given s^2, b)  \\
&\qquad \quad + \cP_h^1 (s^1{}' \given s^1, a)\left[\hat{\cP}_h^{2,k} (s^2{}' \given s^2, b) -  \cP_h^2 (s^2{}' \given s^2, b)\right] \Big|\\
&\qquad \leq  \sum_{s^1{}', s^2{}'} \Big\{\left|\hat{\cP}_h^{1,k} (s^1{}' \given s^1, a)-\cP_h^1 (s^1{}' \given s^1, a) \right| \hat{\cP}_h^{2,k} (s^2{}' \given s^2, b)  \\
&\qquad\quad + \cP_h^1 (s^1{}' \given s^1, a)\left|\hat{\cP}_h^{2,k} (s^2{}' \given s^2, b) -  \cP_h^2 (s^2{}' \given s^2, b) \right|\Big\} \\
&\qquad \leq  \sum_{s^1{}'}  \left|\hat{\cP}_h^{1,k} (s^1{}' \given s^1, a)-\cP_h^1 (s^1{}' \given s^1, a) \right|   +  \sum_{s^2{}'} \left|\hat{\cP}_h^{2,k} (s^2{}' \given s^2, b) -  \cP_h^2 (s^2{}' \given s^2, b) \right| \\
&\qquad =  \left\|\hat{\cP}_h^{1,k} (\cdot \given s^1, a)-\cP_h^1 (\cdot \given s^1, a) \right\|_1   +  \left\|\hat{\cP}_h^{2,k} (\cdot \given s^2, b) -  \cP_h^2 (\cdot \given s^2, b) \right\|_1, 
\end{align*}
where the last inequality is due to $\sum_{s^2{}'} \hat{\cP}_h^{2,k} (s^2{}' \given s^2, b) =1$ and $\sum_{s^1{}'}\cP_h^1 (s^1{}' \given s^1, a)=1$. Thus, we need to bound the two terms $ \|\hat{\cP}_h^{1,k} (\cdot \given s^1, a)-\cP_h^1 (s^1{}' \given s^1, a) \|_1 $ and $\|\hat{\cP}_h^{2,k} (\cdot \given s^2, b) -  \cP_h^2 (\cdot \given s^2, b) \|_1 $ separately.

For $k \geq 1$, we have $\|\hat{\cP}_h^{1,k} (\cdot \given s^1, a)-\cP_h^1 (\cdot \given s^1, a) \|_1 = \max_{\|\zb \|_\infty \leq  1}  ~\langle \hat{\cP}_h^{1,k} (\cdot \given s^1, a)-\cP_h^1 (s^1{}' \given s^1, a) , \zb\rangle_{\cS_1}$ by the duality.  We construct an $\epsilon$-cover for the set $\{\zb\in \RR^{|\cS_1|}: \|\zb\|_\infty \leq 1\}$ with the distance induced by $\|\cdot\|_\infty$, denoted as $\cC_\infty(\epsilon)$, such that for any $\zb \in \RR^{|\cS_1|}$, there always exists $\zb'\in \cC_\infty(\epsilon)$ satisfying $\|\zb-\zb' \|_\infty \leq \epsilon$. The covering number is $\cN_\infty(\epsilon) = |\cC_\infty(\epsilon)|= 1/\epsilon^{|\cS_1|}$. Thus, we know that for any $(s^1, a) \in \cS_1 \times \cA$ and any $\zb$ with $\|\zb\|_\infty \leq 1$, there exists $\zb'\in \cC_\infty(\epsilon)$ such that $\|\zb'-\zb\|_\infty \leq \epsilon$ and 
\begin{align*}
&\big\langle\hat{\cP}_h^{1,k} (\cdot \given s^1, a)-\cP_h^1 (\cdot \given s^1, a), \zb\big\rangle_{\cS_1} \\
 &\qquad = \big\langle \hat{\cP}_h^{1,k} (\cdot \given s^1, a)-\cP_h^1 (\cdot \given s^1, a), \zb' \big\rangle_{\cS_1} + \big\langle \hat{\cP}_h^{1,k} (\cdot \given s^1, a)-\cP_h^1 (\cdot \given s^1, a), \zb- \zb'\big\rangle_{\cS_1} \\
&\qquad \leq  \big\langle \hat{\cP}_h^{1,k} (\cdot \given s^1, a)-\cP_h^1 (\cdot \given s^1, a), \zb' \big\rangle_{\cS_1} +  \epsilon \left\|\hat{\cP}_h^{1,k} (\cdot \given s^1, a)-\cP_h^1 (\cdot \given s^1, a)\right\|_1, 
\end{align*}
such that we further have
\begin{align}
\begin{aligned} \label{eq:net_it}
& \left\|\hat{\cP}_h^{1,k} (\cdot \given s^1, a)-\cP_h^1 (\cdot \given s^1, a)\right\|_1\\
&\qquad = \max_{\|\zb \|_\infty \leq  1}  ~\big\langle \hat{\cP}_h^{1,k} (\cdot \given s^1, a)-\cP_h^1 (\cdot \given s^1, a)), \zb\big\rangle_{\cS_1} \\
 &\qquad \leq  \max_{\zb' \in \cC_\infty(\epsilon)}  ~\big\langle \hat{\cP}_h^{1,k} (\cdot \given s^1, a)-\cP_h^1 (\cdot \given s^1, a), \zb' \big\rangle_{\cS_1} +  \epsilon \left\|\hat{\cP}_h^{1,k} (\cdot \given s^1, a)-\cP_h^1 (\cdot \given s^1, a)\right\|_1.
\end{aligned}
\end{align}
By Hoeffding's inequality and the union bound over all $\zb' \in \cC_\infty(\epsilon)$, when $N_h^k(s^1,a)>0$, with probability at least $1-\delta'$ where $\delta'\in (0, 1]$,
\begin{align} \label{eq:net_bound_it}
\max_{\zb' \in \cC_\infty(\epsilon)}  ~\big\langle \hat{\cP}_h^{1,k} (\cdot \given s^1, a)-\cP_h^1 (\cdot \given s^1, a), \zb' \big\rangle_{\cS_1} \leq  	\sqrt{\frac{|\cS_1| \log (1/\epsilon) + \log (1/\delta')}{2N_h^k(s^1,a)}}.
\end{align}
Letting $\epsilon = 1/2$, by \eqref{eq:net_it} and \eqref{eq:net_bound_it}, with probability at least $1-\delta'$, we have
\begin{align*}
 \left\|\hat{\cP}_h^{1,k} (\cdot \given s^1, a)-\cP_h^1 (\cdot \given s^1, a)\right\|_1 \leq 1 \sqrt{\frac{|\cS| \log 2 + \log (1/\delta')}{2N_h^k(s^1,a)}}.
\end{align*}
When $N_h^k(s^1,a)=0$, we have $\big\|\hat{\cP}_h^{1,k} (\cdot \given s^1, a)-\cP_h^1 (\cdot \given s^1, a)\big\|_1 = \|\cP_h^1 (\cdot \given s^1, a)\|_1 = 1$ such that $2 \sqrt{\frac{|\cS| \log 2 + \log (1/\delta')}{2}} > 1 = \big\|\hat{\cP}_h^{1,k} (\cdot \given s^1, a)-\cP_h^1 (\cdot \given s^1, a)\big\|_1$ always holds. Thus, with probability at least $1-\delta'$, 
\begin{align*}
 \left\|\hat{\cP}_h^{1,k} (\cdot \given s^1, a)-\cP_h^1 (\cdot \given s^1, a)\right\|_1 \leq 2 \sqrt{\frac{|\cS_1| \log 2 + \log (1/\delta')}{2\max \{N_h^k(s^1,a), 1\}}} \leq \sqrt{\frac{2|\cS_1| \log (2/\delta')}{\max \{N_h^k(s^1,a), 1\}}}.
\end{align*}
Then, by the union bound, assuming $K > 1$, letting $\delta'' = |\cS_1| |\cA| H K \delta'/2$, with probability at least $1-\delta''$, for any $(s^1, a)\in \cS_1\times \cA$ and any $h\in [H]$ and $k\in [K]$, we have
\begin{align*}
\left\|\hat{\cP}_h^{1,k} (\cdot \given s^1, a)-\cP_h^1 (\cdot \given s^1, a)\right\|_1 \leq \sqrt{\frac{2|\cS_1| \log (|\cS_1| |\cA| H K/\delta'')}{\max \{N_h^k(s^1,a), 1\}}}.
\end{align*}
Similarly, we can also obtain that with probability at least $1-\delta''$, for any $(s^2, a)\in \cS_2\times \cB$ and any $h\in [H]$ and $k\in [K]$, we have
\begin{align*}
\left\|\hat{\cP}_h^{2,k} (\cdot \given s^2, b)-\cP_h^2 (\cdot \given s^2, b)\right\|_1 \leq \sqrt{\frac{2|\cS_2| \log (|\cS_2| |\cB| H K/\delta'')}{\max \{N_h^k(s^2,b), 1\}}}.
\end{align*}
Further by the union bound, we have with probability at least $1-\delta$ where $\delta= 2\delta''$, 
\begin{align*}
\left\|\hat{\cP}_h^k (\cdot \given s, a, b) - \cP_h (\cdot \given s, a, b)\right\|_1 \leq \sqrt{\frac{2|\cS_1| \log (2|\cS_1| |\cA| H K/\delta)}{\max \{N_h^k(s^1,a), 1\}}} +  \sqrt{\frac{2|\cS_2| \log (2|\cS_2| |\cB| H K/\delta)}{\max \{N_h^k(s^2,b), 1\}}}.
\end{align*}
This completes the proof.
\end{proof}
In \eqref{eq:bonus_decomp_it}, we set $\beta_h^{\cP, k}(s, a,b) =  \sqrt{\frac{2H^2|\cS_1| \log (2|\cS_1| |\cA| H K/\delta)}{\max \{N_h^k(s^1,a), 1\}}} +  \sqrt{\frac{2H^2|\cS_2| \log (2|\cS_2| |\cB| H K/\delta)}{\max \{N_h^k(s^2,b), 1\}}}$, which equals the product of the upper bound in Lemma \ref{lem:P_bound_it} and the factor $H$.

\begin{lemma} \label{lem:pred_err_it1} With probability at least $1-2\delta$, Algorithm \ref{alg:po1_it} ensures that
\begin{align*}
\sum_{k=1}^K \sum_{h=1}^H \EE_{\mu^*, \cP, \nu^k} \big[ \overline{\iota}_h^k(s_h, a_h, b_h)   \biggiven s_1 \big] \leq 0.
\end{align*}
\end{lemma}

\begin{proof} We prove the upper bound of the model prediction error term.  As defined in \eqref{eq:pred_err_up_it1}, we have the instantaneous prediction error at the $h$-step of the $k$-th episode as
\begin{align}
\begin{aligned}\label{eq:pred_err_init1_it}
&  \overline{\iota}_h^k(s, a, b) = r_h(s, a, b) + \big\langle \cP_h(\cdot \given s, a, b), \overline{V}_{h+1}^k(\cdot) \big\rangle_\cS - \overline{Q}_h^k(s, a, b) ,
\end{aligned}
\end{align}
where the equality is by the definition of the prediction error in \eqref{eq:pred_err_up_it1}. By plugging in the definition of $\overline{Q}_h^k$ in Line \ref{line:Q_up_it1} of Algorithm \ref{alg:po1_it}, for any $(s,a,b)$,  we bound the following term as 
\begin{align}
\begin{aligned}\label{eq:pred_err_re1_it}
&\hspace{-0.25cm}r_h(s,a,b) + \big\langle \cP_h(\cdot \given s, a, b), \overline{V}_{h+1}^k(\cdot) \big\rangle_\cS - \overline{Q}_h^k(s,a,b) \\
&\hspace{-0.25cm}~\quad \leq r_h(s,a,b) + \big\langle \cP_h(\cdot \given s, a, b), \overline{V}_{h+1}^k(\cdot) \big\rangle_\cS \\
&\hspace{-0.25cm}~\quad \quad- \min \Big\{ \hat{r}_h^k(s,a,b) + \big\langle \hat{\cP}_h^k(\cdot |s, a, b), \overline{V}_{h+1}^k(\cdot) \big\rangle_\cS  - \beta_h^k, H-h+1 \Big\}  \\
&\hspace{-0.25cm}~\quad \leq  \max \Big\{r_h(s,a,b) - \hat{r}_h^k(s,a,b) + \big\langle \cP_h(\cdot \given s, a, b) - \hat{\cP}_h^k(\cdot |s, a, b), \overline{V}_{h+1}^k(\cdot) \big\rangle_\cS  - \beta_h^k, 0 \Big\}, 
\end{aligned}
\end{align}
where the inequality holds because 
\begin{align*}
&r_h(s,a,b) + \big\langle \cP_h(\cdot \given s, a, b), \overline{V}_{h+1}^k(\cdot) \big\rangle_\cS \\
&\qquad \leq r_h(s,a,b) + \big\|\cP_h(\cdot \given s, a, b)\big\|_1 \|\overline{V}_{h+1}^k(\cdot) \|_\infty  \leq 1 +  \max_{s' \in \cS} \big|\overline{V}_{h+1}^k(s') \big| \leq 1+ H - h, 
\end{align*}
since $\big\|\cP_h(\cdot \given s, a, b)\big\|_1 = 1$ and also the truncation step as shown in Line \ref{line:Q_up_it1} of Algorithm \ref{alg:po1_it} for $\overline{Q}_{h+1}^k$ such that for any $s' \in \cS$
\begin{align}
\begin{aligned}\label{eq:bound_V_up_it}
\big|\overline{V}_{h+1}^k(s') \big| &=  \Big| \big[\mu_{h+1}^k(\cdot | s')\big]^\top \overline{Q}_{h+1}^k(s', \cdot , \cdot) \nu_{h+1}^k(\cdot | s') \Big|\\
&\leq \big\| \mu_{h+1}^k(\cdot | s')\big\|_1 \big\|\overline{Q}_{h+1}^k(s', \cdot , \cdot) \nu_{h+1}^k(\cdot | s') \big\|_\infty \\
&\leq \max_{a,b} \big|\overline{Q}_{h+1}^k(s', a, b)\big|\leq H.
\end{aligned}
\end{align}
Combining \eqref{eq:pred_err_init1_it} and \eqref{eq:pred_err_re1_it} gives
\begin{align}
\begin{aligned}\label{eq:pred_err_al_it}
\overline{\iota}_h^k(s, a, b) &\leq \max \Big\{ r_h(s,a,b) - \hat{r}_h^k(s,a,b) \\
&\quad +   \big\langle \cP_h(\cdot \given s, a, b) - \hat{\cP}_h^k(\cdot |s, a, b), \overline{V}_{h+1}^k(\cdot) \big\rangle_\cS  - \beta_h^k, 0 \Big\} .
\end{aligned}
\end{align}
Note that as shown in \eqref{eq:bonus_decomp_it}, we have
\begin{align*}
\beta_h^k(s,a,b) = \beta_h^{r,k}(s,a,b) + \beta_h^{\cP,k}(s,a, b).
\end{align*}
Then, with probability at least $1-\delta$, we have
\begin{align*}
&r_h(s,a,b) - \hat{r}_h^k(s,a,b) - \beta_h^{r,k}(s,a,b)  \\
&\qquad \leq \big|r_h(s,a,b) - \hat{r}_h^k(s,a,b)\big| - \beta_h^{r,k}(s,a,b)  \\
&\qquad \leq \beta_h^{r,k}(s,a,b)  - \beta_h^{r,k}(s,a,b)  = 0,
\end{align*}
where the last inequality is by Lemma \ref{lem:r_bound_it} and the setting of the bonus for the reward. Moreover, with probability at least $1-\delta$, we have
\begin{align*}
&\big\langle \cP_h(\cdot \given s, a, b) - \hat{\cP}_h^k(\cdot |s, a, b), \overline{V}_{h+1}^k(\cdot) \big\rangle_\cS  - \beta_h^{\cP, k}(s,a, b) \\
&\qquad \leq \big\|\cP_h(\cdot \given s,a, b) - \hat{\cP}_h^k(\cdot |s,a, b)\big\|_1 \big \|\overline{V}_{h+1}^k(\cdot) \big\|_\infty  - \beta_h^{\cP, k}(s,a, b) \\
&\qquad \leq H \big\|\cP_h(\cdot \given s,a, b) - \hat{\cP}_h^k(\cdot |s, a)\big\|_1   - \beta_h^{\cP, k}(s,a, b) \\
&\qquad \leq \beta_h^{\cP,k} (s,a, b)   - \beta_h^{\cP,k}(s,a, b)  = 0,
\end{align*}
where the first inequality is by Cauchy-Schwarz inequality, the second inequality is due to $\max_{s' \in \cS}\big\|\overline{V}_{h+1}^k(s')\big\|_\infty \leq H$ as shown in \eqref{eq:bound_V_up_it}, and the last inequality is by the setting of $\beta_h^{\cP,k}$ in \eqref{eq:bonus_decomp_it} and also Lemma \ref{lem:P_bound_it}. Thus, with probability at least $1-2\delta$, the following inequality holds
\begin{align*}
r_h(s,a,b) - \hat{r}_h^k(s,a,b) + \big\langle \cP_h(\cdot \given s, a, b) - \hat{\cP}_h^k(\cdot |s, a, b), \overline{V}_{h+1}^k(\cdot) \big\rangle_\cS  - \beta_h^k(s, a, b) \leq 0.
\end{align*}
Combining the above inequality with \eqref{eq:pred_err_al_it}, we have that with probability at least $1-2\delta$, for any $h\in [H]$ and $k\in [K]$, the following inequality holds
\begin{align*}
&\overline{\iota}_h^k(s, a, b) \leq 0, ~~\forall (s, a, b) \in \cS \times \cA \times \cB,
\end{align*}
which leads to
\begin{align*}
\sum_{k=1}^K \sum_{h=1}^H \EE_{\mu^*, \cP, \nu^k} \big[ \overline{\iota}_h^k(s_h, a_h, b_h)   \biggiven s_1 \big] \leq 0.
\end{align*}
This completes the proof.
\end{proof}

\begin{lemma}\label{lem:value_diff_it1}
With probability at least $1-\delta$, Algorithm \ref{alg:po1_it} ensures that
\begin{align*}
\sum_{k=1}^K \overline{V}_1^k(s_1)  - \sum_{k=1}^K  V_1^{\mu^k, \nu^k}(s_1) \leq \tilde{\cO} ( \sqrt{  |\cS_1|^2 |\cA| H^4 K } + \sqrt{  |\cS_2|^2 |\cB| H^4 K } + \sqrt{  |\cS_1||\cS_2| |\cA| |\cB| H^2 K } ).
\end{align*}
\end{lemma}

\begin{proof} We assume that a trajectory $\{ (s_h^k, a_h^k, b_h^k, s_{h+1}^k)\}_{h=1}^H$ for all $k\in [K]$ is generated following the policies $\mu^k$, $\nu^k$, and the true transition model $\cP$. Thus, we expand the bias term at the $h$-th step of the $k$-th episode, which is
\begin{align}
\begin{aligned} \label{eq:bia_1_init_it}
&\overline{V}_h^k(s_h^k)  - V_h^{\mu^k, \nu^k}(s_h^k) \\
&\qquad = \big[\mu^k_h(\cdot | s_h^k)\big]^\top \big[\overline{Q}_h^k(s_h^k, \cdot , \cdot)  -  Q_h^{\mu^k, \nu^k}(s_h^k, \cdot , \cdot)\big] \nu_h^k(\cdot | s_h^k)\\
&\qquad= \zeta_h^k + \overline{Q}_h^k(s_h^k, a_h^k, b_h^k)  -  Q_h^{\mu^k, \nu^k}(s_h^k, a_h^k, b_h^k)\\
&\qquad= \zeta_h^k + \big\langle \cP_h(\cdot \given s_h^k, a_h^k, b_h^k), \overline{V}_{h+1}^k(\cdot) - V_{h+1}^{\mu^k, \nu^k} (\cdot) \big\rangle_\cS - \overline{\iota}_h^k(s_h^k, a_h^k, b_h^k) \\
&\qquad = \zeta_h^k + \xi_h^k + \overline{V}_{h+1}^k(s_{h+1}^k) - V_{h+1}^{\mu^k, \nu^k} (s_{h+1}^k) - \overline{\iota}_h^k(s_h^k, a_h^k, b_h^k),
\end{aligned}
\end{align}
where the first equality is by Line \ref{line:V_up} of Algorithm \ref{alg:po1} and \eqref{eq:bellman_V}, the third equality is by plugging in \eqref{eq:bellman_Q} and \eqref{eq:pred_err_up_it1}. Specifically, in the above equality, we introduce two martingale difference sequence, namely, $\{\zeta_h^k\}_{h\geq 0, k\geq 0}$ and $\{\xi_h^k\}_{h\geq 0, k\geq 0}$, which are defined as
\begin{align*}
&\zeta_h^k := \big[\mu^k_h(\cdot | s_h^k)\big]^\top \big[\overline{Q}_h^k(s_h^k, \cdot , \cdot)  -  Q_h^{\mu^k, \nu^k}(s_h^k, \cdot , \cdot)\big] \nu_h^k(\cdot | s_h^k) - \big[ \overline{Q}_h^k(s_h^k, a_h^k, b_h^k)  -  Q_h^{\mu^k, \nu^k}(s_h^k, a_h^k, b_h^k)\big],\\
& \xi_h^k := \big\langle \cP_h(\cdot \given s_h^k, a_h^k, b_h^k), \overline{V}_{h+1}^k(\cdot) - V_{h+1}^{\mu^k, \nu^k} (\cdot) \big\rangle_\cS - \big[ \overline{V}_{h+1}^k(s_{h+1}^k) - V_{h+1}^{\mu^k, \nu^k} (s_{h+1}^k)\big],
\end{align*}
such that 
\begin{align*}
&\EE_{a_h^k \sim \mu^k_h(\cdot | s_h^k), b_h^k \sim \nu^k_h(\cdot | s_h^k)} \big[\zeta_h^k \biggiven \cF_h^k] = 0,\\
& \EE_{s_{h+1}^k \sim \cP_h(\cdot \given s_h^k, a_h^k, b_h^k)} \big[\xi_h^k \biggiven \tilde{\cF}_h^k\big] = 0,
\end{align*}
with $\cF_h^k$ being the filtration of all randomness up to $(h-1)$-th step of the $k$-th episode plus $s_h^k$, and $\tilde{\cF}_h^k$ being the filtration of all randomness up to $(h-1)$-th step of the $k$-th episode plus $s_h^k, a_h^k, b_h^k$.

The equality \eqref{eq:bia_1_init_it} forms a recursion for $\overline{V}_h^k(s_h^k)  - V_h^{\mu^k, \nu^k}(s_h^k)$. We also have $\overline{V}_{H+1}^k(\cdot) = \boldsymbol{0}$ and $V_{H+1}^{\mu^k, \nu^k} (\cdot)= \boldsymbol{0}$. Thus, recursively apply \eqref{eq:bia_1_init_it} from $h=1$ to $H$ leads to the following equality
\begin{align} \label{eq:bias_diff_init_it}
\overline{V}_1^k(s_1)  - V_1^{\mu^k, \nu^k}(s_1) =  \sum_{h=1}^H \zeta_h^k + \sum_{h=1}^H \xi_h^k - \sum_{h=1}^H \overline{\iota}_h^k(s_h^k, a_h^k, b_h^k).
\end{align}
Moreover, by \eqref{eq:pred_err_up_it1} and Line \ref{line:Q_up_it1} of Algorithm \ref{alg:po1_it}, we have
\begin{align*}
-\overline{\iota}_h^k(s_h^k, a_h^k, b_h^k) &= -  r_h(s_h^k, a_h^k, b_h^k) -  \big\langle \cP_h(\cdot \given s_h, a_h, b_h), \overline{V}_{h+1}^k(\cdot) \big\rangle_\cS \\
&\quad + \min \big\{ \hat{r}^k_h(s_h^k, a_h^k, b_h^k)  +  \big\langle \hat{\cP}_h^k(\cdot |s_h, a_h, b_h), \overline{V}_{h+1}^k(\cdot) \big\rangle_\cS  + \beta_h^k(s_h^k, a_h^k, b_h^k), H\big\}.
\end{align*}
Then, we can further bound $-\overline{\iota}_h^k(s_h^k, a_h^k, b_h^k)$ as follows
\begin{align*}
-\overline{\iota}_h^k(s_h^k, a_h^k, b_h^k)&\leq -  r_h(s_h^k, a_h^k, b_h^k) - \big\langle \cP_h(\cdot \given s_h^k, a_h^k, b_h^k), \overline{V}_{h+1}^k(\cdot)\big\rangle_\cS + \hat{r}^k_h(s_h^k, a_h^k, b_h^k) \\
&\quad + \big\langle  \hat{\cP}_h^k(\cdot |s_h^k, a_h^k, b_h^k), \overline{V}_{h+1}^k(\cdot) \big\rangle_\cS  + \beta_h^k(s_h^k, a_h^k, b_h^k) \\
&\leq \big|\hat{r}^k_h(s_h^k, a_h^k, b_h^k) -  r_h(s_h^k, a_h^k, b_h^k)\big| \\
&\quad  + \Big| \big\langle \cP_h(\cdot \given s_h^k, a_h^k, b_h^k) - \hat{\cP}^k_h(\cdot \given s_h^k, a_h^k, b_h^k), \overline{V}_{h+1}^k(\cdot) \big\rangle_\cS \Big| + \beta_h^k(s_h^k, a_h^k, b_h^k),
\end{align*}
where the first inequality is due to $\min \{x,y\} \leq x$. Additionally, we have
\begin{align*}
&\Big| \big\langle \cP_h(\cdot \given s_h^k, a_h^k, b_h^k) - \hat{\cP}^k_h(\cdot \given s_h^k, a_h^k, b_h^k), \overline{V}_{h+1}^k(\cdot) \big\rangle_\cS \Big|  \\
&\qquad \leq  \big\| \overline{V}_{h+1}^k(\cdot)\big\|_\infty \big\| \cP_h(\cdot \given s_h^k, a_h^k, b_h^k) - \hat{\cP}^k_h(\cdot \given s_h^k, a_h^k, b_h^k) \big\|_1 \\
&\qquad \leq H \big\|\cP_h(\cdot \given s_h^k, a_h^k, b_h^k) - \hat{\cP}^k_h(\cdot \given s_h^k, a_h^k, b_h^k)\big\|_1,
\end{align*} 
where the first inequality is by Cauchy-Schwarz inequality and the second inequality is by \eqref{eq:bound_V_up}. Thus, putting the above together, we obtain
\begin{align*}
-\overline{\iota}_h^k(s_h^k, a_h^k, b_h^k)&\leq  \big|\hat{r}^k_h(s_h^k, a_h^k, b_h^k) -  r_h(s_h^k, a_h^k, b_h^k)\big| \\
&\quad + H \big\|\cP_h(\cdot \given s_h^k, a_h^k, b_h^k) - \cP_h(\cdot \given s_h^k, a_h^k, b_h^k) \big\|_1 + \beta_h^k(s_h^k, a_h^k, b_h^k) \\
&\leq  2\beta^{r,k}_h(s_h^k, a_h^k, b_h^k) + 2\beta^{\cP,k}_h(s_h^k, a_h^k, a_h^k),
\end{align*}
where the second inequality is by Lemma \ref{lem:r_bound_it}, Lemma \ref{lem:P_bound_it}, and the decomposition of the bonus term $\beta_h^k$ as \eqref{eq:bonus_decomp_it}.  Due to Lemma \ref{lem:r_bound_it} and Lemma \ref{lem:P_bound_it}, by union bound, for any $h \in [H], k\in [K]$ and $(s_h, a_h, b_h) \in \cS \times \cA \times \cB$, the above inequality holds with probability with probability at least $1-2\delta$. Therefore, by \eqref{eq:bias_diff_init_it}, with probability at least $1-2\delta$, we have
\begin{align}
\begin{aligned}\label{eq:bias_diff_al_it}
&\sum_{k=1}^K \big[\overline{V}_1^k(s_1)  - V_1^{\mu^k, \nu^k}(s_1)\big]\\
&\qquad  \leq \sum_{k=1}^K \sum_{h=1}^H \zeta_h^k + \sum_{k=1}^K\sum_{h=1}^H \xi_h^k + 2 \sum_{k=1}^K\sum_{h=1}^H \beta^{r,k}_h(s_h^k, a_h^k, b_h^k)   + 2\sum_{k=1}^K \sum_{h=1}^H \beta^{\cP,k}_h(s_h^k, a_h^k, b_h^k). 
\end{aligned}
\end{align}
By Azuma-Hoeffding inequality, with probability at least $1-\delta$, the following inequalities hold 
\begin{align*}
&\sum_{k=1}^K \sum_{h=1}^H \zeta_h^k \leq \cO\left(\sqrt{H^3 K \log \frac{1}{\delta}} \right), \quad \sum_{k=1}^K \sum_{h=1}^H \xi_h^k \leq \cO\left(\sqrt{H^3 K \log \frac{1}{\delta}} \right), 
\end{align*} 
where we use the facts that $ | \overline{Q}_h^k(s_h^k, a_h^k, b_h^k)  -  Q_h^{\mu^k, \nu^k}(s_h^k, a_h^k, b_h^k) | \leq 2H$ and $| \overline{V}_{h+1}^k(s_{h+1}^k) - V_{h+1}^{\mu^k, \nu^k} (s_{h+1}^k) |\leq 2H$. Next, we need to bound $\sum_{k=1}^K\sum_{h=1}^H \beta^{r,k}_h(s_h^k, a_h^k, b_h^k)$ and $\sum_{k=1}^K \sum_{h=1}^H \beta^{\cP,k}_h(s_h^k, a_h^k, b_h^k)$ in \eqref{eq:bias_diff_al_it}. We show that 
\begin{align*}
\sum_{k=1}^K\sum_{h=1}^H \beta^{r,k}_h(s_h^k, a_h^k, b_h^k) & = C\sum_{k=1}^K\sum_{h=1}^H   \sqrt{\frac{
\log (|\cS_1| |\cS_2| |\cA| |\cB| H K /\delta)}{ \max \{ N^k_h(s_h^{1,k}, s_h^{2,k},a_h^k,b_h^k), 1\} }} \\
& = C\sum_{k=1}^K\sum_{h=1}^H    \sqrt{\frac{
\log (|\cS_1| |\cS_2| |\cA| |\cB| H K /\delta)}{  N^k_h(s_h^{1,k}, s_h^{2,k},a_h^k,b_h^k)}}  \\
&\leq C \sum_{h=1}^H ~ \sum_{\substack{(s^1,s^2, a, b)\in \cS_1\times \cS_2\times\cA \times \cB\\ N^K_h(s^1, s^2, a, b) > 0}}\sum_{n=1}^{N^K_h(s^1, s^2, a, b)}  \sqrt{\frac{ \log (|\cS_1| |\cS_2| |\cA| |\cB| H K /\delta)}{n}},
\end{align*}
where the second equality is because $(s_h^{1,k}, s_h^{2,k},a_h^k,b_h^k)$ is visited such that $N^k_h(s_h^{1,k}, s_h^{2,k},a_h^k,b_h^k) \geq 1$. In addition, we have
\begin{align*}
&\sum_{h=1}^H ~ \sum_{\substack{(s^1,s^2, a, b)\in \cS_1\times \cS_2\times\cA \times \cB\\ N^K_h(s^1, s^2, a, b) > 0}}\sum_{n=1}^{N^K_h(s^1, s^2, a, b)}  \sqrt{\frac{ \log (|\cS_1| |\cS_2| |\cA| |\cB| H K /\delta)}{n}} \\
&\qquad\leq \sum_{h=1}^H ~ \sum_{(s^1,s^2, a, b)\in \cS_1\times \cS_2\times\cA \times \cB}  \cO \left(\sqrt{ N^K_h(s^1, s^2, a, b) \log \frac{|\cS_1| |\cS_2| |\cA| |\cB| HK}{\delta}} \right) \\
&\qquad \leq \cO \left(H \sqrt{ K |\cS_1||\cS_2||\cA||\cB| \log \frac{|\cS_1||\cS_2 |\cA| |\cB| HK}{\delta}} \right),
\end{align*}
where the last inequality is based on the consideration that $\sum_{(s^1,s^2, a, b)\in \cS_1\times \cS_2\times\cA \times \cB} N_h^K(s^1, s^2,a,b) = K$ such that $\sum_{(s^1,s^2, a, b)\in \cS_1\times \cS_2\times\cA \times \cB} \sqrt{ N^K_h(s^1, s^2, a, b)} \leq \cO\left(\sqrt{ K |\cS_1||\cS_2||\cA||\cB|}\right) $ when $K$ is sufficiently large. Putting the above together, we obtain
\begin{align*}
\sum_{k=1}^K\sum_{h=1}^H \beta^{r,k}_h(s_h^k, a_h^k, b_h^k)  \leq \cO \left(H \sqrt{ K |\cS_1||\cS_2||\cA||\cB| \log \frac{|\cS_1||\cS_2 |\cA| |\cB| HK}{\delta}}\right).
\end{align*}
Similarly, we have
\begin{align*}
&\sum_{k=1}^K \sum_{h=1}^H \beta^{\cP,k}_h(s_h^k, a_h^k, b_h^k) \\
&\qquad  = \sum_{k=1}^K\sum_{h=1}^H   \left(\sqrt{\frac{2H^2|\cS_1| \log (2|\cS_1| |\cA| H K/\delta)}{\max \{N_h^k(s_h^{1,k},a_h^k), 1\}}} +  \sqrt{\frac{2H^2|\cS_2| \log (2|\cS_2| |\cB| H K/\delta)}{\max \{N_h^k(s_h^{2,k},b_h^k), 1\}}} \right)\\
&\qquad\leq \cO \left(H \sqrt{ K |\cS_1|^2 |\cA| H^2  \log \frac{2|\cS_1| |\cA| HK}{\delta}} + H \sqrt{ K |\cS_2|^2 |\cB| H^2  \log \frac{2|\cS_2| |\cB| HK}{\delta}} \right).
\end{align*}
Thus, by \eqref{eq:bias_diff_al_it}, with probability at least $1-\delta$, we have
\begin{align*}
\sum_{k=1}^K \overline{V}_1^k(s_1)  - \sum_{k=1}^K  V_1^{\mu^k, \nu^k}(s_1) \leq \tilde{\cO} ( \sqrt{  |\cS_1|^2 |\cA| H^4 K } + \sqrt{  |\cS_2|^2 |\cB| H^4 K } + \sqrt{  |\cS_1||\cS_2| |\cA| |\cB| H^2 K } ),
\end{align*}
where $\tilde{\cO}$ hides logarithmic terms. This completes the proof.
\end{proof}

Before presenting the next lemma, we first show the following definition of confidence set for the proof of the next lemma.

\begin{definition}[Confidence Set for Player 2] Define the following confidence set for transition models for Player 2
\begin{align*}
\Upsilon^{2,k} := \Big\{\tilde{\cP} : &\left|\tilde{\cP}_h(s^2{}'|s^2,b) - \hat{\cP}^{2,k}_h(s^2{}'|s^2,b)\right| \leq \epsilon_h^{2,k}, ~\|\tilde{\cP}_h(\cdot|s^2,b)\|_1=1, \\
& \text{ and }~\tilde{\cP}_h(s^2{}'|s^2,b) \geq 0, ~\forall (s^2,b,s^2{}')\in \cS_2\times \cB \times \cS_2, \forall k\in [K] \Big\}
\end{align*}
where we define
\begin{align*}
& \epsilon_h^{2,k} := 2 \sqrt{\frac{\hat{\cP}^{2,k}_h(s^2{}'|s^2,b) \log (|\cS_2| |\cB| H K/\delta' )}{\max\{ N_h^k(s^2,b)-1, 1 \} }} +  \frac{14\log (|\cS_2| |\cB| H K/\delta')}{3\max\{ N_h^k(s^2,b)-1, 1 \}}
\end{align*}
with $N_h^k(s^2,b):=\sum_{\tau = 1}^k \mathbbm{1}\{(s^2,b) = (s_h^{2,\tau}, b_h^\tau)\}$, and $\hat{\cP}^{2,k}$ being the empirical transition model for Player 2.

\end{definition}

\begin{lemma} \label{lem:stationary_dist_err_P2} With probability at least $1-\delta$, the difference between $q_h^{\nu^k, \cP^2}$ and $d_h^{\nu^k, \hat{\cP}^{2,k}}$ is bounded as
 \begin{align*}
\sum_{k=1}^K\sum_{h=1}^H\sum_{s^2\in \cS_2} \left|q_h^{\nu^k, \cP^2}(s^2) - d_h^{\nu^k, \hat{\cP}^{2,k}}(s^2) \right|  \leq \tilde{\cO} \left( H^2 |\cS_2| \sqrt{ |\cB| K}  \right) .
\end{align*}
\end{lemma}

\begin{proof} 
By the definition of state distribution for Player 2, we have
\begin{align*}
&\sum_{k=1}^K\sum_{h=1}^H\sum_{s^2\in \cS_2}\left|q_h^{\nu^k, \cP^2}(s^2) - d_h^{\nu^k, \hat{\cP}^{2,k}}(s^2)\right| \\
&\qquad=  \sum_{k=1}^K \sum_{h=1}^H \sum_{s^2\in \cS_2} \left|\sum_{b\in \cB}   w_h^{2,k}(s^2,b) - \sum_{b\in \cB}   \hat{w}_h^{2,k}(s^2,b) \right| \\
&\qquad\leq  \sum_{k=1}^K \sum_{h=1}^H \sum_{s^2\in \cS_2}\sum_{b\in \cB}   \big| w_h^{2,k}(s,a) -  \hat{w}_h^{2,k}(s^2,b) \big|.
\end{align*}
where $ \hat{w}_h^{2,k}(s^2,b)$ is the occupancy measure under the empirical transition model $\hat{\cP}^{2,k}$ and the policy $\nu^k$. Then, since $\hat{\cP}^{2,k}\in \Upsilon^{2,k}$ always holds for any $k$,  by Lemma \ref{lem:occu_mea_bound_P2}, we can bound the last term of the bound inequality such that with probability at least $1-6\delta'$,
\begin{align*}
\sum_{k=1}^K\sum_{h=1}^H\sum_{s^2\in \cS_2} \left|q_h^{\nu^k, \cP^2}(s^2) - d_h^{\nu^k, \hat{\cP}^{2,k}}(s^2) \right| \leq \cE_1 + \cE_2.
\end{align*}
Then, we compute $\cE_1$ by Lemma \ref{lem:occu_err_P2}. With probability at least $1-2\delta'$, we have
\begin{align*}
\cE_1 &= \cO\left[  \sum_{h=2}^H \sum_{h'=1}^{h-1} \sum_{k=1}^K \sum_{s^2\in \cS_2} \sum_{b\in \cB} w_h^k(s^2,b) \left(  \sqrt{ \frac{  |\cS_2|\log (|\cS_2| |\cB| H K/\delta' )}{\max\{ N_h^k(s^2,b), 1\} }} + \frac{ \log (|\cS_2| |\cB| H K/\delta' )}{\max\{ N_h^k(s^2,b), 1\}}  \right) \right]\\
&= \cO\left[  \sum_{h=2}^H \sum_{h'=1}^{h-1} \sqrt{|\cS_2|}\left( \sqrt{|\cS_2| |\cB| K} + |\cS_2| |\cB| \log K + \log \frac{H}{\delta'}  \right)  \log \frac{|\cS_2| |\cB| H K}{\delta'}\right] \\
&= \cO\left[  \left( H^2 |\cS_2| \sqrt{ |\cB| K} + H^2 |\cS_2|^{3/2} |\cB| \log K + H^2 \sqrt{|\cS_2|}\log \frac{H}{\delta'}  \right)  \log \frac{|\cS_2| |\cB| H K}{\delta'}\right] \\
&= \tilde{\cO} \left( H^2 |\cS_2| \sqrt{ |\cB| K} \right),
\end{align*}
where we ignore  $\log K$ when $K$ is sufficiently large such that $\sqrt{K}$ dominates, and $ \tilde{\cO}$ hides logarithm dependence on $|\cS_2|$, $|\cB|$, $H$, $K$, and $ 1/\delta'$. In addition, $\cE_2$ depends on $\mathrm{ploy}(H, |\cS_2|, |\cB|)$ except the factor $\log \frac{|\cS_2| |\cB| H K}{\delta'}$ as shown in Lemma \ref{lem:occu_mea_bound_P2}. Thus, $\cE_2$ can be ignored comparing to $\cE_1$ if $K$ is sufficiently large. Therefore, we obtain that with probability at least $1-8\delta'$, the following inequality holds
\begin{align*}
\sum_{k=1}^K\sum_{h=1}^H\sum_{s^2\in \cS_2} \left|q_h^{\nu^k, \cP^2}(s^2) - d_h^{\nu^k, \hat{\cP}^{2,k}}(s^2) \right|  \leq \tilde{\cO} \left( H^2 |\cS_2| \sqrt{ |\cB| K}  \right) .
\end{align*}
We further let $\delta = 8\delta'$ such that $\log \frac{|\cS_2| |\cB| H K}{\delta'} = \log \frac{8|\cS_2| |\cB| H K}{\delta}$ which does not change the order as above.  Then, with probability at least $1-\delta$, we have $\sum_{k=1}^K\sum_{h=1}^H\sum_{s^2\in \cS_2} |q_h^{\nu^k, \cP^2}(s^2) - d_h^{\nu^k, \hat{\cP}^{2,k}}(s^2)| \leq \tilde{\cO} ( H^2 |\cS_2| \sqrt{ |\cB| K} )$. This completes the proof.
\end{proof}

\subsection{Other Supporting Lemmas}

\begin{lemma}\label{lem:pushback}  Let $f: \Lambda \mapsto \RR$ be a convex function, where $\Lambda$ is the probability simplex defined as $\Lambda: = \{ \xb \in \RR^d : \|\xb\|_1=1 \text{ and } \xb_i  \geq 0, \forall i \in [d]\}$. For any $\alpha \geq 0$,  $\zb \in \Lambda$,  and $\yb \in \Lambda^o$ where $\Lambda^o \subset \Lambda$ with only relative interior points of $\Lambda$, supposing $\xb^{\mathrm{opt}} = \argmin_{\xb \in \Lambda} f(\xb) + \alpha D_{\mathrm{KL}} (\xb, \yb)$, then the following inequality holds
\begin{align*}
f(\xb^{\mathrm{opt}}) + \alpha D_{\mathrm{KL}} (\xb^{\mathrm{opt}}, \yb) \leq f(\zb) + \alpha D_{\mathrm{KL}} (\zb, \yb) - \alpha D_{\mathrm{KL}} (\zb, \xb^{\mathrm{opt}}).
\end{align*}
\end{lemma}
This lemma is for mirror descent algorithms, whose proof can be obtained by slight modification from existing works \citep{tseng2008accelerated,nemirovski2009robust,wei2019online}.

The following lemmas are adapted from the recent papers \citep{efroni2020optimistic,jin2019learning}, where we can find their detailed proofs. 

\begin{lemma}\label{lem:conf_set_P2} With probability at least $1-4\delta'$, the true transition model $\cP^2$ satisfies that for any $k\in [K]$,
\begin{align*}
\cP \in \Upsilon^{2,k}.
\end{align*}

\end{lemma}

This lemma indicates that the estimated transition model $\hat{\cP}^{2,k}_h(s^2{}'|s^2,b)$ for Player 2 by \eqref{eq:estimate} is closed to the true transition model $\cP^2_h(s^2{}'|s^2,b)$ with high probability. The upper bound is by empirical Bernstein's inequality and the union bound. 

The next lemma is adapted from Lemma 10 in \citet{jin2019learning}.

\begin{lemma} \label{lem:occu_err_P2}We let $w_h^{2,k}(s^2,b)$ denote the occupancy measure at the $h$-th step of the $k$-th episode under the true transition model $\cP^2$ and the current policy $\nu^k$. Then, with probability at least $1-2\delta'$ we have for all $h\in[H]$, the following results hold
\begin{align*}
\sum_{k=1}^K \sum_{s^2\in \cS_2} \sum_{b\in \cB}\frac{w_h^k(s^2,b)}{\max\{N_h^k(s^2,b), 1\}} = \cO\left(|\cS_2| |\cB| \log K + \log \frac{H}{\delta'}\right),
\end{align*}
and
\begin{align*}
\sum_{k=1}^K \sum_{s^2\in \cS_2} \sum_{b\in \cB}\frac{w_h^k(s^2,b)}{\sqrt{\max\{N_h^k(s^2,b), 1\}}} = \cO\left(\sqrt{|\cS_2| |\cB| K} + |\cS_2| |\cB| \log K + \log \frac{H}{\delta'}\right).
\end{align*}

\end{lemma}

By Lemma \ref{lem:conf_set_P2} and Lemma \ref{lem:occu_err_P2}, we have the following lemma to show the difference of two occupancy measures, which is modified from parts of the proof of Lemma 4 in \citet{jin2019learning}.

\begin{lemma} \label{lem:occu_mea_bound_P2} For Player 2, we let $w_h^{2,k}(s^2,b)$ be the occupancy measure at the $h$-th step of the $k$-th episode under the true transition model $\cP^2$ and the current policy $\nu^k$, and $\tilde{w}_h^{2,k}(s^2,b)$ be the occupancy measure at the $h$-th step of the $k$-th episode under any transition model $\tilde{\cP}^{2,k} \in \Upsilon^{2,k}$ and the current policy $\nu^k$ for any $k$. Then, with probability at least $1-6\delta'$ we have for all $h\in[H]$, the following inequality holds
\begin{align*}
\sum_{k=1}^K \sum_{h=1}^K \sum_{s\in \cS_2} \sum_{b\in \cB}  \big|\tilde{w}_h^{2,k}(s^2,b) - w_h^{2,k}(s^2,b) \big|\leq \cE_1 + \cE_2, 
\end{align*}
where $\cE_1$ and $\cE_2$ are in the level of 
\begin{align*}
\cE_1 = \cO\left[  \sum_{h=2}^H \sum_{h'=1}^{h-1} \sum_{k=1}^K \sum_{s^2\in \cS_2} \sum_{b\in \cB} w_h^k(s^2,b) \left(  \sqrt{ \frac{  |\cS_2|\log (|\cS_2| |\cB| H K/\delta' )}{\max\{ N_h^k(s^2,b), 1\} }} + \frac{ \log (|\cS_2| |\cB| H K/\delta' )}{\max\{ N_h^k(s^2,b), 1\}}  \right) \right]
\end{align*} 
and
\begin{align*}
\cE_2 = \cO\left( \mathrm{poly}(H,|\cS_2|,|\cB|) \cdot\log \frac{|\cS_2| |\cB| H K}{\delta'} \right),
\end{align*}
where $ \mathrm{poly}(H,|\cS_2|,|\cB|)$ denotes the polynomial dependency on $H,|\cS_2|,|\cB|$.
\end{lemma}

\section{Proofs for Section \ref{sec:SCT}}

\begin{lemma} \label{lem:V_diff_1} At the $k$-th episode of Algorithm \ref{alg:po1}, the difference between value functions $V_1^{\mu^*, \nu^k}(s_1)$ and $V_1^{\mu^k, \nu^k}(s_1)$ is
\begin{align*}
&V_1^{\mu^*, \nu^k}(s_1) - V_1^{\mu^k, \nu^k}(s_1) \\
&\qquad =  \overline{V}_1^k(s_1) -  V_1^{\mu^k, \nu^k}(s_1) + \sum_{h=1}^H \EE_{\mu^*, \cP} \Big[ \big\langle \mu_h^*(\cdot | s_h)-\mu_h^k(\cdot | s_h), U_h^k(s_h,\cdot) \big\rangle_{\cA}  \Biggiven s_1 \Big] \\
&\qquad\quad + \sum_{h=1}^H \EE_{\mu^*, \cP, \nu^k} \big[ \overline{\varsigma}_h^k(s_h, a_h, b_h) \biggiven s_1 \big].
\end{align*}
where $s_h, a_h, b_h$ are random variables for state and actions, $U_h^k(s,a) := \langle \overline{Q}_h^k(s, a , \cdot), \nu_h^k(\cdot\given s) \rangle_\cB$, and we define the model prediction error of $Q$-function as
\begin{align}
\begin{aligned} \label{eq:pred_err_up} 
&\overline{\varsigma}_h^k(s, a, b) = r_h(s,a,b) +  \cP_h\overline{V}_{h+1}^k(s, a) - \overline{Q}_h^k(s,a,b).
\end{aligned}
\end{align} 
\end{lemma}

\begin{proof} We start the proof by decomposing the value function difference as 
\begin{align}
V_1^{\mu^*, \nu^k}(s_1) - V_1^{\mu^k, \nu^k}(s_1) = V_1^{\mu^*, \nu^k}(s_1) - \overline{V}_1^k(s_1) + \overline{V}_1^k(s_1) -  V_1^{\mu^k, \nu^k}(s_1).  \label{eq:V_diff_1_decomp}
\end{align}
Note that the term $\overline{V}_1^k(s_1) -  V_1^{\mu^k, \nu^k}(s_1)$ is the bias between the estimated value function $\overline{V}_1^k(s_1)$ generated by Algorithm \ref{alg:po1} and the value function $V_1^{\mu^k, \nu^k}(s_1)$ under the true transition model $\cP$ at the $k$-th episode.

We focus on analyzing the other term $V_1^{\mu^*, \nu^k}(s_1) - \overline{V}_1^k(s_1)$ in this proof. For any $h$ and $s$, we have the following decomposition
\begin{align}
\begin{aligned} \label{eq:V_diff_1_init}
&V_h^{\mu^*, \nu^k}(s) - \overline{V}_h^k(s) \\
&\qquad=  [\mu^*_h(\cdot | s)]^\top Q_h^{\mu^*, \nu^k}(s, \cdot , \cdot) \nu^k_h(\cdot | s) -  \big[\mu_h^k(\cdot | s)\big]^\top \overline{Q}_h^ k(s, \cdot , \cdot) \nu_h^k(\cdot | s)\\
&\qquad=  [\mu^*_h(\cdot | s)]^\top Q_h^{\mu^*, \nu^k}(s, \cdot , \cdot) \nu^k_h(\cdot | s) -  [\mu_h^*(\cdot | s)]^\top \overline{Q}_h^k(s, \cdot , \cdot) \nu_h^k(\cdot | s) \\
&\qquad \quad +   [\mu_h^*(\cdot | s)]^\top \overline{Q}_h^k(s, \cdot , \cdot) \nu_h^k(\cdot | s) -  \big[\mu_h^k(\cdot | s)\big]^\top \overline{Q}_h^k(s, \cdot , \cdot) \nu_h^k(\cdot | s) \\
&\qquad =  [\mu^*_h(\cdot | s)]^\top \big[Q_h^{\mu^*, \nu^k}(s, \cdot , \cdot)  -   \overline{Q}_h^k(s, \cdot , \cdot) \big] \nu_h^k(\cdot | s) \\
&\qquad \quad +   \big[\mu_h^*(\cdot | s)-\mu_h^k(\cdot | s)\big]^\top \overline{Q}_h^k(s, \cdot , \cdot) \nu_h^k(\cdot | s),
\end{aligned}
\end{align}
where the first inequality is by the definition of $V_h^{\mu^*, \nu^k}$ in \eqref{eq:bellman_V} and the definition of $\overline{V}_h^k$ in Line \ref{line:V_up} of Algorithm \ref{alg:po1}. Moreover, by the definition of $Q_h^{\mu^*, \nu^k}(s, \cdot , \cdot)$ in \eqref{eq:bellman_Q} and the model prediction error $\overline{\varsigma}_h^k$ for Player 1 in \eqref{eq:pred_err_up}, we have
\begin{align*}
&[\mu^*_h(\cdot | s)]^\top \big[Q_h^{\mu^*, \nu^k}(s, \cdot , \cdot)  -   \overline{Q}_h^k(s, \cdot , \cdot) \big] \nu_h^k(\cdot | s) \\
& \qquad =  \sum_{a \in \cA} \sum_{b\in \cB} \mu^*_h(a | s) \bigg[ \sum_{s'\in \cS}  \cP_h(s'|s, a) \big[V_{h+1}^{\mu^*, \nu^k}(s') -  \overline{V}_{h+1}^k(s')\big] + \overline{\varsigma}_h^k(s,a, b) \bigg] \nu_h^k(b | s)\\
&\qquad =  \sum_{a \in \cA} \sum_{s'\in \cS} \mu^*_h(a | s) \cP_h(s'|s, a) \big[V_{h+1}^{\mu^*, \nu^k}(s') -  \overline{V}_{h+1}^k(s')\big]  +  \sum_{a \in \cA} \sum_{b\in \cB}  \mu^*_h(a | s) \overline{\varsigma}_h^k(s,a, b)  \nu_h^k(b | s).
\end{align*}
where the last equality holds due to $\sum_{b\in \cB} \nu_h^k(b\given s) = 1$.
Combining this equality with \eqref{eq:V_diff_1_init} gives
\begin{align}
\begin{aligned} \label{eq:V_diff_1_rec}
V_h^{\mu^*, \nu^k}(s) - \overline{V}_h^k(s)  & =  \sum_{a \in \cA}  \sum_{s'\in \cS} \mu^*_h(a | s) \cP_h(s'|s, a) \big[V_{h+1}^{\mu^*, \nu^k}(s') -  \overline{V}_{h+1}^k(s')\big]   \\
&\quad +  \sum_{a \in \cA} \sum_{b\in \cB}  \mu^*_h(a | s) \overline{\varsigma}_h^k(s,a, b)  \nu_h^k(b | s) \\
&\quad + \sum_{a \in \cA} \sum_{b\in \cB}  \big[\mu_h^*(a | s)-\mu_h^k(a | s)\big]  \overline{Q}_h^k(s, a, b) \nu_h^k(b | s).
\end{aligned}
\end{align}
Note that \eqref{eq:V_diff_1_rec} indicates a recursion of the value function difference $V_h^{\mu^*, \nu^k}(s) - \overline{V}_h^k(s)$. Since we define $V_{H+1}^{\mu^*, \nu^k}(s) = 0$ and $\overline{V}_{H+1}^k(s) = 0$, by recursively applying \eqref{eq:V_diff_1_rec} from $h = 1$ to $H$, we obtain
\begin{align}  
\begin{aligned}\label{eq:V_diff_1_al}
V_1^{\mu^*, \nu^k}(s_1) - \overline{V}_1^k(s_1)  & =   \sum_{h=1}^H \EE_{\mu^*, \cP} \big\{  [\mu^*_h(\cdot| s_h)]^\top \overline{\varsigma}_h^k(s_h, \cdot, \cdot)  \nu_h^k(\cdot| s_h)  \biggiven s_1 \big\} \\
& \quad  + \sum_{h=1}^H \EE_{\mu^*, \cP} \big\{ \big[\mu_h^*(\cdot | s_h)-\mu_h^k(\cdot | s_h)\big]^\top  \overline{Q}_h^k(s_h, \cdot, \cdot) \nu_h^k(\cdot | s_h)  \biggiven  s_1 \big\},
\end{aligned}
\end{align}
where $s_h$ are a random variables denoting the state at the $h$-th step following a distribution determined jointly by $\mu^*, \cP$. Further combining  \eqref{eq:V_diff_1_al} with \eqref{eq:V_diff_1_decomp}, we eventually have
\begin{align*}
&V_1^{\mu^*, \nu^k}(s_1) - V_1^{\mu^k, \nu^k}(s_1) \\
&\qquad  =  \overline{V}_1^k(s_1) -  V_1^{\mu^k, \nu^k}(s_1) + \sum_{h=1}^H \EE_{\mu^*, \cP} \big\{  [\mu^*_h(\cdot| s_h)]^\top \overline{\varsigma}_h^k(s_h, \cdot, \cdot)  \nu_h^k(\cdot| s_h)  \biggiven s_1 \big\} \\
&\qquad \quad  + \sum_{h=1}^H \EE_{\mu^*, \cP} \big\{ \big[\mu_h^*(\cdot | s_h)-\mu_h^k(\cdot | s_h)\big]^\top  \overline{Q}_h^k(s_h, \cdot, \cdot) \nu_h^k(\cdot | s_h)  \biggiven  s_1 \big\},
\end{align*}
which is equivalent to the result in this lemma. 
This completes our proof.
\end{proof}

%
\begin{lemma} \label{lem:V_diff_2} At the $k$-th episode of Algorithm \ref{alg:po2}, with probability at least $1-\delta$, the difference between the value functions $V_1^{\mu^k, \nu^k}(s_1)$ and $V_1^{\mu^k, \nu^*}(s_1)$ for all $k\in [K]$ is decomposed as
\begin{align*}
&V_1^{\mu^k, \nu^k}(s_1) - V_1^{\mu^k, \nu^*}(s_1) \\
&\qquad \leq 2 \sum_{h=1}^H  \EE_{\mu^k, \cP, \nu^k} \big[\beta_h^{r,k}(s_h,a_h,b_h) \biggiven s_1\big]   + \sum_{h=1}^H \sum_{s\in \cS} d_h^{\mu^k, \hat{\cP}^k}(s) \big[\mu_h^k(\cdot| s)\big]^\top  \underline{\varsigma}_h^k(s, \cdot, \cdot)\nu_h^*(\cdot | s) \nonumber\\ 
&\qquad\quad +  \sum_{h=1}^H \sum_{s\in \cS} d_h^{\mu^k, \hat{\cP}^k}(s) \big\langle W_h^k(s, \cdot),  \nu_h^k(\cdot | s) - \nu_h^*(\cdot | s) \big\rangle_\cB + 2\sum_{h=1}^H \sum_{s\in \cS}  \left|q_h^{\mu^k, \cP}(s)- d_h^{\mu^k, \hat{\cP}^k}(s)\right|,\nonumber
\end{align*}
where $s_h, a_h, b_h$ are random variables for state and actions, $W_h^k(s, b) =  \langle \tilde{r}_h^k(s, \cdot , b),   \mu^k_h( \cdot \given s)  \rangle_\cA $, and we define the error term as
\begin{align}\label{eq:pred_err_up2}
\underline{\varsigma}_h^k(s, a, b) = \tilde{r}_h^k (s, a, b)-r_h (s, a, b).
\end{align} 
\end{lemma}
\begin{proof}

We start our proof by decomposing the value difference term for any $h$ and $s$ as follows
\begin{align}
\begin{aligned} \label{eq:V_diff_2_init}
&V_h^{\mu^k, \nu^k}(s) - V_h^{\mu^k, \nu^*}(s) \\
&\qquad=  \big[\mu^k_h(\cdot | s)\big]^\top Q_h^{\mu^k, \nu^k}(s, \cdot , \cdot) \nu^k_h(\cdot | s) -  \big[\mu_h^k(\cdot | s)\big]^\top Q_h^{\mu^k, \nu^*}(s, \cdot , \cdot) \nu_h^*(\cdot | s)\\
&\qquad =  \big[\mu_h^k(\cdot | s)\big]^\top Q_h^{\mu^k, \nu^k}(s, \cdot , \cdot) \big[\nu_h^k(\cdot | s) - \nu_h^*(\cdot | s)\big]  \\
&\qquad \quad +   \big[\mu^k_h(\cdot | s)\big]^\top \big[Q_h^{\mu^k, \nu^k}(s, \cdot , \cdot)  -   Q_h^{\mu^k, \nu^*}(s, \cdot , \cdot) \big] \nu_h^*(\cdot | s),
\end{aligned}
\end{align}
where the first equality is by the Bellman equation for $V_h^{\mu, \nu}(s) $ in \eqref{eq:bellman_V} and the second equality is obtained by subtracting and adding the term $\big[\mu_h^k(\cdot | s)\big]^\top  Q_h^{\mu^k, \nu^k}(s, \cdot , \cdot) \nu_h^*(\cdot | s)$ in the first equality. Moreover, by the Bellman equation for $Q_h^{\mu, \nu}$ in \eqref{eq:bellman_Q}, we can expand the last term in \eqref{eq:V_diff_2_init} as 
\begin{align}
\begin{aligned} \label{eq:V_diff_2_init2}
& \big[\mu^k_h(\cdot | s)\big]^\top \big[Q_h^{\mu^k, \nu^k}(s, \cdot , \cdot)  -   Q_h^{\mu^k, \nu^*}(s, \cdot , \cdot) \big] \nu_h^*(\cdot | s) \\
& \qquad =  \sum_{a \in \cA} \sum_{b\in \cB} \mu^k_h(a | s)  \sum_{s'\in \cS}  \cP_h(s'|s, a) \big[ V_{h+1}^{\mu^k, \nu^k}(s')- V_{h+1}^{\mu^k, \nu^*}(s') \big]  \nu_h^*(b | s)\\
&\qquad =  \sum_{a \in \cA} \sum_{s'\in \cS} \mu^k_h(a | s) \cP_h(s'|s, a) \big[V_{h+1}^{\mu^k, \nu^k}(s')- V_{h+1}^{\mu^k, \nu^*}(s')\big].
\end{aligned}
\end{align}
where the last equality holds due to $\sum_{b\in \cB} \nu_h^*(b\given s) = 1$.
Combining \eqref{eq:V_diff_2_init2} with \eqref{eq:V_diff_2_init} gives
\begin{align}
\begin{aligned} \label{eq:V_diff_2_rec}
V_h^{\mu^k, \nu^k}(s) - V_h^{\mu^k, \nu^*}(s)   & =  \sum_{a \in \cA} \sum_{b\in \cB}  \mu_h^k(a | s)  Q_h^{\mu^k, \nu^k}(s, a, b) \big[\nu_h^k(b | s) - \nu_h^*(b | s)\big] \\
&\quad +  \sum_{a \in \cA}  \sum_{s'\in \cS} \mu^k_h(a | s) \cP_h(s'|s, a) \big[V_{h+1}^{\mu^k, \nu^k}(s')- V_{h+1}^{\mu^k, \nu^*}(s')\big].
\end{aligned}
\end{align}
Note that \eqref{eq:V_diff_2_rec} indicates a recursion of the value function difference $V_h^{\mu^k, \nu^k}(s) - V_h^{\mu^k, \nu^*}(s)$. Since we define $V_{H+1}^{\mu, \nu}(s) = 0$ for any $\mu$ and $\nu$, by recursively applying \eqref{eq:V_diff_2_rec} from $h = 1$ to $H$, we obtain
\begin{align}  
\begin{aligned}\label{eq:V_diff_2_al}
&V_1^{\mu^k, \nu^k}(s_1) - V_1^{\mu^k, \nu^*}(s_1)   \\
&\qquad =  \sum_{h=1}^H \EE_{\mu^k, \cP} \big\{ \big[\mu_h^k(\cdot | s_h)\big]^\top  Q_h^{\mu^k, \nu^k}(s_h, \cdot, \cdot) \big[\nu_h^k(\cdot | s_h)-\nu_h^*(\cdot | s_h)\big]  \biggiven  s_1 \big\},
\end{aligned}
\end{align}
where $s_h$ are a random variables following a distribution determined jointly by $\mu^k, \cP$. Note that since we have defined the distribution of $s_h$ under $\mu^k$ and $\cP$ as
\begin{align*}
q_h^{\mu^k, \cP}(s) = \Pr\big(s_h = s \biggiven  \mu^k, \cP, s_1  \big), 
\end{align*}
we can rewrite \eqref{eq:V_diff_2_al} as
\begin{align} 
\begin{aligned}\label{eq:V_diff_2_al2}
&V_1^{\mu^k, \nu^k}(s_1) - V_1^{\mu^k, \nu^*}(s_1)  \\
&\qquad =  \sum_{h=1}^H \sum_{s\in \cS} \sum_{a\in \cA}\sum_{b\in \cB} q_h^{\mu^k, \cP}(s) \mu_h^k(a| s) Q_h^{\mu^k, \nu^k}(s, a, b) \big[\nu_h^k(b | s)-\nu_h^*(b | s)\big].
\end{aligned}
\end{align}
By plugging the Bellman equation for Q-function as \eqref{eq:bellman_Q} into \eqref{eq:V_diff_2_al2}, we further expand \eqref{eq:V_diff_2_al2} as
\begin{align*}  
&V_1^{\mu^k, \nu^k}(s_1) - V_1^{\mu^k, \nu^*}(s_1)  \\
&\qquad  =  \sum_{h=1}^H \sum_{s\in \cS} \sum_{a\in \cA}\sum_{b\in \cB} q_h^{\mu^k, \cP}(s) \mu_h^k(a| s) \big[ r_h (s, a, b) +  \big\langle  \cP_h(\cdot | s, a), V_{h+1}^{\mu^k, \nu^k} (\cdot) \big\rangle \big] [\nu_h^k(b | s)-\nu_h^*(b | s)]\\
&\qquad  =  \sum_{h=1}^H \sum_{s\in \cS} \sum_{a\in \cA}\sum_{b\in \cB} q_h^{\mu^k, \cP}(s) \mu_h^k(a| s) \left[ r_h (s, a, b) \right] [\nu_h^k(b | s)-\nu_h^*(b | s)]\\
&\qquad  =  \sum_{h=1}^H \sum_{s\in \cS} q_h^{\mu^k, \cP}(s) [\mu_h^k(\cdot| s)]^\top   r_h (s, \cdot, \cdot) \big[\nu_h^k(\cdot | s)-\nu_h^*(\cdot | s)\big],
\end{align*}
where the second equality by
\begin{align*}
&\sum_{h=1}^H \sum_{s\in \cS} \sum_{a\in \cA}\sum_{b\in \cB} q_h^{\mu^k, \cP}(s) \mu_h^k(a| s)  \big\langle  \cP_h(\cdot | s, a), V_{h+1}^{\mu^k, \nu^k} (\cdot) \big\rangle_\cS  [\nu_h^k(b | s)-\nu_h^*(b | s)] \\
&\qquad = \sum_{h=1}^H \sum_{s\in \cS} \sum_{a\in \cA} q_h^{\mu^k, \cP}(s) \mu_h^k(a| s)  \big\langle  \cP_h(\cdot | s, a), V_{h+1}^{\mu^k, \nu^k} (\cdot) \big\rangle_\cS  \sum_{b\in \cB}[\nu_h^k(b | s)-\nu_h^*(b | s)] \\
&\qquad = 0.
\end{align*}
In particular, the last equality above is due to
\begin{align*}
\sum_{b\in \cB}\big[\nu_h^k(b | s)-\nu_h^*(b | s)\big] = 1-1=0.
\end{align*}
Thus, we have 
\begin{align}
V_1^{\mu^k, \nu^k}(s_1) - V_1^{\mu^k, \nu^*}(s_1)   =  \sum_{h=1}^H \sum_{s\in \cS} q_h^{\mu^k, \cP}(s) \big[\mu_h^k(\cdot| s)\big]^\top   r_h (s, \cdot, \cdot) \big[\nu_h^k(\cdot | s)-\nu_h^*(\cdot | s)\big]. \label{eq:decomp2_al}
\end{align}
Recall that we also define the estimate of the state reaching probability $q_h^{\mu^k, \cP}(s)$ as 
\begin{align*}
d_h^{\mu^k,\hat{\cP}^k}(s) = \Pr\big(s_h = s \biggiven  \mu^k, \hat{\cP}^k, s_1  \big).
\end{align*}
Now we define the following term associated with $\hat{\cP}^k$, $\hat{r}^k$, $\mu^k$, $\nu^k$, and the initial state $s_1$ as
\begin{align*}
\underline{V}^k_1 := \sum_{h=1}^H \sum_{s\in \cS} d_h^{\mu^k, \hat{\cP}^k}(s) \big[\mu_h^k(\cdot| s)\big]^\top   \tilde{r}_h^k (s, \cdot, \cdot) \nu_h^k(\cdot | s),
\end{align*}
with $\tilde{r}$ defined in Line \ref{line:def_til_r} of Algorithm \ref{alg:po2}, which is
\begin{align*}
\tilde{r}_h^k(s, a, b) =  \max \big\{ \hat{r}^k_h(s,a,b)  - \beta_h^{r,k}(s,a,b), ~0 \big\}.
\end{align*}
Thus, by \eqref{eq:decomp2_al}, we have the following decomposition
\begin{align}
\begin{aligned} \label{eq:decomp2_al1}
&\hspace{-0.2cm}V_1^{\mu^k, \nu^k}(s_1) - V_1^{\mu^k, \nu^*}(s_1)  \\
&\hspace{-0.2cm} = V_1^{\mu^k, \nu^k}(s_1) - V_1^{\mu^k, \nu^*}(s_1) - \underline{V}^k_1 + \underline{V}^k_1 \\
&\hspace{-0.2cm}=  \underbrace{\sum_{h=1}^H \sum_{s\in \cS} \left\{ q_h^{\mu^k, \cP}(s) \big[\mu_h^k(\cdot| s)\big]^\top   r_h (s, \cdot, \cdot) \nu_h^k(\cdot | s) - d_h^{\mu^k, \hat{\cP}^k}(s) \big[\mu_h^k(\cdot| s)\big]^\top   \tilde{r}_h^k (s, \cdot, \cdot) \nu_h^k(\cdot | s) \right\}}_{\text{Term(I)}}\\
&\hspace{-0.2cm}\quad+ \underbrace{\sum_{h=1}^H \sum_{s\in \cS} \left\{d_h^{\mu^k, \hat{\cP}^k}(s) \big[\mu_h^k(\cdot| s)\big]^\top   \tilde{r}_h^k (s, \cdot, \cdot) \nu_h^k(\cdot | s) - q_h^{\mu^k, \cP}(s) \big[\mu_h^k(\cdot| s)\big]^\top   r_h (s, \cdot, \cdot) \nu_h^*(\cdot | s) \right\}}_{\text{Term(II)}}.
\end{aligned}
\end{align}
We first bound Term(I) as
\begin{align} 
\text{Term(I)} &= \sum_{h=1}^H \sum_{s\in \cS} \left\{ q_h^{\mu^k, \cP}(s) \big[\mu_h^k(\cdot| s)\big]^\top   r_h (s, \cdot, \cdot) \nu_h^k(\cdot | s) - d_h^{\mu^k, \hat{\cP}^k}(s) \big[\mu_h^k(\cdot| s)\big]^\top   \tilde{r}_h^k (s, \cdot, \cdot) \nu_h^k(\cdot | s) \right\}\nonumber\\
&= \sum_{h=1}^H \sum_{s\in \cS} q_h^{\mu^k, \cP}(s) \big[\mu_h^k(\cdot| s)\big]^\top   \big[r_h (s, \cdot, \cdot) - \tilde{r}^k_h (s, \cdot, \cdot)\big] \nu_h^k(\cdot | s) \nonumber\\
&\quad + \sum_{h=1}^H \sum_{s\in \cS}  \Big[q_h^{\mu^k, \cP}(s)-d_h^{\mu^k, \hat{\cP}^k}(s)\Big] \big[\mu_h^k(\cdot| s)\big]^\top   \tilde{r}_h^k (s, \cdot, \cdot) \nu_h^k(\cdot | s)  \nonumber\\
&\leq 2\sum_{h=1}^H  \EE_{\mu^k, \cP, \nu^k} \big[\beta_h^{r,k}(s,a,b)\big]  + \sum_{h=1}^H \sum_{s\in \cS}  \left|q_h^{\mu^k, \cP}(s)- d_h^{\mu^k, \hat{\cP}^k}(s)\right|, \label{eq:sc_p2_term1}
\end{align}
where the inequality is due to $|\hat{r}_h^k (s, a, b) - r_h (s, a, b)| \leq \beta_h^{r,k}(s,a,b)$ with probability at least $1-\delta$ by Lemma \ref{lem:r_bound} such that we have
\begin{align*}
r_h (s, a, b) -\tilde{r}_h^k (s, a, b) &= r_h (s, a, b) -\max\big\{ \hat{r}_h^k(s,a,b) - \beta_h^{r,k}(s,a,b), 0\big\} \\
&= \min\big\{ r_h (s, a, b) - \hat{r}_h^k(s,a,b) + \beta_h^{r,k}(s,a,b), r_h (s, a, b)\big\} \\
&\leq r_h (s, a, b) - \hat{r}_h^k(s,a,b) + \beta_h^{r,k}(s,a,b) \leq 2\beta_h^{r,k}(s,a,b)
\end{align*}
and then
\begin{align*}
&\sum_{s\in \cS}  q_h^{\mu^k, \cP}(s) \big[\mu_h^k(\cdot| s)\big]^\top   \big[r_h (s, \cdot, \cdot) -\tilde{r}_h^k (s, \cdot, \cdot)\big] \nu_h^k(\cdot | s) \leq 2 \EE_{\mu^k, \cP, \nu^k} \big[\beta_h^{r,k}(s,a,b)\big]. 
\end{align*}
In addition, the inequality in \eqref{eq:sc_p2_term1} is also due to
\begin{align*}
\bigg|\big[\mu_h^k(\cdot| s)\big]^\top   \tilde{r}^k_h (s, \cdot, \cdot) \nu_h^k(\cdot | s) \bigg| &\leq \bigg|\sum_a \sum_b \mu_h^k(a| s) \tilde{r}^k_h (s, a, b) \nu_h^k(b | s)\bigg| \\
&\leq \sum_a \sum_b \mu_h^k(a| s) \cdot \big|\tilde{r}^k_h (s, a, b)\big|\cdot \nu_h^k(b | s) \leq 1,
\end{align*}
because of $0\leq \tilde{r}^k_h (s, a, b) = \max\big\{ \hat{r}_h^k(s, a, b) -\beta^{r,k}_h(s, a, b), 0\big\} \leq \hat{r}_h^k(s, a, b) \leq 1$. Therefore, with probability at least $1-\delta$, we have
\begin{align}\label{eq:decomp2_term1}
&\text{Term(I)} \leq 2 \sum_{h=1}^H  \EE_{\mu^k, \cP, \nu^k} \big[\beta_h^{r,k}(s_h,a_h,b_h)\big]  + \sum_{h=1}^H \sum_{s\in \cS}  \left|q_h^{\mu^k, \cP}(s)- d_h^{\mu^k, \hat{\cP}^k}(s)\right|.
\end{align}
Next, we bound Term(II) in the following way
\begin{align*}
\text{Term(II)} &= \sum_{h=1}^H \sum_{s\in \cS} d_h^{\mu^k, \hat{\cP}^k}(s) \big[\mu_h^k(\cdot| s)\big]^\top   \tilde{r}_h^k (s, \cdot, \cdot) \big[\nu_h^k(\cdot | s) - \nu_h^*(\cdot | s) \big] \\
&\quad  + \sum_{h=1}^H \sum_{s\in \cS} \Big[d_h^{\mu^k, \hat{\cP}^k}(s)-q_h^{\mu^k, \cP}(s)\Big] \big[\mu_h^k(\cdot| s)\big]^\top   r_h (s, \cdot, \cdot) \nu_h^*(\cdot | s)  \\
&\quad   + \sum_{h=1}^H \sum_{s\in \cS} d_h^{\mu^k, \hat{\cP}^k}(s) \big[\mu_h^k(\cdot| s)\big]^\top  \underline{\varsigma}_h^k (s, \cdot, \cdot) \nu_h^*(\cdot | s),
\end{align*}
where $\underline{\varsigma}_h^k (s, a, b)$ is defined in \eqref{eq:pred_err_up2}.
Here the first term in the above equality is associated with the mirror descent step in Algorithm \ref{alg:po2}. The second term can be similarly bounded by $\sum_{h=1}^H \sum_{s\in \cS}  |q_h^{\mu^k, \cP}(s)- d_h^{\mu^k, \hat{\cP}^k}(s)|$.  Thus, we have
\begin{align}
\text{Term(II)} &\leq \sum_{h=1}^H \sum_{s\in \cS} d_h^{\mu^k, \hat{\cP}^k}(s) \big[\mu_h^k(\cdot| s)\big]^\top   \tilde{r}_h^k (s, \cdot, \cdot) \big[\nu_h^k(\cdot | s) - \nu_h^*(\cdot | s) \big] \label{eq:decomp2_term2} \\ 
&\quad + \sum_{h=1}^H \sum_{s\in \cS}  \left|q_h^{\mu^k, \cP}(s)- d_h^{\mu^k, \hat{\cP}^k}(s)\right| + \sum_{h=1}^H \sum_{s\in \cS} d_h^{\mu^k, \hat{\cP}^k}(s) \big[\mu_h^k(\cdot| s)\big]^\top  \underline{\varsigma}_h^k (s, \cdot, \cdot) \nu_h^*(\cdot | s). \nonumber
\end{align}
Combining \eqref{eq:decomp2_term1}, \eqref{eq:decomp2_term2} with \eqref{eq:decomp2_al1}, we obtain that with probability at least $1-\delta$, the following inequality holds
\begin{align*}
&V_1^{\mu^k, \nu^k}(s_1) - V_1^{\mu^k, \nu^*}(s_1)\\
&\qquad \leq 2 \sum_{h=1}^H  \EE_{\mu^k, \cP, \nu^k} \big[\beta_h^{r,k}(s_h,a_h,b_h) \biggiven s_1\big]   + \sum_{h=1}^H \sum_{s\in \cS} d_h^{\mu^k, \hat{\cP}^k}(s) \big[\mu_h^k(\cdot| s)\big]^\top  \underline{\varsigma}_h^k(s, \cdot, \cdot)\nu_h^*(\cdot | s) \nonumber\\ 
&\qquad\quad +  \sum_{h=1}^H \sum_{s\in \cS} d_h^{\mu^k, \hat{\cP}^k}(s) \big\langle W_h^k(s, \cdot),  \nu_h^k(\cdot | s) - \nu_h^*(\cdot | s) \big\rangle_\cB + 2\sum_{h=1}^H \sum_{s\in \cS}  \left|q_h^{\mu^k, \cP}(s)- d_h^{\mu^k, \hat{\cP}^k}(s)\right|,
\end{align*}
where $W_h^k(s, b) =  \langle \tilde{r}_h^k(s, \cdot , b),   \mu^k_h( \cdot \given s)  \rangle_\cA $. This completes our proof.
\end{proof}

\begin{lemma} \label{lem:mirror_1} With setting $\eta = \sqrt{  \log |\cA|/(KH^2)}$, the mirror ascent steps of Algorithm \ref{alg:po1} lead to 
\begin{align*}
\sum_{k=1}^K \sum_{h=1}^H \EE_{\mu^*, \cP} \Big[\big\langle  \mu_h^*(\cdot|s)- \mu_h^k(\cdot|s),  U_h^k(s, \cdot)\big\rangle_\cA \Big]  \leq  \cO\left( \sqrt{H^4 K \log |\cA| }\right),
\end{align*}
where $U_h^k(s, a) = \langle \overline{Q}_h^k(s, a, \cdot), \nu_h^k(\cdot|s)\rangle_\cB$, $\forall (s,a) \in \cS\times\cA$.
\end{lemma}
\begin{proof}
As shown in \eqref{eq:ascent}, the mirror ascent step at the $k$-th episode is to solve the following maximization problem
\begin{align*}
\maximize_{\mu } \sum_{h=1}^H  \big\langle \mu_h(\cdot|s) - \mu_h^k(\cdot|s),  U_h^k(s, \cdot) \big\rangle_\cA - \frac{1}{\eta} \sum_{h=1}^H  D_{\mathrm{KL}}\big( \mu_h(\cdot| s), \mu_h^k(\cdot| s ) \big), 
\end{align*}
with $U_h^k(s, a) = \langle \overline{Q}_h^k(s, a, \cdot), \nu_h^k(\cdot|s)\rangle_\cB$. We can further equivalently rewrite this maximization problem as a minimization problem as
\begin{align*}
\minimize_{\mu} -\sum_{h=1}^H  \big\langle \mu_h(\cdot|s) - \mu_h^k(\cdot|s),  U_h^k(s, \cdot) \big\rangle_\cA + \frac{1}{\eta} \sum_{h=1}^H  D_{\mathrm{KL}}\big( \mu_h(\cdot| s), \mu_h^k(\cdot| s ) \big). 
\end{align*}
Note that the closed-form solution $\mu_h^{k+1}(\cdot|s), \forall s\in \cS$, to this minimization problem is guaranteed to stay in the relative interior of a probability simplex when initialize $\mu^0_h(\cdot | s) = \boldsymbol{1} / |\cA|$. Thus, we can apply Lemma \ref{lem:pushback} and obtain that for any $\mu = \{\mu_h\}_{h=1}^H$, the following inequality holds
\begin{align*}
&-\eta \big\langle \mu_h^{k+1}(\cdot|s),  U_h^k(s, \cdot)\big\rangle_\cA  + \eta \big\langle \mu_h(\cdot|s),  U_h^k(s, \cdot)\big\rangle_\cA \\
&\qquad \leq D_{\mathrm{KL}}\big( \mu_h(\cdot| s), \mu_h^k(\cdot| s ) \big) -  D_{\mathrm{KL}}\big( \mu_h(\cdot| s), \mu_h^{k+1}(\cdot| s ) \big) -  D_{\mathrm{KL}}\big( \mu_h^{k+1}(\cdot| s), \mu_h^k(\cdot| s ) \big).
\end{align*}
Then, by rearranging the terms, we have
\begin{align}
\begin{aligned} \label{eq:ascent_al}
&\eta \big\langle  \mu_h^*(\cdot|s)- \mu_h^k(\cdot|s),  U_h^k(s, \cdot)\big\rangle_\cA  \\
&\qquad \leq D_{\mathrm{KL}}\big( \mu^*_h(\cdot|s), \mu_h^k(\cdot|s) \big) -  D_{\mathrm{KL}}\big( \mu^*_h(\cdot|s), \mu_h^{k+1} (\cdot|s) \big) -  D_{\mathrm{KL}}\big( \mu_h^{k+1}(\cdot|s), \mu_h^k(\cdot|s) \big) \\
&\qquad \quad + \eta \big\langle  \mu_h^{k+1}(\cdot|s) -  \mu_h^k(\cdot|s),  U_h^k(s, \cdot)\big\rangle_\cA.
\end{aligned}
\end{align}
Due to Pinsker's inequality, we have
\begin{align*}
&-D_{\mathrm{KL}}\big( \mu_h^{k+1}(\cdot| s), \mu_h^k(\cdot| s ) \big) \leq -\frac{1}{2} \big\|\mu_h^{k+1}(\cdot| s) - \mu_h^k(\cdot| s )\big\|^2_1.
\end{align*}
Moreover, by Cauchy-Schwarz inequality, we have
\begin{align*}
 \eta \big\langle  \mu_h^{k+1}(\cdot|s) -  \mu_h^k(\cdot|s),  U_h^k(s, \cdot)\big\rangle_\cA \leq \eta H \big\|\mu_h^{k+1}(\cdot|s) - \mu_h^k(\cdot|s) \big\|_1.
\end{align*}
Thus, we have
\begin{align}
\begin{aligned} \label{eq:ascent_err}
&-D_{\mathrm{KL}}\big( \mu_h^{k+1}(\cdot| s), \mu_h^k(\cdot| s ) \big) +  \eta \big\langle  \mu_h^{k+1}(\cdot|s) -  \mu_h^k(\cdot|s),  U_h^k(s, \cdot)\big\rangle_\cA \\
&\qquad \leq -\frac{1}{2} \big\|\mu_h^{k+1}(\cdot| s) - \mu_h^k(\cdot| s )\big\|^2_1 + \eta H \big\|\mu_h^{k+1}(\cdot|s) - \mu_h^k(\cdot|s) \big\|_1  \leq  \frac{1}{2}\eta^2H^2,
\end{aligned}
\end{align}
where the last inequality is by viewing $\big\|\mu_h^{k+1}(\cdot|s) - \mu_h^k(\cdot|s) \big\|_1$ as a variable $x$ and finding the maximal value of $-1/2\cdot x^2 + \eta H x$ to obtain the upper bound $1/2 \cdot \eta^2H^2$.

Thus, combing \eqref{eq:ascent_err} with \eqref{eq:ascent_al}, the policy improvement step in Algorithm \ref{alg:po1} implies
\begin{align*}
&\eta \big\langle  \mu_h^*(\cdot|s)- \mu_h^k(\cdot|s),  U_h^k(s, \cdot)\big\rangle_\cA \\
&\qquad \leq D_{\mathrm{KL}}\big( \mu^*_h(\cdot | s), \mu_h^k (\cdot | s) \big) -  D_{\mathrm{KL}}\big( \mu^*_h(\cdot | s), \mu_h^{k+1}(\cdot | s) \big)  + \frac{1}{2}\eta^2H^2,
\end{align*}
which further leads to
\begin{align*}
&\sum_{h=1}^H \EE_{\mu^*, \cP} \Big[\big\langle  \mu_h^*(\cdot|s)- \mu_h^k(\cdot|s),  U_h^k(s, \cdot)\big\rangle_\cA \Big]\\
&\qquad  \leq \frac{1}{\eta } \sum_{h=1}^H  \EE_{\mu^*, \cP}  \big[D_{\mathrm{KL}}\big( \mu^*_h(\cdot | s), \mu_h^k (\cdot | s) \big) -  D_{\mathrm{KL}}\big( \mu^*_h(\cdot | s), \mu_h^{k+1}(\cdot | s) \big)\big] + \frac{1}{2}\eta H^3.
\end{align*}
Moreover, we take summation from $k=1$ to $K$ of both sides and then obtain
\begin{align*}
&\sum_{k=1}^K \sum_{h=1}^H \EE_{\mu^*, \cP} \Big[\big\langle  \mu_h^*(\cdot|s)- \mu_h^k(\cdot|s),  U_h^k(s, \cdot)\big\rangle_\cA \Big]\\
&\qquad  \leq \frac{1}{\eta } \sum_{h=1}^H  \EE_{\mu^*, \cP}  \big[D_{\mathrm{KL}}\big( \mu^*_h(\cdot | s), \mu_h^1 (\cdot | s) \big) -  D_{\mathrm{KL}}\big( \mu^*_h(\cdot | s), \mu_h^{K+1}(\cdot | s) \big)\big] + \frac{1}{2}\eta K  H^3 \\
&\qquad  \leq \frac{1}{\eta } \sum_{h=1}^H  \EE_{\mu^*, \cP}  \big[ D_{\mathrm{KL}}\big( \mu^*_h(\cdot | s), \mu_h^1 (\cdot | s) \big)\big]  + \frac{1}{2}\eta K  H^3,
\end{align*}
where the last inequality is non-negativity of KL divergence. By the initialization in Algorithm \ref{alg:po1}, it is guaranteed that $\mu_h^1 (\cdot | s) = \boldsymbol{1}/|\cA|$, which thus leads to $D_{\mathrm{KL}}\left( \mu^*_h(\cdot | s), \mu_h^1 (\cdot | s)\right) \leq \log |\cA|$. Then, with setting $\eta = \sqrt{  \log |\cA|/(KH^2)}$, we bound the last term as
\begin{align*}
\frac{1}{\eta } \sum_{h=1}^H  \EE_{\mu^*, \cP}  \big[D_{\mathrm{KL}}\big( \mu^*_h(\cdot | s), \mu_h^1 (\cdot | s) \big)\big]  + \frac{1}{2}\eta K H^3 \leq \cO\left( \sqrt{H^4 K \log |\cA| }\right),
\end{align*}
which gives
\begin{align*}
&\sum_{k=1}^K \sum_{h=1}^H \EE_{\mu^*, \cP} \Big[\big\langle  \mu_h^*(\cdot|s)- \mu_h^k(\cdot|s),  U_h^k(s, \cdot)\big\rangle_\cA \Big]   \leq \cO\left( \sqrt{H^4 K \log |\cA| }\right),
\end{align*}
This completes the proof.
\end{proof}


\begin{lemma} \label{lem:r_bound} For any $k \in [K]$, $h \in [H]$ and all $(s,a, b) \in \cS \times \cA \times \cB$, with probability at least $1-\delta$, we have
\begin{align*}
\big|\hat{r}_h^k (s, a, b) - r_h (s, a, b)\big| \leq \sqrt{\frac{4
\log (|\cS||\cA| |\cB| HK/\delta)}{ \max\{N^k_h(s,a,b), 1\}}}.
\end{align*}
\end{lemma}
This lemma is the same as Lemma \ref{lem:r_bound_it}. We rewrite it here for the completeness of the proofs in this section. In \eqref{eq:bonus_decomp}, we set $\beta_h^{r, k}(s,a,b) = \sqrt{\frac{4
\log (|\cS| |\cA| |\cB| H K /\delta)}{ \max \{ N^k_h(s,a,b), 1\} }}$, which equals the bound in Lemma \ref{lem:r_bound}.

\begin{lemma} \label{lem:P_bound} For any $k \in [K]$, $h \in [H]$ and all $(s,a) \in \cS \times \cA$, with probability at least $1-\delta$, we have
\begin{align*}
\left\|\hat{\cP}_h^k (\cdot \given s, a) - \cP_h (\cdot \given s, a)\right\|_1 \leq \sqrt{\frac{2|\cS| \log (|\cS||\cA|HK/\delta)}{ \max\{N_h^k(s,a), 1\}}}.
\end{align*}
\end{lemma}
\begin{proof} For $k \geq 1$, we have $\|\hat{\cP}_h^k (\cdot \given s, a) - \cP_h (\cdot \given s, a)\|_1 = \max_{\|\zb \|_\infty \leq  1}  ~\langle \hat{\cP}_h^k (\cdot \given s, a) - \cP_h (\cdot \given s, a), \zb\rangle_\cS$ by the duality.  We construct an $\epsilon$-cover for the set $\{\zb\in \RR^{|\cS|}: \|\zb\|_\infty \leq 1\}$ with the distance induced by $\|\cdot\|_\infty$, denoted as $\cC_\infty(\epsilon)$, such that for any $\zb \in \RR^{|\cS|}$, there always exists $\zb'\in \cC_\infty(\epsilon)$ satisfying $\|\zb-\zb' \|_\infty \leq \epsilon$. The covering number is $\cN_\infty(\epsilon)=|\cC_\infty(\epsilon)|= 1/\epsilon^{|\cS|}$. Thus, we have for any $(s, a) \in \cS \times \cA$ and any $\zb$ with $\|\zb\|_\infty \leq 1$, there exists $\zb'\in \cC_\infty(\epsilon)$ such that $\|\zb'-\zb\|_\infty \leq \epsilon$ and 
\begin{align*}
&\big\langle \hat{\cP}_h^k (\cdot \given s, a) - \cP_h (\cdot \given s, a), \zb\big\rangle_\cS \\
&\qquad = \big\langle \hat{\cP}_h^k (\cdot \given s, a) - \cP_h (\cdot \given s, a), \zb' \big\rangle_\cS + \big\langle \hat{\cP}_h^k (\cdot \given s, a) - \cP_h (\cdot \given s, a), \zb- \zb'\big\rangle_\cS \\
&\qquad \leq  \big\langle \hat{\cP}_h^k (\cdot \given s, a) - \cP_h (\cdot \given s, a), \zb' \big\rangle_\cS +  \epsilon \left\|\hat{\cP}_h^k (\cdot \given s, a) - \cP_h (\cdot \given s, a)\right\|_1, 
\end{align*}
such that we further have
\begin{align}
\begin{aligned} \label{eq:net}
& \left\|\hat{\cP}_h^k (\cdot \given s, a) - \cP_h (\cdot \given s, a)\right\|_1\\
&\qquad = \max_{\|\zb \|_\infty \leq  1}  ~\big\langle \hat{\cP}_h^k (\cdot \given s, a) - \cP_h (\cdot \given s, a), \zb\big\rangle_\cS \\
 &\qquad \leq  \max_{\zb' \in \cC_\infty(\epsilon)}  ~\big\langle \hat{\cP}_h^k (\cdot \given s, a) - \cP_h (\cdot \given s, a), \zb' \big\rangle_\cS +  \epsilon \left\|\hat{\cP}_h^k (\cdot \given s, a) - \cP_h (\cdot \given s, a)\right\|_1.
\end{aligned}
\end{align}
By Hoeffding's inequality and the union bound over all $\zb' \in \cC_\infty(\epsilon)$, when $N_h^k(s,a)>0$, with probability at least $1-\delta'$ where $\delta'\in (0, 1]$,
\begin{align} \label{eq:net_bound}
\max_{\zb' \in \cC_\infty(\epsilon)}  ~\big\langle \hat{\cP}_h^k (\cdot \given s, a) - \cP_h (\cdot \given s, a), \zb' \big\rangle_\cS \leq  	\sqrt{\frac{|\cS| \log (1/\epsilon) + \log (1/\delta')}{2N_h^k(s,a)}}.
\end{align}
Letting $\epsilon = 1/2$, by \eqref{eq:net} and \eqref{eq:net_bound}, with probability at least $1-\delta'$, we have
\begin{align*}
 \left\|\hat{\cP}_h^k (\cdot \given s, a) - \cP_h (\cdot \given s, a)\right\|_1 \leq 1 \sqrt{\frac{|\cS| \log 2 + \log (1/\delta')}{2N_h^k(s,a)}}.
\end{align*}
When $N_h^k(s,a)=0$, we have $\big\|\hat{\cP}_h^k (\cdot \given s, a) - \cP_h (\cdot \given s, a)\big\|_1 = \|\cP_h (\cdot \given s, a)\|_1 = 1$ such that $2 \sqrt{\frac{|\cS| \log 2 + \log (1/\delta')}{2}} > 1 = \big\|\hat{\cP}_h^k (\cdot \given s, a) - \cP_h (\cdot \given s, a)\big\|_1$ always holds. Thus, with probability at least $1-\delta'$, 
\begin{align*}
 \left\|\hat{\cP}_h^k (\cdot \given s, a) - \cP_h (\cdot \given s, a)\right\|_1 \leq 2 \sqrt{\frac{|\cS| \log 2 + \log (1/\delta')}{2\max \{N_h^k(s,a), 1\}}} \leq \sqrt{\frac{2|\cS| \log (2/\delta')}{\max \{N_h^k(s,a), 1\}}}.
\end{align*}
Then, by the union bound, assuming $K > 1$, letting $\delta = |\cS| |\cA| H K \delta'/2$, with probability at least $1-\delta$, for any $(s, a)\in \cS\times \cA$ and any $h\in [H]$ and $k\in [K]$, we have
\begin{align*}
\left\|\hat{\cP}_h^k (\cdot \given s, a) - \cP_h (\cdot \given s, a)\right\|_1 \leq \sqrt{\frac{2|\cS| \log (|\cS| |\cA| H K/\delta)}{\max \{N_h^k(s,a), 1\}}},
\end{align*}
This completes the proof.
\end{proof}
In \eqref{eq:bonus_decomp}, we set $\beta_h^{\cP, k}(a,b) = \sqrt{ \frac{2H^2 |\cS| \log (|\cS||\cA|HK/\delta)}{ \max\{N_h^k(s,a), 1\}}}$, which equals the product of the upper bound in Lemma \ref{lem:P_bound} and the factor $H$.

\begin{lemma} \label{lem:pred_err_1} With probability at least $1-2\delta$, Algorithm \ref{alg:po1} ensures that
\begin{align*}
&\sum_{k=1}^K \sum_{h=1}^H \EE_{\mu^*, \cP, \nu^k} \big[ \overline{\varsigma}_h^k(s_h, a_h, b_h)   \biggiven s_1 \big] \leq 0. 
\end{align*}
\end{lemma}

\begin{proof}

We prove the upper bound of the model prediction error term.  We can decompose the instantaneous prediction error at the $h$-step of the $k$-th episode as
\begin{align}
\begin{aligned}\label{eq:pred_err_init1}
&  \overline{\varsigma}_h^k(s, a, b) = r_h(s, a, b) + \big\langle \cP_h(\cdot \given s, a), \overline{V}_{h+1}^k(\cdot) \big\rangle_\cS - \overline{Q}_h^k(s, a, b) ,
\end{aligned}
\end{align}
where the equality is by the definition of the prediction error in \eqref{eq:pred_err_up}. By plugging in the definition of $\overline{Q}_h^k$ in Line \ref{line:Q_up} of Algorithm \ref{alg:po1}, for any $(s,a,b)$,  we bound the following term as 
\begin{align}
\begin{aligned}\label{eq:pred_err_re1}
&r_h(s,a,b) + \big\langle \cP_h(\cdot \given s, a), \overline{V}_{h+1}^k(\cdot) \big\rangle_\cS - \overline{Q}_h^k(s,a,b) \\
&\qquad \leq r_h(s,a,b) + \big\langle \cP_h(\cdot \given s, a), \overline{V}_{h+1}^k(\cdot) \big\rangle_\cS \\
&\qquad \quad- \min \Big\{ \hat{r}_h^k(s,a,b) + \big\langle \hat{\cP}_h^k(\cdot |s, a), \overline{V}_{h+1}^k(\cdot) \big\rangle_\cS  - \beta_h^k, H-h+1 \Big\}  \\
&\qquad \leq  \max \Big\{r_h(s,a,b) - \hat{r}_h^k(s,a,b) + \big\langle \cP_h(\cdot \given s, a) - \hat{\cP}_h^k(\cdot |s, a), \overline{V}_{h+1}^k(\cdot) \big\rangle_\cS  - \beta_h^k, 0 \Big\}, 
\end{aligned}
\end{align}
where the inequality holds because 
\begin{align*}
&r_h(s,a,b) + \big\langle \cP_h(\cdot \given s_h, a_h), \overline{V}_{h+1}^k(\cdot) \big\rangle_\cS \\
&\qquad \leq r_h(s,a,b) + \big\|\cP_h(\cdot \given s_h, a_h)\big\|_1 \|\overline{V}_{h+1}^k(\cdot) \|_\infty  \leq 1 +  \max_{s' \in \cS} \big|\overline{V}_{h+1}^k(s') \big| \leq 1 + H-h, 
\end{align*}
since $\big\|\cP_h(\cdot \given s_h, a_h)\big\|_1 = 1$ and also the truncation step as shown in Line \ref{line:Q_up} of Algorithm \ref{alg:po1} for $\overline{Q}_{h+1}^k$ such that for any $s' \in \cS$
\begin{align}
\begin{aligned}\label{eq:bound_V_up}
\big|\overline{V}_{h+1}^k(s') \big| &=  \Big| \big[\mu_{h+1}^k(\cdot | s')\big]^\top \overline{Q}_{h+1}^k(s', \cdot , \cdot) \nu_{h+1}^k(\cdot | s') \Big|\\
&\leq \big\| \mu_{h+1}^k(\cdot | s')\big\|_1 \big\|\overline{Q}_{h+1}^k(s', \cdot , \cdot) \nu_{h+1}^k(\cdot | s') \big\|_\infty \\
&\leq \max_{a,b} \big|\overline{Q}_{h+1}^k(s', a, b)\big|\leq H-h.
\end{aligned}
\end{align}
Combining \eqref{eq:pred_err_init1} and \eqref{eq:pred_err_re1} gives
\begin{align}
\begin{aligned}
\label{eq:pred_err_al}
\overline{\varsigma}_h^k(s, a, b) &\leq \max \Big\{ r_h(s,a,b) - \hat{r}_h^k(s,a,b) \\
& \quad +   \big\langle \cP_h(\cdot \given s, a) - \hat{\cP}_h^k(\cdot |s, a), \overline{V}_{h+1}^k(\cdot) \big\rangle_\cS  - \beta_h^k, 0 \Big\}.
\end{aligned}
\end{align}
Note that as shown in \eqref{eq:bonus_decomp}, we have
\begin{align*}
\beta_h^k(s,a,b) = \beta_h^{r,k}(s,a,b) + \beta_h^{\cP,k}(s,a).
\end{align*}
Then, with probability at least $1-\delta$, we have
\begin{align*}
&r_h(s,a,b) - \hat{r}_h^k(s,a,b) - \beta_h^{r,k}(s,a,b)  \\
&\qquad \leq \big|r_h(s,a,b) - \hat{r}_h^k(s,a,b)\big| - \beta_h^{r,k}(s,a,b)  \\
&\qquad \leq \beta_h^{r,k}(s,a,b)  - \beta_h^{r,k}(s,a,b)  = 0,
\end{align*}
where the last inequality is by Lemma \ref{lem:r_bound} and the setting of the bonus for the reward. Moreover, with probability at least $1-\delta$, we have
\begin{align*}
&\big\langle \cP_h(\cdot \given s, a) - \hat{\cP}_h^k(\cdot |s, a), \overline{V}_{h+1}^k(\cdot) \big\rangle_\cS  - \beta_h^{\cP, k}(s,a) \\
&\qquad \leq \big\|\cP_h(\cdot \given s, a) - \hat{\cP}_h^k(\cdot |s, a)\big\|_1 \big \|\overline{V}_{h+1}^k(\cdot) \big\|_\infty  - \beta_h^{\cP, k}(s,a) \\
&\qquad \leq H \big\|\cP_h(\cdot \given s, a) - \hat{\cP}_h^k(\cdot |s, a)\big\|_1   - \beta_h^{\cP, k}(s,a) \\
&\qquad \leq \beta_h^{\cP,k} (s,a)   - \beta_h^{\cP,k}(s,a)  = 0,
\end{align*}
where the first inequality is by Cauchy-Schwarz inequality, the second inequality is due to $\max_{s' \in \cS}\big\|\overline{V}_{h+1}^k(s')\big\|_\infty \leq H$ as shown in \eqref{eq:bound_V_up}, and the last inequality is by the setting of $\beta_h^{\cP,k}$ and also Lemma \ref{lem:P_bound}. Thus, with probability at least $1-2\delta$, the following inequality holds
\begin{align*}
r_h(s,a,b) - \hat{r}_h^k(s,a,b) + \big\langle \cP_h(\cdot \given s, a) - \hat{\cP}_h^k(\cdot |s, a), \overline{V}_{h+1}^k(\cdot) \big\rangle_\cS  - \beta_h^k(s, a, b) \leq 0.
\end{align*}
Combining the above inequality with \eqref{eq:pred_err_al}, we have that with probability at least $1-2\delta$, for any $h\in [H]$ and $k\in [K]$, the following inequality holds
\begin{align*}
&\overline{\varsigma}_h^k(s, a, b) \leq 0, ~~\forall (s, a, b) \in \cS \times \cA \times \cB,
\end{align*}
which leads to
\begin{align*}
\sum_{k=1}^K \sum_{h=1}^H \EE_{\mu^*, \cP, \nu^k} \big[ \overline{\varsigma}_h^k(s_h, a_h, b_h)   \biggiven s_1 \big] \leq 0.
\end{align*}
This completes the proof.
\end{proof}

\begin{lemma} \label{lem:bias_1} With probability at least $1-\delta$, Algorithm \ref{alg:po1} ensures that
\begin{align*}
\sum_{k=1}^K \overline{V}_1^k(s_1)  - \sum_{k=1}^K  V_1^{\mu^k, \nu^k}(s_1) \leq \tilde{\cO} \left( \sqrt{  |\cS|^2 |\cA| H^4 K } + \sqrt{  |\cS| |\cA| |\cB| H^2 K } \right).
\end{align*}
\end{lemma}

\begin{proof} We assume that a trajectory $\{ (s_h^k, a_h^k, b_h^k, s_{h+1}^k)\}_{h=1}^H$ for all $k\in [K]$ is generated following the policies $\mu^k$, $\nu^k$, and the true transition model $\cP$. Thus, we expand the bias term at the $h$-th step of the $k$-th episode, which is
\begin{align}
\begin{aligned} \label{eq:bia_1_init}
\overline{V}_h^k(s_h^k)  - V_h^{\mu^k, \nu^k}(s_h^k) & = \big[\mu^k_h(\cdot | s_h^k)\big]^\top \big[\overline{Q}_h^k(s_h^k, \cdot , \cdot)  -  Q_h^{\mu^k, \nu^k}(s_h^k, \cdot , \cdot)\big] \nu_h^k(\cdot | s_h^k)\\
&= \zeta_h^k + \overline{Q}_h^k(s_h^k, a_h^k, b_h^k)  -  Q_h^{\mu^k, \nu^k}(s_h^k, a_h^k, b_h^k)\\
&= \zeta_h^k + \big\langle \cP_h(\cdot \given s_h^k, a_h^k), \overline{V}_{h+1}^k(\cdot) - V_{h+1}^{\mu^k, \nu^k} (\cdot) \big\rangle_\cS - \overline{\varsigma}_h^k(s_h^k, a_h^k, b_h^k) \\
& = \zeta_h^k + \xi_h^k + \overline{V}_{h+1}^k(s_{h+1}^k) - V_{h+1}^{\mu^k, \nu^k} (s_{h+1}^k) - \overline{\varsigma}_h^k(s_h^k, a_h^k, b_h^k),
\end{aligned}
\end{align}
where the first equality is by Line \ref{line:V_up} of Algorithm \ref{alg:po1} and \eqref{eq:bellman_V}, the third equality is by plugging in \eqref{eq:bellman_Q} and \eqref{eq:pred_err_up}. Specifically, in the above equality, we introduce two martingale difference sequence, namely, $\{\zeta_h^k\}_{h\geq 0, k\geq 0}$ and $\{\xi_h^k\}_{h\geq 0, k\geq 0}$, which are defined as
\begin{align*}
&\zeta_h^k := \big[\mu^k_h(\cdot | s_h^k)\big]^\top \big[\overline{Q}_h^k(s_h^k, \cdot , \cdot)  -  Q_h^{\mu^k, \nu^k}(s_h^k, \cdot , \cdot)\big] \nu_h^k(\cdot | s_h^k) - \big[ \overline{Q}_h^k(s_h^k, a_h^k, b_h^k)  -  Q_h^{\mu^k, \nu^k}(s_h^k, a_h^k, b_h^k)\big],\\
& \xi_h^k := \big\langle \cP_h(\cdot \given s_h^k, a_h^k), \overline{V}_{h+1}^k(\cdot) - V_{h+1}^{\mu^k, \nu^k} (\cdot) \big\rangle_\cS - \big[ \overline{V}_{h+1}^k(s_{h+1}^k) - V_{h+1}^{\mu^k, \nu^k} (s_{h+1}^k)\big],
\end{align*}
such that 
\begin{align*}
&\EE_{a_h^k \sim \mu^k_h(\cdot | s_h^k), b_h^k \sim \nu^k_h(\cdot | s_h^k)} \big[\zeta_h^k \biggiven \cF_h^k] = 0,\qquad \EE_{s_{h+1}^k \sim \cP_h(\cdot \given s_h^k, a_h^k)} \big[\xi_h^k \biggiven \tilde{\cF}_h^k\big] = 0,
\end{align*}
with $\cF_h^k$ being the filtration of all randomness up to $(h-1)$-th step of the $k$-th episode plus $s_h^k$, and $\tilde{\cF}_h^k$ being the filtration of all randomness up to $(h-1)$-th step of the $k$-th episode plus $s_h^k, a_h^k, b_h^k$.

We can observe that the equality \eqref{eq:bia_1_init} construct a recursion for $\overline{V}_h^k(s_h^k)  - V_h^{\mu^k, \nu^k}(s_h^k)$. Moreover, we also have $\overline{V}_{H+1}^k(\cdot) = \boldsymbol{0}$ and $V_{H+1}^{\mu^k, \nu^k} (\cdot)= \boldsymbol{0}$. Thus, recursively apply \eqref{eq:bia_1_init} from $h=1$ to $H$ leads to the following equality
\begin{align} \label{eq:bias_diff_init}
\overline{V}_1^k(s_1)  - V_1^{\mu^k, \nu^k}(s_1) =  \sum_{h=1}^H \zeta_h^k + \sum_{h=1}^H \xi_h^k - \sum_{h=1}^H \overline{\varsigma}_h^k(s_h^k, a_h^k, b_h^k).
\end{align}
Moreover, by \eqref{eq:pred_err_up} and Line \ref{line:Q_up} of Algorithm \ref{alg:po1}, we have
\begin{align*}
-\overline{\varsigma}_h^k(s_h^k, a_h^k, b_h^k) &= -  r_h(s_h^k, a_h^k, b_h^k) -  \big\langle \cP_h(\cdot \given s_h, a_h), \overline{V}_{h+1}^k(\cdot) \big\rangle_\cS \\
&\quad + \min \big\{ \hat{r}^k_h(s_h^k, a_h^k, b_h^k)  +  \big\langle \hat{\cP}_h^k(\cdot |s_h, a_h), \overline{V}_{h+1}^k(\cdot) \big\rangle_\cS  + \beta_h^k(s_h^k, a_h^k, b_h^k), H-h+1 \big\}.
\end{align*}
Then, we can further bound $-\overline{\varsigma}_h^k(s_h^k, a_h^k, b_h^k)$ as follows
\begin{align*}
-\overline{\varsigma}_h^k(s_h^k, a_h^k, b_h^k)&\leq -  r_h(s_h^k, a_h^k, b_h^k) - \big\langle \cP_h(\cdot \given s_h^k, a_h^k), \overline{V}_{h+1}^k(\cdot)\big\rangle_\cS + \hat{r}^k_h(s_h^k, a_h^k, b_h^k) \\
&\quad + \big\langle  \hat{\cP}_h^k(\cdot |s_h^k, a_h^k), \overline{V}_{h+1}^k(\cdot) \big\rangle_\cS  + \beta_h^k(s_h^k, a_h^k, b_h^k) \\
&\leq \big|\hat{r}^k_h(s_h^k, a_h^k, b_h^k) -  r_h(s_h^k, a_h^k, b_h^k)\big| \\
&\quad  + \Big| \big\langle \cP_h(\cdot \given s_h^k, a_h^k) - \hat{\cP}^k_h(\cdot \given s_h^k, a_h^k), \overline{V}_{h+1}^k(\cdot) \big\rangle_\cS \Big| + \beta_h^k(s_h^k, a_h^k, b_h^k),
\end{align*}
where the first inequality is due to $\min \{x,y\} \leq x$. Additionally, we have
\begin{align*}
\Big| \big\langle \cP_h(\cdot \given s_h^k, a_h^k) - \hat{\cP}^k_h(\cdot \given s_h^k, a_h^k), \overline{V}_{h+1}^k(\cdot) \big\rangle_\cS \Big|  
& \leq  \big\| \overline{V}_{h+1}^k(\cdot)\big\|_\infty \big\| \cP_h(\cdot \given s_h^k, a_h^k) - \hat{\cP}^k_h(\cdot \given s_h^k, a_h^k) \big\|_1 \\
& \leq H \big\|\cP_h(\cdot \given s_h^k, a_h^k) - \hat{\cP}^k_h(\cdot \given s_h^k, a_h^k)\big\|_1,
\end{align*} 
where the first inequality is by Cauchy-Schwarz inequality and the second inequality is by \eqref{eq:bound_V_up}. Thus, putting the above together, we obtain
\begin{align*}
-\overline{\varsigma}_h^k(s_h^k, a_h^k, b_h^k)&\leq  \big|\hat{r}^k_h(s_h^k, a_h^k, b_h^k) -  r_h(s_h^k, a_h^k, b_h^k)\big| + H \big\|\overline{V}_{h+1}^k(\cdot)  - \overline{V}_{h+1}^k(\cdot) \big\|_1 + \beta_h^k(s_h^k, a_h^k, b_h^k) \\
&\leq  2\beta^{r,k}_h(s_h^k, a_h^k, b_h^k) + 2\beta^{\cP,k}_h(s_h^k, a_h^k),
\end{align*}
where the second inequality is by Lemma \ref{lem:r_bound}, Lemma \ref{lem:P_bound}, and the decomposition of the bonus term $\beta_h^k$ as \eqref{eq:bonus_decomp}.  Due to Lemma \ref{lem:r_bound} and Lemma \ref{lem:P_bound}, by the union bound, for any $h \in [H], k\in [K]$ and $(s_h, a_h, b_h) \in \cS \times \cA \times \cB$, the above inequality holds with probability at least $1-2\delta$. Therefore, by \eqref{eq:bias_diff_init}, with probability at least $1-2\delta$, we have
\begin{align}
\begin{aligned}\label{eq:bias_diff_al}
&\sum_{k=1}^K \big[\overline{V}_1^k(s_1)  - V_1^{\mu^k, \nu^k}(s_1)\big]\\
&\qquad  \leq \sum_{k=1}^K \sum_{h=1}^H \zeta_h^k + \sum_{k=1}^K\sum_{h=1}^H \xi_h^k + 2 \sum_{k=1}^K\sum_{h=1}^H \beta^{r,k}_h(s_h^k, a_h^k, b_h^k)   + 2\sum_{k=1}^K \sum_{h=1}^H \beta^{\cP,k}_h(s_h^k, a_h^k). 
\end{aligned}
\end{align}
By Azuma-Hoeffding inequality, with probability at least $1-\delta$, the following inequalities hold 
\begin{align*}
&\sum_{k=1}^K \sum_{h=1}^H \zeta_h^k \leq \cO\left(\sqrt{H^3 K \log \frac{1}{\delta}} \right),\\
&\sum_{k=1}^K \sum_{h=1}^H \xi_h^k \leq \cO\left(\sqrt{H^3 K \log \frac{1}{\delta}} \right), 
\end{align*} 
where we use the facts that $ | \overline{Q}_h^k(s_h^k, a_h^k, b_h^k)  -  Q_h^{\mu^k, \nu^k}(s_h^k, a_h^k, b_h^k) | \leq 2H$ and $| \overline{V}_{h+1}^k(s_{h+1}^k) - V_{h+1}^{\mu^k, \nu^k} (s_{h+1}^k) |\leq 2H$. Next, we need to bound $\sum_{k=1}^K\sum_{h=1}^H \beta^{r,k}_h(s_h^k, a_h^k, b_h^k)$ and $\sum_{k=1}^K \sum_{h=1}^H \beta^{\cP,k}_h(s_h^k, a_h^k)$ in \eqref{eq:bias_diff_al}. We show that 
\begin{align*}
\sum_{k=1}^K\sum_{h=1}^H \beta^{r,k}_h(s_h^k, a_h^k, b_h^k) & = C\sum_{k=1}^K\sum_{h=1}^H   \sqrt{\frac{ \log (|\cS| |\cA| |\cB| HK/\delta)}{\max\{N^k_h(s_h^k, a_h^k, b_h^k), 1\}}} \\
& = C\sum_{k=1}^K\sum_{h=1}^H   \sqrt{\frac{ \log (|\cS| |\cA| |\cB| HK/\delta)}{N^k_h(s_h^k, a_h^k, b_h^k)}} \\
&\leq C \sum_{h=1}^H ~ \sum_{\substack{(s, a, b)\in \cS\times\cA \times \cB\\ N^K_h(s, a, b) > 0}}\sum_{n=1}^{N^K_h(s, a, b)}  \sqrt{\frac{ \log (|\cS| |\cA| |\cB| HK/\delta)}{n}},
\end{align*}
where the second equality is because $(s_h^k, a_h^k, b_h^k)$ is visited such that $N^k_h(s_h^k, a_h^k, b_h^k) \geq 1$. In addition, we have
\begin{align*}
&\sum_{h=1}^H ~ \sum_{\substack{(s, a, b)\in \cS\times\cA \times \cB\\ N^K_h(s, a, b) > 0}}\sum_{n=1}^{N^K_h(s, a, b)}  \sqrt{\frac{ \log (|\cS| |\cA| |\cB| HK/\delta)}{n}} \\
&\qquad\leq \sum_{h=1}^H ~ \sum_{(s, a, b)\in \cS\times\cA \times \cB}  \cO \left(\sqrt{ N^K_h(s, a, b) \log \frac{|\cS| |\cA| |\cB| HK}{\delta}} \right) \\
&\qquad \leq \cO \left(H \sqrt{ K |\cS||\cA||\cB| \log \frac{|\cS| |\cA| |\cB| HK}{\delta}} \right),
\end{align*}
where the last inequality is based on the consideration that $\sum_{(s, a, b)\in \cS\times\cA \times \cB} N_h^K(s,a,b) = K$ such that $\sum_{(s, a, b)\in \cS\times\cA \times \cB} \sqrt{ N^K_h(s, a, b)} \leq \cO\left(\sqrt{ K |\cS||\cA||\cB|}\right) $ when $K$ is sufficiently large. Putting the above together, we obtain
\begin{align*}
\sum_{k=1}^K\sum_{h=1}^H \beta^{r,k}_h(s_h^k, a_h^k, b_h^k)  \leq \cO \left(H \sqrt{ K |\cS||\cA||\cB| \log \frac{|\cS| |\cA| |\cB| HK}{\delta}} \right).
\end{align*}
Similarly, we have
\begin{align*}
\sum_{k=1}^K \sum_{h=1}^H \beta^{\cP,k}_h(s_h^k, a_h^k) & = \sum_{k=1}^K\sum_{h=1}^H   \sqrt{\frac{H^2 |\cS| \log (|\cS| |\cA| HK/\delta)}{\max\{N^k_h(s_h^k, a_h^k), 1\}}} \\
&\leq \sum_{h=1}^H ~ \sum_{(s, a )\in \cS\times\cA }  \cO \left(\sqrt{ N^K_h(s, a) H^2 |\cS| \log \frac{|\cS| |\cA| HK}{\delta}} \right) \\
&\leq \sum_{h=1}^H ~ \sum_{(s, a )\in \cS\times\cA }  \cO \left(\sqrt{ \sum_{b\in B}N^K_h(s, a, b) H^2 |\cS|  \log \frac{|\cS| |\cA| HK}{\delta}} \right) \\
&\leq \cO \left(H \sqrt{ K |\cS|^2 |\cA| H^2  \log \frac{|\cS| |\cA| HK}{\delta}} \right),
\end{align*}
where the second inequality is due to $ \sum_{b\in \cB} N^K_h(s, a, b) = N^K_h(s, a)$, and the last inequality is based on the consideration that $\sum_{(s, a, b)\in \cS\times\cA \times \cB} N_h^K(s,a,b) = K$ such that $\sum_{(s, a)\in \cS\times\cA } \sqrt{ \sum_{b\in \cB} N^K_h(s, a, b)} \leq \cO(\sqrt{ K |\cS||\cA|}) $ when $K$ is sufficiently large. 

Thus, by \eqref{eq:bias_diff_al}, with probability at least $1-\delta$, we have
\begin{align*}
\sum_{k=1}^K \overline{V}_1^k(s_1)  - \sum_{k=1}^K  V_1^{\mu^k, \nu^k}(s_1) \leq \tilde{\cO} ( \sqrt{  |\cS|^2 |\cA| H^4 K } + \sqrt{  |\cS| |\cA| |\cB| H^2 K } )
\end{align*}
where $\tilde{\cO}$ hides logarithmic terms. This completes the proof.
\end{proof}


\begin{lemma} \label{lem:mirror_2} With setting $\gamma = \sqrt{  |\cS|\log |\cB|/K}$, the mirror descent steps of Algorithm \ref{alg:po2} lead to
\begin{align*} 
\sum_{k=1}^K \sum_{h=1}^H \sum_{s\in \cS} d_h^{\mu^k, \hat{\cP}^k}(s) \big\langle W_h^k(s, \cdot),  \nu_h^k(\cdot | s) - \nu_h^*(\cdot | s) \big\rangle \leq \cO \left( \sqrt{H^2 |\cS| K \log |\cB|} \right),
\end{align*}
where $W_h^k(s, b) =  \langle \tilde{r}_h^k(s, \cdot , b),  \mu^k_h( \cdot \given s)  \rangle_\cA $.
\end{lemma}

\begin{proof}

Similar to the proof of Lemma \ref{lem:mirror_1}, and also by Lemma \ref{lem:pushback}, for any $\nu = \{\nu_h\}_{h=1}^H$ and $s\in \cS$, the mirror descent step in Algorithm \ref{alg:po2} leads to 
\begin{align*}
&\gamma  d_h^{\mu^k, \hat{\cP}^k}(s)  \big\langle W_h^k(s,\cdot), \nu^{k+1}_h(\cdot|s) \big\rangle_\cB   - \gamma d_h^{\mu^k, \hat{\cP}^k}(s) \big\langle   W_h^k(s,\cdot), \nu_h(\cdot|s) \big\rangle_\cB \\
&\qquad  \leq  D_{\mathrm{KL}}\big( \nu_h(\cdot| s), \nu_h^k(\cdot| s) \big) - D_{\mathrm{KL}}\big( \nu_h(\cdot| s), \nu_h^{k+1}(\cdot| s ) \big) -  D_{\mathrm{KL}}\big( \nu_h^{k+1}(\cdot| s), \nu_h^k(\cdot| s) \big),
\end{align*}
according to \eqref{eq:descent}, where $W_h^k(s, b) = \big\langle \mu^k_h(\cdot | s), \tilde{r}_h^k (s,\cdot,b) \big\rangle$.  Then, by rearranging the terms, we have
\begin{align}
\begin{aligned}\label{eq:mirror_descent_al}
&\gamma d_h^{\mu^k, \hat{\cP}^k}(s)  \big\langle  W_h^k(s,\cdot), \nu_h^k(\cdot|s)-\nu_h^*(\cdot|s)\big\rangle_\cB  \\
&\qquad  \leq D_{\mathrm{KL}}\big( \nu^*_h(\cdot|s), \nu_h^k(\cdot|s) \big) -  D_{\mathrm{KL}}\big( \nu^*_h(\cdot|s), \nu_h^{k+1}(\cdot|s) \big) -  D_{\mathrm{KL}}\big( \nu_h^{k+1}(\cdot|s), \nu_h^k(\cdot|s) \big) \\
&\qquad  \quad  - \gamma d_h^{\mu^k, \hat{\cP}^k}(s) \big\langle  W_h^k(s,\cdot), \nu_h^{k+1}(\cdot|s) - \nu_h^k(\cdot|s) \big\rangle_\cB.
\end{aligned}
\end{align}
Due to Pinsker's inequality, we have
\begin{align} \label{eq:pins_2}
&-D_{\mathrm{KL}}\big( \nu_h^{k+1}(\cdot| s), \nu_h^k(\cdot| s) \big)  \leq -\frac{1}{2} \big\|\nu_h^{k+1}(\cdot| s) - \nu_h^k(\cdot| s)\big\|^2_1.
\end{align}
Moreover, we have
\begin{align}
\begin{aligned}\label{eq:inner_diff_2}
&- \gamma d_h^{\mu^k, \hat{\cP}^k}(s)\big\langle  W_h^k(s, \cdot) , \nu_h^k(\cdot|s) - \nu_h^{k+1}(\cdot|s)\big\rangle_\cB  \\
&\qquad \leq \gamma  d_h^{\mu^k, \hat{\cP}^k}(s) \big\|W_h^k(s, \cdot) \big\|_\infty \big\|\nu_h^{k+1}(\cdot|s)-\nu_h^k(\cdot|s)\big\|_1\\
&\qquad \leq \gamma  d_h^{\mu^k, \hat{\cP}^k}(s) \big\|\nu_h^{k+1}(\cdot|s)-\nu_h^k(\cdot|s)\big\|_1,
\end{aligned}
\end{align}
where the last inequality is by 
\begin{align*}
\|W_h^k(s, \cdot) \|_\infty &= \max_{b\in \cB} W_h^k(s, b)\leq  \max_{s\in \cS, b\in \cB} W_h^k(s, b)\\
&\leq \max_{s\in \cS, b\in \cB} \big\langle \tilde{r}_h^{k-1}(s, \cdot , b),   \mu^k_h( \cdot \given s)  \big\rangle\\
&\leq \max_{s\in \cS, b\in \cB} \big\|\tilde{r}_h^{k-1}(s, \cdot , b)\big\|_\infty  \big\| \mu^k_h( \cdot \given s)  \big\|_1 \leq 1.
\end{align*}
due to the definition of $W_h^k$ and $0\leq \tilde{r}_h^k(s, a, b) =  \max \{ \hat{r}^k_h(s,a,b)  - \beta_h^{r, k}, 0 \} \leq \hat{r}^k_h(s,a,b) \leq 1$. Combining \eqref{eq:pins_2} and \eqref{eq:inner_diff_2} gives
\begin{align*}
&-D_{\mathrm{KL}}\big( \nu_h^{k+1}(\cdot| s), \nu_h^k(\cdot| s) \big) - \gamma d_h^{\mu^k, \hat{\cP}^k}(s)\big\langle  W_h^k(s, \cdot) , \nu_h^k(\cdot|s) - \nu_h^{k+1}(\cdot|s)\big\rangle \\
&\qquad \leq -\frac{1}{2} \big\|\nu_h^{k+1}(\cdot| s) - \nu_h^k(\cdot| s)\big\|^2_1 + \gamma  d_h^{\mu^k, \hat{\cP}^k}(s) \big\|\nu_h^{k+1}(\cdot|s)-\nu_h^k(\cdot|s)\big\|_1\\
&\qquad \leq  \frac{1}{2} \big[d_h^{\mu^k, \hat{\cP}^k}(s)\big]^2 \gamma^2 \leq  \frac{1}{2} d_h^{\mu^k, \hat{\cP}^k}(s) \gamma^2,
\end{align*}
where the second inequality is obtained via solving $\max_{x} \{ -1/2 \cdot x^2 + \gamma  d_h^{\mu^k, \hat{\cP}^k}(s) \cdot x
\}$ if letting $x = \|\nu_h^{k+1}(\cdot|s)-\nu_h^k(\cdot|s)\|_1$. Plugging the above inequality into \eqref{eq:mirror_descent_al} gives
\begin{align*}
& \gamma d_h^{\mu^k, \hat{\cP}^k}(s) \big\langle W_h^k(s, \cdot) , \nu_h^k(\cdot|s) - \nu_h^*(\cdot|s)\big\rangle_\cB \\
&\qquad \leq D_{\mathrm{KL}}\big( \nu^*_h(\cdot|s), \nu_h^k(\cdot|s) \big) -  D_{\mathrm{KL}}\big( \nu^*_h(\cdot|s), \nu_h^{k+1}(\cdot|s) \big) + \frac{1}{2} d_h^{\mu^k, \hat{\cP}^k}(s) \gamma^2.
\end{align*}
Thus, the policy improvement step implies
\begin{align*}
&\sum_{h=1}^H \sum_{s\in \cS} d_h^{\mu^k, \hat{\cP}^k}(s) \big\langle W_h^k(s, \cdot),  \nu_h^k(\cdot | s) - \nu_h^*(\cdot | s)\big\rangle_\cB\\
&\qquad \leq \frac{1}{\gamma}\sum_{h=1}^H \sum_{s\in \cS} \big[ D_{\mathrm{KL}}\big( \nu^*_h(\cdot|s), \nu_h^k(\cdot|s) \big) -  D_{\mathrm{KL}}\big( \nu^*_h(\cdot|s), \nu_h^{k+1}(\cdot|s) \big)\big] + \frac{1}{\gamma} \sum_{h=1}^H \sum_{s\in \cS} \frac{1}{2} d_h^{\mu^k, \hat{\cP}^k}(s) \gamma^2\\
&\qquad\leq \frac{1}{\gamma} \sum_{h=1}^H \sum_{s\in \cS} \big[ D_{\mathrm{KL}}\big( \nu^*_h(\cdot|s), \nu_h^k(\cdot|s) \big) -  D_{\mathrm{KL}}\big( \nu^*_h(\cdot|s), \nu_h^{k+1}(\cdot|s) \big) \big] +  \frac{1}{2} H \gamma.
\end{align*}
Further taking summation from $k=1$ to $K$ on both sides of the above inequality gives
\begin{align*}
&\sum_{k=1}^K \sum_{h=1}^H \sum_{s\in \cS} d_h^{\mu^k, \hat{\cP}^k}(s) \big\langle W_h^k(s, \cdot),  \nu_h^k(\cdot | s) - \nu_h^*(\cdot | s) \big\rangle_\cB \\
& \qquad \leq \frac{1}{\gamma} \sum_{h=1}^H \sum_{s\in \cS} \big[ D_{\mathrm{KL}}\big( \nu^*_h(\cdot|s), \nu_h^1(\cdot|s) \big) -  D_{\mathrm{KL}}\big( \nu^*_h(\cdot|s), \nu_h^{K+1}(\cdot|s) \big) \big] +  \frac{1}{2} H K\gamma\\
& \qquad \leq \frac{1}{\gamma} \sum_{h=1}^H \sum_{s\in \cS} D_{\mathrm{KL}}\big( \nu^*_h(\cdot|s), \nu_h^1(\cdot|s) \big)  +  \frac{1}{2} H K\gamma.
\end{align*}
Note that by the initialization in Algorithm \ref{alg:po2}, it is guaranteed that $\nu_h^1 (\cdot | s) = \boldsymbol{1}/|\cB|$, which thus leads to $D_{\mathrm{KL}}\left( \mu^*_h(\cdot | s), \mu_h^1 (\cdot | s)\right) \leq \log |\cB|$. By setting $\gamma = \sqrt{  |\cS|\log |\cB|/K}$, we further bound the term as
\begin{align*}
\frac{1}{\gamma} \sum_{h=1}^H \sum_{s\in \cS} D_{\mathrm{KL}}\big( \nu^*_h(\cdot|s), \nu_h^1(\cdot|s) \big)  +  \frac{1}{2} H K\gamma  \leq \cO \left( \sqrt{H^2 |\cS| K \log |\cB|} \right),
\end{align*}
which gives
\begin{align*}
\sum_{k=1}^K \sum_{h=1}^H \sum_{s\in \cS} d_h^{\mu^k, \hat{\cP}^k}(s) \big\langle W_h^k(s, \cdot),  \nu_h^k(\cdot | s) - \nu_h^*(\cdot | s) \big\rangle_\cB \leq \cO \left( \sqrt{H^2 |\cS| K \log |\cB|} \right).
\end{align*}
This completes the proof.
\end{proof}

\begin{lemma} \label{lem:pred_err2} With probability at least $1-\delta$, Algorithm \ref{alg:po2} ensures that
\begin{align*}
\sum_{k=1}^K\sum_{h=1}^H \sum_{s\in \cS} d_h^{\mu^k, \hat{\cP}^k}(s) [\mu_h^k(\cdot| s)]^\top \underline{\varsigma}_h^k(s, \cdot, \cdot)  \nu_h^*(\cdot | s)  \leq 0,
\end{align*}
where $\underline{\varsigma}_h^k(s, a, b)= \tilde{r}_h^k (s,a,b)-r_h (s,a,b)$.
\end{lemma}

\begin{proof} With probability at least $1-\delta$, for any $(s,a,b)\in \cS\times\cA\times \cB, h\in [H], k\in [K]$, we have
\begin{align*}
    \underline{\varsigma}_h^k(s, a, b)&= \tilde{r}_h^k (s,a,b)-r_h (s,a,b)\\
    &=\max \big\{ \hat{r}^{k-1}_h(s,a,b)- r_h (s, a, b) - \beta_h^{r, k-1}, -r_h (s, a, b) \big\}\\
    &\leq \max \big\{ 0, -r_h (s, a, b) \big\}=0,
\end{align*}
where $\tilde{r}_h^k (s,a,b)$ is computed as in Algorithm \ref{alg:po2} and the inequality is by $\hat{r}^{k-1}_h(s,a,b)- r_h (s, a, b) - \beta_h^{r, k-1} \leq 0$ with probability at least $1-\delta$ by Lemma \ref{lem:r_bound}. The above result reflects the optimism of $\tilde{r}_h^k$. Therefore, with probability at least $1-\delta$, we have
\begin{align*}
&\sum_{k=1}^K\sum_{h=1}^H \sum_{s\in \cS} d_h^{\mu^k, \hat{\cP}^k}(s) [\mu_h^k(\cdot| s)]^\top \underline{\varsigma}_h^k(s, \cdot, \cdot)  \nu_h^*(\cdot | s) \\
&\quad\  = \sum_{h=1}^H \sum_{s\in \cS} d_h^{\mu^k, \hat{\cP}^k}(s) \sum_{a,b}\mu_h^k(a| s)\big[ \tilde{r}_h^k (s, a, b)-r_h (s, a, b) \big] \nu_h^*(b | s)\\
&\quad\ \leq 0.
\end{align*}
This completes the proof. 
\end{proof}


Before giving the next lemma, we first present the following definition for the proof of the next lemma.

\begin{definition}[Confidence Set] Define the following confidence set for transition models
\begin{align*}
\Upsilon^k := \Big\{\tilde{\cP} : &\left|\tilde{\cP}_h(s'|s,a) - \hat{\cP}^k_h(s'|s,a)\right| \leq \epsilon_h^k, ~\|\tilde{\cP}_h(\cdot|s,a)\|_1=1, \\
& \text{ and }~\tilde{\cP}_h(s'|s,a) \geq 0, ~\forall (s,a,s')\in \cS\times \cA \times \cS, \forall k\in [K] \Big\}
\end{align*}
where we define
\begin{align*}
& \epsilon_h^k := 2 \sqrt{\frac{\hat{\cP}^k_h(s'|s,a) \log (|\cS| |\cA| H K/\delta' )}{\max\{ N_h^k(s,a)-1, 1 \} }} +  \frac{14\log (|\cS| |\cA| H K/\delta')}{3\max\{ N_h^k(s,a)-1, 1 \}}
\end{align*}
with $N_h^k(s,a):=\sum_{\tau = 1}^k \mathbbm{1}\{(s,a) = (s_h^\tau, a_h^\tau)\}$ and $\hat{\cP}^k$ being the empirical transition model.

\end{definition}

\begin{lemma} \label{lem:stationary_dist_err} With probability at least $1-\delta$, the difference between $q^{\mu^k, \cP}$ and $d^k$ are bounded as
 \begin{align*}
&\sum_{k=1}^K\sum_{h=1}^H\sum_{s\in \cS}\left|q_h^{\mu^k, \cP}(s) - d_h^{\mu^k, \hat{\cP}^k}(s)\right| \leq \tilde{\cO}\left(H^2|\cS|\sqrt{|\cA| K} \right).
\end{align*}
\end{lemma}

\begin{proof}
By the definition of state distribution, we first have
\begin{align*}
&\sum_{k=1}^K\sum_{h=1}^H\sum_{s\in \cS}\left|q_h^{\mu^k, \cP}(s) - d_h^{\mu^k, \hat{\cP}^k}(s)\right| \\
&\qquad =  \sum_{k=1}^K \sum_{h=1}^H \sum_{s\in \cS} \left|\sum_{a\in \cA}   w_h^k(s,a) - \sum_{a\in \cA}   \hat{w}_h^k(s,a) \right| \\
&\qquad\leq  \sum_{k=1}^K \sum_{h=1}^H \sum_{s\in \cS} \sum_{a\in \cA}   \big| w_h^k(s,a) -  \hat{w}_h^k(s,a) \big|.
\end{align*}
where $ \hat{w}_h^k(s,a)$ is the occupancy measure under the empirical transition model $\hat{\cP}^k$ and the policy $\mu^k$. Then, since $\hat{\cP}^k\in \Upsilon^k$ always holds for any $k$,  by Lemma \ref{lem:occu_mea_bound}, we can bound the last term of the bound inequality such that with probability at least $1-6\delta'$,
\begin{align*}
\sum_{k=1}^K\sum_{h=1}^H\sum_{s\in \cS}\left|q_h^{\mu^k, \cP}(s) - d_h^{\mu^k, \hat{\cP}^k}(s)\right| \leq \cE_1 + \cE_2.
\end{align*}
Next, we compute the order of $\cE_1$ by Lemma \ref{lem:occu_err}. With probability at least $1-2\delta'$, we have
\begin{align*}
\cE_1 &= \cO\left[  \sum_{h=2}^H \sum_{h'=1}^{h-1} \sum_{k=1}^K \sum_{s\in \cS} \sum_{a\in \cA} w_h^k(s,a) \left(  \sqrt{ \frac{  |\cS|\log (|\cS| |\cA| H K/\delta' )}{\max\{ N_h^k(s,a), 1\} }} + \frac{ \log (|\cS| |\cA| H K/\delta' )}{\max\{ N_h^k(s,a), 1\}}  \right) \right]\\
&= \cO\left[  \sum_{h=2}^H \sum_{h'=1}^{h-1} \sqrt{|\cS|}\left( \sqrt{|\cS| |\cA| K} + |\cS| |\cA| \log K + \log \frac{H}{\delta'}  \right)  \log \frac{|\cS| |\cA| H K}{\delta'}\right] \\
&= \cO\left[  \left( H^2 |\cS| \sqrt{ |\cA| K} + H^2 |\cS|^{3/2} |\cA| \log K + H^2 \sqrt{|\cS|}\log \frac{H}{\delta'}  \right)  \log \frac{|\cS| |\cA| H K}{\delta'}\right] \\
&= \tilde{\cO} \left( H^2 |\cS| \sqrt{ |\cA| K} \right),
\end{align*}
where we ignore  $\log K$ terms when $K$ is sufficiently large such that $\sqrt{K}$ dominates, and $ \tilde{\cO}$ hides logarithm dependence on $|\cS|$, $|\cA|$, $H$, $K$, and $ 1/\delta'$. On the other hand, $\cE_2$ also depends on $\mathrm{ploy}(H, |\cS|, |\cA|)$ except the factor $\log \frac{|\cS| |\cA| H K}{\delta'}$ as shown in Lemma \ref{lem:occu_mea_bound}. Thus, $\cE_2$ can be ignored comparing to $\cE_1$ if $K$ is sufficiently large. Therefore, we eventually obtain that with probability at least $1-8\delta'$, the following inequality holds
\begin{align*}
\sum_{k=1}^K\sum_{h=1}^H\sum_{s\in \cS}\left|q_h^{\mu^k, \cP}(s) - d_h^{\mu^k, \hat{\cP}^k}(s)\right| \leq \tilde{\cO} \left( H^2 |\cS| \sqrt{ |\cA| K} \right) .
\end{align*}
We let $\delta = 8\delta'$ such that $\log \frac{|\cS| |\cA| H K}{\delta'} = \log \frac{8|\cS| |\cA| H K}{\delta}$ without changing the order as shown above.  Then, with probability at least $1-\delta$, we have $\sum_{k=1}^K\sum_{h=1}^H\sum_{s\in \cS} |q_h^{\mu^k, \cP}(s) - d_h^{\mu^k, \hat{\cP}^k}(s)| \leq \tilde{\cO} ( H^2 |\cS| \sqrt{ |\cA| K} )$. This completes the proof.
\end{proof}


\begin{lemma} \label{lem:reward_err} With probability at least $1-\delta$, the following inequality holds 
\begin{align*}
\sum_{k=1}^K \sum_{h=1}^H  \EE_{\mu^k, \cP, \nu^k} \big[\beta_h^{r,k}(s_h,a_h,b_h) \biggiven s_1\big] \leq \tilde{\cO} \left(\sqrt{|\cS| |\cA| |\cB| H^2 K } \right).
\end{align*} 
\end{lemma}

\begin{proof} Since we have
\begin{align*}
&\sum_{k=1}^K \sum_{h=1}^H  \EE_{\mu^k, \cP, \nu^k} \big[\beta_h^{r,k}(s_h,a_h,b_h) \biggiven s_1\big] \\
&\qquad = \sum_{k=1}^K \sum_{h=1}^H  \EE_{\mu^k, \cP, \nu^k} \left[C \sqrt{\frac{ \log (|\cS| |\cA| |\cB| HK/\delta)}{N^k_h(s,a, b)}}~\right] \\
&\qquad = C \sqrt{\log \frac{|\cS| |\cA| |\cB| HK}{\delta}} \sum_{k=1}^K \sum_{h=1}^H  \EE_{\mu^k, \cP, \nu^k} \left[\sqrt{\frac{ 1 }{N^k_h(s,a, b)}}~\right],
\end{align*}
then we can apply Lemma \ref{lem:expect_bonus} and obtain
\begin{align*}
\sum_{k=1}^K \sum_{h=1}^H  \EE_{\mu^k, \cP, \nu^k} \big[\beta_h^{r,k}(s_h,a_h,b_h) \biggiven s_1\big] \leq \tilde{\cO} \left(\sqrt{|\cS| |\cA| |\cB| H^2 K } \right),
\end{align*}
with probability at least $1-\delta$. Here $\tilde{\cO}$ hides logarithm dependence on $ |\cS|, |\cA|, |\cB|, H,  K$, and $1/\delta$. This completes the proof.
\end{proof}

\subsection{Other Supporting Lemmas}


\begin{lemma}\label{lem:conf_set} With probability at least $1-4\delta'$, the true transition model $\cP$ satisfies that for any $k\in [K]$,
\begin{align*}
\cP \in \Upsilon^k.
\end{align*}

\end{lemma}

This lemma implies that the estimated transition model $\hat{\cP}^k_h(s'|s,a)$ by \eqref{eq:estimate} is closed to the true transition model $\cP_h(s'|s,a)$ with high probability. The upper bound for their difference is by empirical Bernstein's inequality and the union bound. 

The next lemma is modified from Lemma 10 in \citet{jin2019learning}.

\begin{lemma} \label{lem:occu_err}We let $w_h^k(s,a)$ denote the occupancy measure at the $h$-th step of the $k$-th episode under the true transition model $\cP$ and the current policy $\mu^k$. Then, with probability at least $1-2\delta'$ we have for all $h\in[H]$, the following results hold
\begin{align*}
\sum_{k=1}^K \sum_{s\in \cS} \sum_{a\in \cA}\frac{w_h^k(s,a)}{\max\{N_h^k(s,a), 1\}} = \cO\left(|\cS| |\cA| \log K + \log \frac{H}{\delta'}\right),
\end{align*}
and
\begin{align*}
\sum_{k=1}^K \sum_{s\in \cS} \sum_{a\in \cA}\frac{w_h^k(s,a)}{\sqrt{\max\{N_h^k(s,a), 1\}}} = \cO\left(\sqrt{|\cS| |\cA| K} + |\cS| |\cA| \log K + \log \frac{H}{\delta'}\right).
\end{align*}

\end{lemma}

Furthermore, by Lemma \ref{lem:conf_set} and Lemma \ref{lem:occu_err}, we give the following lemma to characterize the difference of two occupancy measures, which is modified from parts of the proof of Lemma 4 in \citet{jin2019learning}.

\begin{lemma} \label{lem:occu_mea_bound} Let $w_h^k(s,a)$ be the occupancy measure at the $h$-th step of the $k$-th episode under the true transition model $\cP$ and the current policy $\mu^k$, and $\tilde{w}_h^k(s,a)$ be the occupancy measure at the $h$-th step of the $k$-th episode under any transition model $\tilde{\cP}^k \in \Upsilon^k$ and the current policy $\mu^k$ for any $k$. Then, with probability at least $1-6\delta'$ we have $\forall h\in[H]$, the following inequality holds
\begin{align*}
\sum_{k=1}^K \sum_{h=1}^K \sum_{s\in \cS} \sum_{a\in \cA}  \big|\tilde{w}_h^k(s,a) - w_h^k(s,a) \big|\leq \cE_1 + \cE_2, 
\end{align*}
where $\cE_1$ and $\cE_2$ are in the level of 
\begin{align*}
\cE_1 = \cO\left[  \sum_{h=2}^H \sum_{h'=1}^{h-1} \sum_{k=1}^K \sum_{s\in \cS} \sum_{a\in \cA} w_h^k(s,a) \left(  \sqrt{ \frac{  |\cS|\log (|\cS| |\cA| H K/\delta' )}{\max\{ N_h^k(s,a), 1\} }} + \frac{ \log (|\cS| |\cA| H K/\delta' )}{\max\{ N_h^k(s,a), 1\}}  \right) \right]
\end{align*} 
and
\begin{align*}
\cE_2 = \cO\left( \mathrm{poly}(H,|\cS|,|\cA|) \cdot\log \frac{|\cS| |\cA| H K}{\delta'} \right),
\end{align*}
where $ \mathrm{poly}(H,|\cS|,|\cA|)$ denotes the polynomial dependency on $H,|\cS|,|\cA|$.
\end{lemma}

\begin{lemma}\label{lem:expect_bonus} With probability at least $1-\delta$, the following inequality holds
\begin{align*}
\sum_{k=1}^K \sum_{h=1}^H  \EE_{\mu^k, \cP, \nu^k} \left[\sqrt{\frac{1}{\max\{ N_h^k(s,a,b), 1 \}}} ~\right] \leq \tilde{\cO} \left(\sqrt{|\cS| |\cA| |\cB| H^2 K } + |\cS| |\cA| |\cB| H \right),
\end{align*}
where $\tilde{\cO}$ hides logarithmic terms. 
\end{lemma}

\begin{proof}
The zero-sum Markov game with single-controller transition can interpreted as a regular MDP learning problem with policies $w^k_h(a,b \given s) = \mu^k_h(a|s) \nu^k_h(b|s)$ and a transition model $\cP_h(s' | s, a, b) = \cP_h(s' | s, a)$ with a joint action $(a,b)$ in the action space of size $|\cA|  |\cB|$. Thus, we apply Lemma 19 of \citet{efroni2020optimistic}, which extends lemmas in \citet{zanette2019tighter,efroni2019tight} to MDP with non-stationary dynamics by adding a factor of $H$,  to obtain our lemma. This completes the proof. 
\end{proof}

\section{Simulation Experiment} \label{sec:simul}

In this section, we provide a simulation experiment for our proposed algorithms to verify our theoretical findings. For simplicity, we only present the simulation for the single-controller game learning algorithms proposed in Section \ref{sec:SCT}.

\begin{figure}[h]
\centering
  \includegraphics[width=0.8\linewidth]{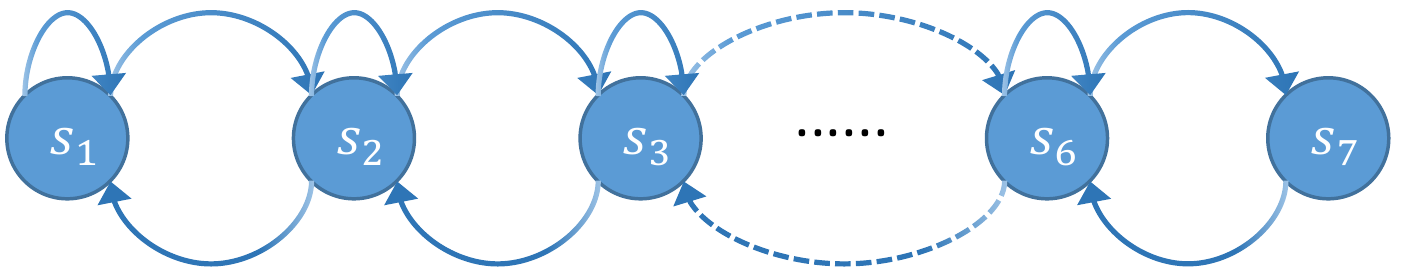}
  \caption{Environment for simulation.}
  \label{fig:env}
\end{figure}

Specifically, we consider an environment with 7 states in a chain shape as in Figure \ref{fig:env}, namely $|\cS| = 7$. We also let the episode length be $H=7$. The action spaces for both players are as defined as $\cA = \cB =\{0, 1\}$ such that $|\cA| = |\cB| = 2$. Each episode starts from the state $s_1$. Since the transition is only controlled by Player 1's action, we define is as follows: (1) $\cP_h(s_2|s_1,a=0)=0.1$, $\cP_h(s_1|s_1,a=0)=0.9$, $\cP_h(s_2|s=s_1,a=1)=0.9$, $\cP_h(s_1|s_1,a=1)=0.1$. (2) $\cP_h(s_{i+1}|s_i,a=0)=0.05$, $\cP_h(s_i|s_i,a=0)=0.05$, $\cP_h(s_{i-1}|s_i,a=0)=0.9$, $\cP_h(s_{i+1}|s_i,a=1)=0.9$, $\cP_h(s_i|s_i,a=1)=0.05$, and $\cP_h(s_{i-1}|s_i,a=1)=0.05$, for all $1<i<7$. The above definition of the transition model indicates that the state transition is only allowed within the nearest neighbors in the chain. In addition, we define the reward function in the following way: (1) $r_h(s_i,\cdot, \cdot) = 0.1$ for all $i<7$. (2) for the state $s_7$, the reward function is defined as $r_h(s_7,a=0, b=0) = 0.9$, $r_h(s_7,a=0, b=1) = 0.2$, $r_h(s_7,a=1, b=0) = 0.6$, and $r_h(s_7,a=1, b=1) = 0.4$. We let the observation $r_h^k$ of the reward follow a uniform distribution $\mathrm{Unif}[r_h(s,a,b) - 0.1, r_h(s,a,b) + 0.1]$. One can simply know that the target policies for the two players are: $\mu_h^*(a=1|s=s_h)=1$ for all $h\leq 7$ and  $\nu_7^*(a_7=1|s=s_7)=1$. Such policies indicate that Player 1 always takes the action $1$ and Player 2 can take a random policy when $h<7$ as it will not affect the observed reward and take the action $1$ at $h=7$. Then we can calculate $V^{\mu^*, \nu^*}_1(s_1) = 0.8594323$ according to the Bellman equation.

\begin{figure}[h]
\centering
  \includegraphics[width=0.8\linewidth]{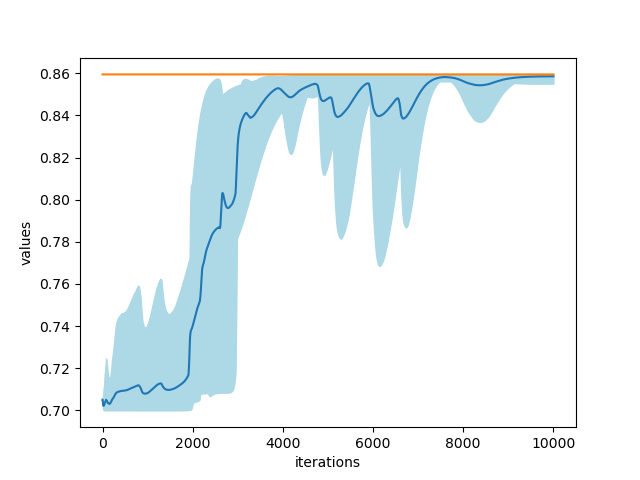}
  \caption{Simulation result.}
  \label{fig:result}
\end{figure}

We run Algorithms \ref{alg:po1} and \ref{alg:po2} together for $K = 10000$ rounds.  
In order to get a faster convergence in practice, we set relatively smaller UCB bonuses and larger step sizes. Then, the step sizes $\eta$ and $\gamma$ in Algorithms \ref{alg:po1} and \ref{alg:po2} are 50 times the theoretically suggested values, i.e., $\eta = 50\times\sqrt{\log |\cA|/(KH^2)}$ and $\gamma = 50\times\sqrt{|\cS|\log |\cB|/K}$. The bonuses are 0.01 times the theoretically suggested values, i.e., $\beta_h^{r, k}(s,a,b) = 0.01\times\sqrt{\frac{ 4\log (|\cS| |\cA| |\cB| HK/\delta)}{\max\{N^k_h(s,a, b), 1\}}}$ and $\beta_h^{\cP, k}(s,a) = 0.01\times\sqrt{\frac{2H^2 |\cS| \log (|\cS||\cA|HK/\delta)}{ \max\{N_h^k(s,a), 1\}}}$ with setting $\delta = 0.01$. In the experiment, we aim to show that the learned the value function $V_1^{\mu^k,\nu^k}(s_1)$ can converge to $V_1^{\mu^*,\nu^*}(s_1)$, where $(\mu^k,\nu^k)$ are the policies generated by the algorithms at the $k$-th round. We run the algorithms 5 times and present the averaged value function as in Figure \ref{fig:result}. Here the blue curve is $V_1^{\mu^k,\nu^k}(s_1)$ for $k\leq K$, which is averaged for the experiments of 5 times. The orange curve is $V^{\mu^*, \nu^*}_1(s_1) = 0.8594323$. From the experimental results, we can observe that after running Algorithms \ref{alg:po1} and \ref{alg:po2} together for $K = 10000$ rounds, $V_1^{\mu^k,\nu^k}(s_1)$ will converge to $V_1^{\mu^*,\nu^*}(s_1)$ in a sublinear rate.

\end{appendices}

\end{document}